\documentclass[11pt]{article}

\usepackage{main_style}
\usepackage{bibunits}
\defaultbibliographystyle{unsrt}

\title{Entropy and mutual information in\\ models of deep neural networks}

\author[1]{Marylou Gabrié\thanks{Corresponding author: \href{mailto:marylou.gabrie@ens.fr}{marylou.gabrie@ens.fr}}}
\author[2,3]{Andre Manoel}
\author[4]{Clément Luneau}
\author[4]{Jean Barbier}
\author[4]{Nicolas Macris}
\author[1,5,6,7]{Florent Krzakala}
\author[3,5]{Lenka Zdeborová}

\affil[1]{Laboratoire de Physique Statistique, École Normale Supérieure, PSL University}
\affil[2]{Parietal Team, INRIA, CEA, Université Paris-Saclay}
\affil[3]{Institut de Physique Théorique, CEA, CNRS, Université Paris-Saclay}
\affil[4]{Laboratoire de Théorie des Communications, École Polytechnique Fédérale de Lausanne}
\affil[5]{Department of Mathematics, Duke University, Durham NC}
\affil[6]{Sorbonne Universit\'es}
\affil[7]{LightOn, Paris}

\begin{document}
\maketitle

\begin{abstract} 
    We examine a class of stochastic deep learning models with a tractable method to
    compute information-theoretic quantities. Our contributions are
    three-fold: (i) We show how entropies and mutual informations can be
    derived from heuristic statistical physics methods, under the assumption
    that weight matrices are independent and orthogonally-invariant.  (ii)
    We extend particular cases in which this result is known to be
    rigorously exact by providing a proof for two-layers networks with
    Gaussian random weights, using the recently introduced adaptive
    interpolation method.  (iii) We propose an experiment framework with
    generative models of synthetic datasets, on which we train deep neural
    networks with a weight constraint designed so that the assumption in (i)
    is verified during learning.  We study the behavior of entropies and
    mutual informations throughout learning and conclude that, in the
    proposed setting, the relationship between compression and
    generalization remains elusive.
\end{abstract}

\addtocontents{toc}{\protect\setcounter{tocdepth}{0}}

\begin{bibunit}
%
The successes 
of deep
learning methods have spurred efforts towards quantitative modeling 
of the performance of deep neural networks. 
In particular, an information-theoretic approach linking
generalization capabilities
to compression has been receiving increasing interest. The
intuition behind the study of mutual informations in latent variable models
dates back to the information bottleneck (IB) theory of \cite{IB}. Although
recently reformulated in the context of deep learning \cite{Tishby2015},
verifying its relevance in practice requires the computation of mutual
informations for high-dimensional variables, a notoriously hard problem.
Thus, pioneering works in this direction focused either on small network
models with discrete (continuous, eventually binned) activations
\cite{shwartz-ziv_opening_2017}, or on linear networks
\cite{Chechik2005,saxe_information_2018}.

In the present paper we follow a different direction, and build on recent
results from statistical physics
\cite{kabashima_inference_2008,manoel_multi-layer_2017} and information
theory \cite{fletcher_inference_2017,reeves_additivity_2017} to propose, in
Section \ref{sec:main}, a formula to compute information-theoretic
quantities for a class of deep neural network models.
The models we approach, described in Section \ref{sec:dlmodels}, are 
non-linear feed-forward neural networks trained on synthetic datasets
with constrained weights.
Such networks capture some of the key properties of the
deep learning setting that are usually difficult to include in tractable
frameworks: non-linearities, arbitrary large width and depth, and
correlations in the input data. We demonstrate the proposed method in a series of
numerical experiments in Section \ref{sec:results}. First observations
suggest a rather complex picture, where the role of compression in the
generalization ability of deep neural networks is yet to be elucidated.

\section{Multi-layer model and main theoretical results}
\label{sec:main}
{\bf A stochastic multi-layer model---}
We consider a model of multi-layer stochastic feed-forward neural network
where each element $x_i$ of the input layer $\bm{x} \in \mathbb{R}^{n_0}$ is
distributed { independently} as $P_0 (x_i)$, while hidden units $t_{\ell, i}$ at each
successive layer $\bm{t}_\ell \in \mathbb{R}^{n_{\ell}}$ (vectors are column vectors) come from $P_{\ell}
(t_{\ell, i} | \bm{W}_{\ell, i}^\intercal \bm{t}_{\ell-1})$, with $\bm{t}_0 \equiv
\bm{x}$ and $\bm{W}_{\ell, i}$ denoting the $i$-th row of the matrix
of weights
$W_\ell \in \mathbb{R}^{n_\ell \times n_{\ell - 1}}$. In other words
\begin{equation}
    t_{0, i} \equiv x_i \sim P_0 (\cdot), \quad
    t_{1, i} \sim P_1 (\cdot | \bm{W}_{1, i}^\intercal \bm{x}),\quad
    \ldots \quad
    t_{L, i} \sim P_L (\cdot | \bm{W}_{L, i}^\intercal \bm{t}_{L - 1}),
    \label{eq:mlmodel}
\end{equation}
given a set of weight matrices { $\{W_\ell\}_{\ell=1}^L$ and distributions $\{P_\ell\}_{\ell=1}^L$} which
encode possible non-linearities and stochastic noise applied to the { hidden layer variables, and $P_0$ that generates the visible variables}. In particular, for a non-linearity
{ $t_{\ell,i}=\varphi_{\ell}(h,\xi_{\ell,i})$, where $\xi_{\ell,i}\sim P_\xi(\cdot)$} is the stochastic noise { (independent for each $i$)}, we
have $P_{\ell} (t_{\ell, i} | \bm{W}_{\ell, i}^\intercal \bm{t}_{\ell-1}) = \int
{ dP_\xi(\xi_{\ell,i})} \, \delta\big(t_{\ell, i} - \varphi_{\ell}
(\bm{W}_{\ell, i}^\intercal \bm{t}_{\ell-1}, \xi_{\ell,i})\big)$. Model
(\ref{eq:mlmodel}) thus describes a Markov chain which we denote by $\X \to
\T_1 \to \T_2 \to \dots \to \T_L$, with $\T_{\ell}=\varphi_{\ell}(W_{\ell}
\T_{\ell-1}, { \bm{\xi}}_\ell)$, { $\bm{\xi}_\ell=\{{\xi}_{\ell,i}\}_{i=1}^{n_\ell}$}, and the activation function $\varphi_\ell$ applied componentwise.

{\bf Replica formula---} We shall work in the asymptotic
high-dimensional statistics regime where all $\tilde \alpha_\ell
\equiv n_\ell/n_0$ are of order one 
while $n_0\!\to\!\infty$, and make the important assumption that all
matrices $W_\ell$ are orthogonally-invariant random matrices
independent from each other; in other words, each matrix
$W_\ell\!\in\!\mathbb{R}^{n_\ell \times n_{\ell-1}}$ can be decomposed as a
product of three matrices, $W_\ell\!=\!U_\ell S_\ell V_\ell$,
where $U_\ell\!\in\!\mathrm{O}(n_\ell)$ and
$V_\ell\!\in\!\mathrm{O}(n_{\ell-1})$ are independently sampled from the
Haar measure, and $S_\ell$ is a diagonal matrix of
singular values. The main technical tool we use is a formula for
the entropies of the hidden variables,
${H(\T_\ell) = -\eT{\ell} \ln P_{\T_\ell} ({\bm{t}_\ell})}$, and the mutual information between adjacent layers
$I(\T_\ell; \T_{\ell-1}) = H(\T_\ell) + \eTT{\ell}{\ell-1} \ln
P_{\T_\ell | \T_{\ell - 1}} (\bm{t}_\ell | \bm{t}_{\ell-1})$, based on
the heuristic replica method
\cite{mezard_spin_1987,mezard_information_2009,kabashima_inference_2008,manoel_multi-layer_2017}:
\noindent
\begin{claim}[Replica formula] \label{claim1}
Assume model (\ref{eq:mlmodel}) with $L$ layers in the high-dimensional limit with {componentwise activation functions and weight matrices generated from the ensemble described above},
and denote by $\lambda_{W_k}$ the eigenvalues of $W_k^\intercal W_k$. 
Then { for any $\ell\in\{1,\ldots,L\}$ the normalized entropy of $\T_\ell$} is given by the minimum among all
stationary points of the replica potential:
\begin{align}
\label{eq:asymptotic}
    \lim_{n_0 \to \infty} \frac 1{n_0} H({\T_\ell}) = \min
        \extr_{\bm{A}, \bm{V}, \bm{\tilde{A}}, \bm{\tilde{V}}} \phi_\ell
        (\bm{A}, \bm{V}, \bm{\tilde{A}}, \bm{\tilde{V}}),
\end{align}
which depends on $\ell$-dimensional vectors $\bm{A}, \bm{V}, \bm{\tilde{A}},
\bm{\tilde{V}}$, and is written in terms of mutual information $I$ and
conditional entropies $H$ of scalar variables as
\begin{align}
    &\phi_\ell (\bm{A}, \bm{V}, \bm{\tilde{A}}, \bm{\tilde{V}}) = I\Big(t_0; t_0 +
        \frac{\xi_0}{\sqrt{\tilde{A}_1}}\Big) - \frac12 \sum_{k = 1}^{\ell}
        \tilde{\alpha}_{k - 1} \big[\tilde{A}_k V_k + \alpha_k
        A_k \tilde{V}_k - F_{W_k} (A_k V_k)\big] 
        \nonumber \\
    &\kern3em + \, \sum_{k = 1}^{\ell - 1} \tilde{\alpha}_k \Big[
        H(t_k | \xi_k; \tilde{A}_{k + 1}, \tilde{V}_k,
        \tilde{\rho}_k) - \frac12 \log (2 \pi e \tilde{A}^{-1}_{k+ 1})
        \Big] + \tilde{\alpha}_\ell H(t_\ell | \xi_\ell; \tilde{V}_\ell,
        \tilde{\rho}_\ell), 
    \label{eq:phi_multi} 
\end{align}
where $\alpha_k = n_k / n_{k - 1}$, $\tilde{\alpha}_k = n_{k}
/ n_0$, $\rho_k = \int dP_{k - 1}(t) \, t^2$,
$\tilde{\rho}_k = (\mathbb{E}_{\lambda_{W_k}} \lambda_{W_k})\rho_k/{\alpha_{k}}
$, and ${\xi_k\sim{\cal N}(0,1)}$ for $k=0,\ldots,\ell$. In the computation of the conditional entropies in
(\ref{eq:phi_multi}), the scalar $t_k$-variables are generated from
$P(t_0)=P_0(t_0)$ and
\begin{align}
    P(t_{k} | \xi_k; A, V, \rho) &= \mathbb{E}_{\tilde \xi, \tilde{z}}
        \, P_k (t_k + \tilde{\xi} / \sqrt{A} |
        \sqrt{\rho - V} \xi_k + \sqrt{V} \tilde{z}), \quad k=1,\dots,\ell-1,\\
    P(t_{\ell} | \xi_\ell; V, \rho) &=  \mathbb{E}_{\tilde{z}} \, P_\ell(t_\ell |
        \sqrt{\rho - V} \xi_\ell + \sqrt{V} \tilde{z})  ,
\end{align}
where $\tilde{\xi}$ and $\tilde{z}$ are independent ${\cal N}(0,1)$ random variables. 
Finally, the function $F_{W_{k}}(x)$ depends on the
distribution of the eigenvalues $\lambda_{W_\ell}$ following
\begin{equation}
    F_{W_{k}} (x) = \min_{\theta{ \in {\mathbb{R}}}} \; \big\{ 2 \alpha_k \theta + (\alpha_k -
        1) \ln (1 - \theta) + \mathbb{E}_{\lambda_{W_k}} \ln [ x
        \lambda_{W_k} + (1 - \theta) (1 - \alpha_k \theta) ] \big\}.
\end{equation}
\end{claim}
The computation of the entropy in the large dimensional limit, a
computationally difficult task, has thus been reduced to an extremization of
a function of $4 \ell$ variables, that requires evaluating single or
bidimensional integrals. This extremization can be done efficiently by means of
a fixed-point iteration starting from different initial conditions,
as detailed in the Supplementary Material\citesm.
Moreover, a user-friendly Python
package is provided \cite{dnner}, which performs the computation for
different choices of prior $P_0$, activations $\varphi_\ell$ and spectra
$\lambda_{W_\ell}$. Finally, the mutual information between successive
layers $I(\T_\ell; \T_{\ell-1})$ can be obtained from the entropy following the evaluation of an
additional bidimensional integral, see Section \ref{sec:further1} of the
Supplementary Material\citesm.

Our approach in the derivation of (\ref{eq:phi_multi}) builds on
recent progresses in statistical estimation and information theory for
generalized linear models following the application of methods from
statistical physics of disordered systems
\cite{mezard_spin_1987,mezard_information_2009} in communication
\cite{tulino_support_2013}, statistics \cite{Donoho2016} and machine
learning problems
\cite{seung_statistical_1992,engel_statistical_2001}. In particular,
we use advanced mean field theory \cite{opper2001advanced} and the
heuristic replica method
\cite{mezard_spin_1987,kabashima_inference_2008}, along with its
recent extension to multi-layer estimation
\cite{manoel_multi-layer_2017,fletcher_inference_2017}, in order to
derive the above formula (\ref{eq:phi_multi}).  This derivation is
lengthy and thus given in the Supplementary Material\citesm.
In a related contribution, Reeves
\cite{reeves_additivity_2017} proposed a formula for the mutual
information in the multi-layer setting, using heuristic
information-theoretic arguments. As ours, it exhibits layer-wise
additivity, and the two formulas are conjectured to be equivalent.

%

{\bf Rigorous statement---}
We recall the assumptions under which the replica formula of Claim
\ref{claim1} is conjectured to be exact: {\it (i) weight matrices are
  drawn from an ensemble of random orthogonally-invariant matrices,
  (ii) matrices at different layers are statistically independent and
  (iii) layers have a large dimension and respective sizes of adjacent
  layers are such that weight matrices have aspect ratios
  $\{\alpha_k,\tilde \alpha_k\}_{k=1}^\ell$ of order one.}  While we
could not prove the replica prediction in full generality, we stress
that it comes with multiple credentials: (i) for Gaussian prior $P_0$
and Gaussian distributions $P_\ell$, it corresponds to the exact
analytical solution when weight matrices are independent of each other
(see Section \ref{sec:further2} of the Supplementary Material\citesm).
(ii) In the single-layer case with a Gaussian
weight matrix, it reduces to formula (\ref{eq:phi_sl}) in the
Supplementary Material\citesm, which has been recently
rigorously proven for (almost) all activation functions $\varphi$
\cite{barbier_phase_2017}. (iii) In the case of Gaussian distributions
$P_\ell$, it has also been proven for a large ensemble of random
matrices \cite{toappear} and (iv) it is consistent with all the
results of the AMP
\cite{donoho2009message,zdeborova_statistical_2016,GAMP} and VAMP
\cite{rangan_vector_2017} algorithms, and their multi-layer versions
\cite{manoel_multi-layer_2017, fletcher_inference_2017}, known to
perform well for these estimation problems.

In order to go beyond results for the single-layer problem
and heuristic arguments, we prove Claim~\ref{claim1} for the more involved multi-layer case, assuming
Gaussian i.i.d. matrices and two non-linear layers:
\begin{theorem}[Two-layer Gaussian replica formula] \label{th:RS_2layer_main}
	Suppose ${(H1)}$ the input units distribution $P_0$ is separable and has bounded support; ${(H2)}$ the activations
        $\varphi_1$ and $\varphi_2$ corresponding to $P_1 (t_{1, i} |
        \bm{W}_{1, i}^\intercal \bm{x})$ and $P_2 (t_{2, i} | \bm{W}_{2, i}^\intercal
        \bm{t}_1)$ are bounded $\cC^2$ with bounded first and second
        derivatives w.r.t their first argument; and ${(H3)}$ the weight matrices $W_1$, $W_2$
        have Gaussian i.i.d. entries. Then for model \eqref{eq:mlmodel} with two layers $L=2$ the
        high-dimensional limit of the entropy verifies Claim \ref{claim1}. 
\end{theorem}

The theorem, that closes the conjecture presented in
\cite{manoel_multi-layer_2017}, is proven using the adaptive
interpolation method of
\cite{barbier2017stochastic,barbier_phase_2017} in a multi-layer
setting, as first developed in \cite{2017arXiv170910368B}. The lengthy
proof, presented in details in the Supplementary Material\citesm, is of
independent interest and adds further credentials to the replica
formula, as well as offers a clear direction to further
developments. Note that, following the same approximation arguments as
in \cite{barbier_phase_2017} where the proof is given for the
single-layer case, the hypothesis $(H1)$ can be relaxed to the
existence of the second moment of the prior, $(H2)$ can be dropped and
$(H3)$ extended to matrices with i.i.d. entries of zero mean,
$O(1/n_0)$ variance and finite third moment.


\section{Tractable models for deep learning} 
\label{sec:dlmodels}

The multi-layer model presented above can be leveraged to simulate two prototypical
settings of deep supervised learning on synthetic datasets amenable to the replica
tractable computation of entropies and mutual informations.   

\begin{wrapfigure}{R}{0.5\textwidth}
    \vspace{-3ex}
   \includegraphics[width=0.5\textwidth]{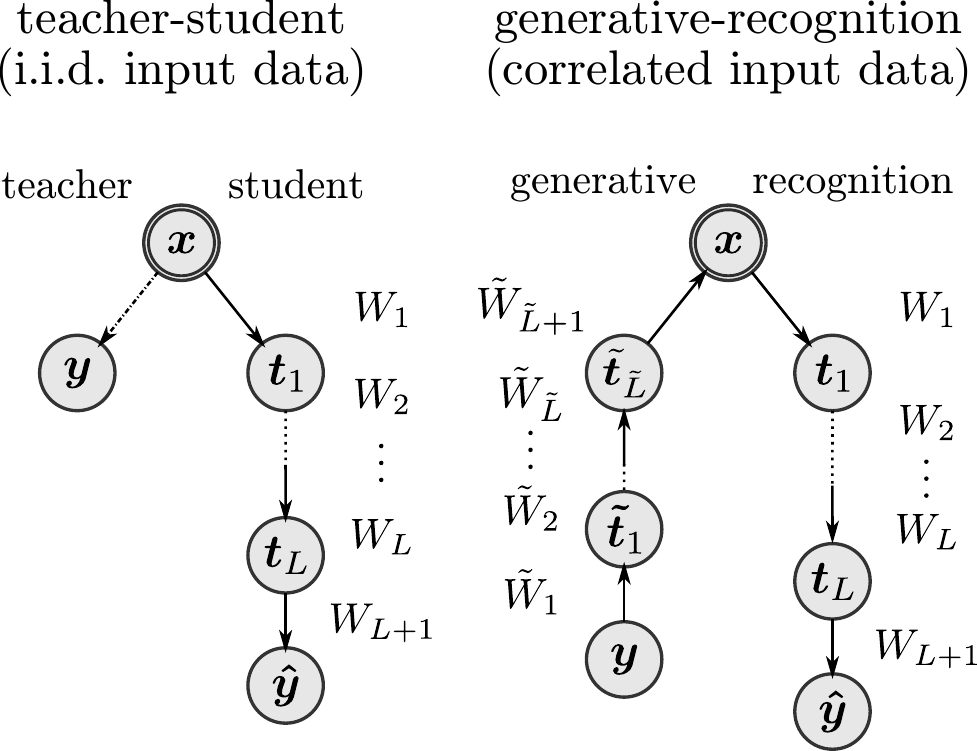}
   \caption{Two models of synthetic data}
   \label{fig:scheme}
    \vspace{-1ex}
\end{wrapfigure}

The first scenario is the so-called {\it teacher-student} (see Figure
\ref{fig:scheme}, left). Here, we assume that the input $\bm{x}$ is
distributed according to a \emph{separable} prior distribution $ P_X(\bm{x}) = \prod_i P_0(x_i)$,
factorized in the components of $\bm{x}$, and the
corresponding label $\bm y$ is given by applying a mapping
${\bm x} \to \bm y$, called {\it the teacher}. After generating a
train and test set in this manner, we 
perform the training of a deep
neural network, {\it the student}, on the synthetic dataset. In
this case, the data themselves have a simple structure given by $P_0$. 

In constrast, the second scenario allows {\it generative models} (see
Figure \ref{fig:scheme}, right) that create more structure, and that
are reminiscent of the {\it generative-recognition} pair of models
of a Variational Autoencoder (VAE). A code vector $\bm{y}$ is sampled
from a separable prior distribution $P_Y(\bm{y}) = \prod_i P_0(y_i)$ and a corresponding data
point $\bm{x}$ is generated by a possibly stochastic neural network,
{\it the generative model}. This setting allows to create input data  $\bm{x}$
featuring correlations, differently from the teacher-student scenario.
The studied supervised learning task then consists in training a deep
neural net, {\it the recognition model}, to recover the code
$\bm{y}$ from $\bm{x}$.

In both cases, the chain going from $\X$ to any later layer is a
Markov chain in the form of \eqref{eq:mlmodel}. In the first scenario,
model \eqref{eq:mlmodel} directly maps to the student network. In the
second scenario however, model \eqref{eq:mlmodel} actually maps to the
feed-forward combination of the generative model followed by the
recognition model. This shift is necessary to verify the assumption
that the starting point (now given by $\Y$) has a separable distribution.
In particular, it generates correlated input data $\X$ while still allowing
for the computation of the entropy of any $\T_\ell$.

At the start of a neural network training, weight matrices initialized as i.i.d.
Gaussian random matrices satisfy the necessary assumptions of the formula
of Claim 1. In their singular value decomposition
\begin{equation}
\label{eq:usv}
W_\ell = U_\ell S_\ell V_\ell
\end{equation}
the matrices $U_\ell \in \mathrm{O}(n_\ell)$ and
$V_\ell \in \mathrm{O}(n_{\ell-1})$, are typical independent samples from the
Haar measure across all layers.
To make sure weight matrices remain close enough to independent
during learning, we define a custom weight constraint which consists in 
keeping $U_\ell$ and
$V_\ell$ fixed while only the matrix $S_\ell$,
constrained to be diagonal, is updated. The number of parameters is thus reduced from $n_\ell \times n_{\ell-1}$ to $\min(n_\ell , n_{\ell-1})$. 
We refer to layers following this weight constraint as USV-layers.
For the replica formula of Claim 1 to be correct, the matrices $S_\ell$ from different layers should furthermore remain uncorrelated during the learning.
In Section \ref{sec:results}, we consider the training of linear
networks for which information-theoretic quantities can be computed
analytically, and confirm numerically that with USV-layers the replica predicted entropy is correct at all times. 
In the following, we assume that is also the case for non-linear networks.

In Section \ref{sec:mnist} of the Supplementary Material\citesm we train a neural network with USV-layers on
a simple real-world dataset (MNIST), showing that these layers can learn to represent complex functions despite their restriction. 
We further note that such a product decomposition is reminiscent of a series of works on adaptative
structured efficient linear layers (SELLs and ACDC)
\cite{Moczulski2015,Yang2015} motivated this time by speed gains, where only diagonal matrices are learned (in these works the matrices $U_\ell$ and
$V_\ell$ are chosen instead as permutations of Fourier or Hadamard matrices,
so that the matrix multiplication can be replaced by fast transforms). 
In Section \ref{sec:results}, we discuss learning experiments with USV-layers on synthetic datasets.  

While we have defined model \eqref{eq:mlmodel} as a stochastic model,
traditional feed forward neural networks are deterministic. In the numerical
experiments of Section \ref{sec:results}, we train and test networks without injecting
noise, and only assume a noise model in the computation of information-theoretic
quantities. Indeed, for continuous variables the presence of
noise is necessary for mutual informations to remain finite (see
discussion of Appendix C in \cite{saxe_information_2018}).  We assume at layer $\ell$ an additive white Gaussian noise
of small amplitude just before passing through its activation function
to obtain $H(\T_\ell)$ and $I(\T_{\ell}; \T_{\ell-1})$, while keeping
the mapping $\X \rightarrow \T_{\ell-1}$ deterministic. This choice
attempts to stay as close as possible to the deterministic neural
network, but remains inevitably somewhat arbitrary (see again
discussion of Appendix C in \cite{saxe_information_2018}).




{\bf Other related works---}
The strategy of studying neural networks models, with random weight matrices
and/or random data, using methods originated in statistical physics
heuristics, such as the replica and the cavity methods
\cite{mezard_spin_1987} has a long history. Before the deep learning era,
this approach led to pioneering results in learning for the Hopfield model
\cite{amit1985storing} and for the random perceptron
\cite{gardner1989three,mezard_space_1989,seung_statistical_1992,engel_statistical_2001}. 

Recently, the successes of deep learning along with the disqualifying
complexity of studying real world problems have sparked a revived interest
in the direction of random weight matrices. Recent results  --without
exhaustivity-- were obtained on the spectrum of the Gram
matrix at each layer using random matrix theory
\cite{louart2017harnessing,pennington2017nonlinear}, on expressivity of deep
neural networks \cite{raghu2016expressive}, on the dynamics of propagation
and learning \cite{saxe2013exact,schoenholz2016deep,
Advani2017,Baldassi11079}, on the high-dimensional non-convex landscape
where the learning takes place \cite{NIPS2014_5486}, or on the universal
random Gaussian neural nets of \cite{giryes2016deep}.

The information bottleneck theory \cite{IB} applied to neural networks
consists in computing the mutual information between the data and the
learned hidden representations on the one hand, and between labels and again
hidden learned representations on the other hand \cite{Tishby2015,
shwartz-ziv_opening_2017}. A successful training should maximize the
information with respect to the labels and simultaneously minimize the
information with respect to the input data, preventing overfitting and
leading to a good generalization.  While this intuition suggests new
learning algorithms and regularizers \cite{Chalk2016, Achille2016,
Alemi2017, Achille2017, Kolchinsky2017, Belghazi2017, Zhao2017}, we can also
hypothesize that this mechanism is already at play in a priori unrelated
commonly used optimization methods, such as the simple stochastic gradient
descent (SGD).  It was first tested in practice by
\cite{shwartz-ziv_opening_2017} on very small neural networks, to allow
the entropy to be estimated by binning of the hidden neurons activities.
Afterwards, the authors of \cite{saxe_information_2018} reproduced the
results of \cite{shwartz-ziv_opening_2017} on small networks using the
continuous entropy estimator of \cite{Kolchinsky2017}, but found that the
overall behavior of mutual information during learning is greatly affected
when changing the nature of non-linearities. Additionally, they investigate
the training of larger linear networks on i.i.d. normally distributed inputs
where entropies at each hidden layer can be computed analytically for an
additive Gaussian noise. The strategy proposed in the present paper allows
us to evaluate entropies and mutual informations in non-linear networks
larger than in \cite{saxe_information_2018, shwartz-ziv_opening_2017}.

\section{Numerical experiments}
\label{sec:results}
We present a series of experiments both aiming at further
validating the replica estimator and leveraging its power in noteworthy
applications. A first application presented in the paragraph 3.1 consists 
in using the replica formula in settings
where it is proven to be rigorously exact as a basis of comparison for other
entropy estimators.  The same experiment also contributes to the discussion of
the information bottleneck theory for neural networks by showing how, without
any learning, information-theoretic quantities have different behaviors for
different non-linearities. In the following paragraph 3.2, we validate the 
accuracy of the replica formula in a
learning experiment with USV-layers ---where it is not proven to be exact ---
by considering the case of linear networks for which information-theoretic
quantities can be otherwise computed in closed-form. We finally consider in
the paragraph 3.3, a
second application testing the information bottleneck theory for large non-linear networks. To this aim, we use the replica estimator to study
compression effects during learning.

{\bf 3.1 Estimators and activation comparisons---} 
Two non-parametric estimators have already been considered by
\cite{saxe_information_2018} to compute entropies and/or mutual
informations during learning. The kernel-density approach of
Kolchinsky et. al. \cite{Kolchinsky2017} consists in fitting a mixture
of Gaussians (MoG) to samples of the variable of interest and
subsequently compute an upper bound on the entropy of the MoG
\cite{Kolchinsky2017a}. The method of Kraskov et
al. \cite{Kraskov2004} uses nearest neighbor distances between samples
to directly build an estimate of the entropy. Both methods require the
computation of the matrix of distances between samples. 
Recently, \cite{Belghazi2017} proposed a new non-parametric estimator for mutual informations which involves the optimization of a neural network to tighten a bound.
It is
unfortunately computationally hard to test how these estimators behave in high
dimension as even for a known distribution the computation of the entropy
is intractable (\#P-complete) in most cases. However the replica method
proposed here is a valuable point of comparison for cases where it is rigorously exact. 

In the first numerical experiment we place ourselves in the setting of Theorem
1: a 2-layer network with i.i.d weight matrices, where the formula of Claim 1
is thus rigorously exact in the limit of large networks, and we compare the
replica results with the non-parametric estimators of \cite{Kolchinsky2017}
and \cite{Kraskov2004}.
Note that the requirement for smooth activations $(H2)$ of Theorem 1 can be relaxed (see discussion below the Theorem).
Additionally, non-smooth functions can be approximated arbitrarily closely by
smooth functions with equal information-theoretic quantities, up to
numerical precision.


We consider a neural network with layers of equal size
$n = 1000$ that we denote: $\X \rightarrow \T_1 \rightarrow \T_2$.
The input variable components are i.i.d. Gaussian with mean 0 and
variance 1. The weight matrices entries are also i.i.d. Gaussian with
mean 0. Their standard-deviation is rescaled by a factor $1/\sqrt{n}$
and then multiplied by a coefficient $\sigma$ varying between $0.1$
and $10$, i.e.  around the recommended value for training
initialization. To compute entropies, we consider noisy versions of
the latent variables where an additive white Gaussian noise of very
small variance ($\sigma^2_{\rm noise}=10^{-5}$) is added right before
the activation function, $\T_1 = f(W_1\X + \bm{\epsilon}_1)$ and
$\T_2 = f(W_2 f(W_1 \X) + {\bm\epsilon}_2)$ with
$\bm{\epsilon}_{1,2} \sim \mathcal{N}(0, \sigma^2_{\rm noise} I_n)$, which is also
done in the remaining experiments to guarantee the mutual
informations to remain finite. The non-parametric estimators \cite{Kolchinsky2017, Kraskov2004} were evaluated
using 1000 samples, as the cost of computing pairwise distances is
significant in such high dimension and we checked that the entropy
estimate is stable over independent draws of a sample of such a size (error bars
smaller than marker size). On Figure \ref{fig:comparisons}, we compare
the different estimates of $H(\T_1)$ and $H(\T_2)$ for different
activation functions: linear, hardtanh or ReLU. The hardtanh
activation is a piecewise linear approximation of the tanh,
$\mathrm{hardtanh}(x) \!=\! - 1$ for $x\!<\!-1$, $x$ for $-1\!<\! x\!< \!1$, and
$1$ for $x\!>\!1$, for which the integrals in the replica formula can be evaluated faster than for the tanh.

In the linear and hardtanh case, the non-parametric methods are following the tendency of the replica estimate when $\sigma$ is varied, but appear to systematically
over-estimate the entropy. For linear networks with Gaussian inputs and additive
Gaussian noise, every layer is also a multivariate Gaussian and therefore
entropies can be directly computed in closed form (\emph{exact} in the plot legend). When using the
Kolchinsky estimate in the linear case we also check the consistency of two
strategies, either fitting the MoG to the noisy sample or fitting the MoG to
the deterministic part of the $\T_\ell$ and augment the resulting variance with
$\sigma^2_{\rm noise}$, as done in \cite{Kolchinsky2017} (\emph{Kolchinsky et al. parametric} in the plot legend). In the network
with hardtanh non-linearities, we check that for small weight values, the
entropies are the same as in a linear network with same weights (\emph{linear
approx} in the plot legend, computed using the exact analytical result for linear networks and therefore plotted in a similar color to \emph{exact}).   Lastly, in the case of the ReLU-ReLU
network, we note that non-parametric methods are predicting an entropy
increasing as the one of a linear network with identical weights,
whereas the replica computation reflects its knowledge of the cut-off and
accurately features a slope equal to half of the linear network entropy (\emph{1/2 linear approx} in the plot legend). While non-parametric estimators are
invaluable tools able to approximate entropies from the mere knowledge of
samples,they inevitably introduce estimation errors.
The replica method is taking the opposite view. While being restricted to a
class of models, it can leverage its knowledge of the neural network structure to provide a reliable estimate. To our knowledge, there is no other entropy estimator able to incorporate such information about the underlying multi-layer model.

Beyond informing about estimators accuracy, this experiment also
unveils a simple but possibly important distinction between
activation functions. For the hardtanh activation, as the random
weights magnitude increases, the entropies decrease after reaching a
maximum, whereas they only increase for the unbounded activation
functions we consider -- even for the single-side saturating
ReLU. This loss of information for bounded activations was also
observed by \cite{saxe_information_2018}, where entropies were
computed by discretizing the output as a single neuron with bins of
equal size. In this setting, as the tanh activation starts to saturate
for large inputs, the extreme bins (at $-1$ and $1$) concentrate more
and more probability mass, which explains the information loss. Here
we confirm that the phenomenon is also observed when computing the
entropy of the hardtanh (without binning and with small noise injected
before the non-linearity). We check via the replica formula that the
same phenomenology arises for the mutual informations
$I(\X; \T_\ell)$ (see Section\ref{sec:micomp}).

\begin{figure}[t!]
      \includegraphics[width=1.01\linewidth]{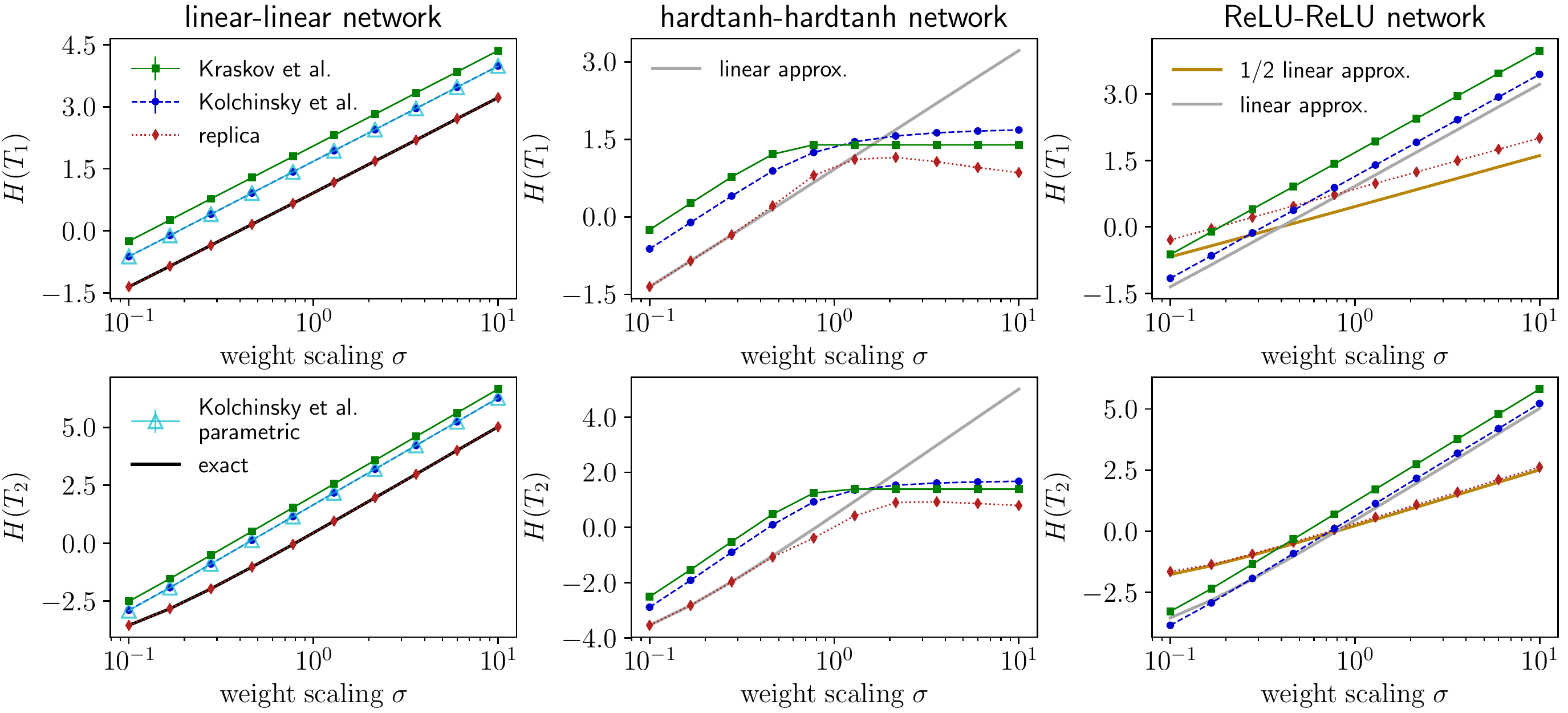}
     \caption{
        Entropy of latent variables in stochastic networks $\X \rightarrow
        \T_1 \rightarrow \T_2 $, with equally sized layers $n=1000$, inputs
        drawn from $\mathcal{N}(0, I_{n})$, weights from  $\mathcal{N}(0,
        \sigma^2 I_{n^2} / n)$, as a function of the weight scaling
        parameter $\sigma$. An additive white Gaussian noise $\mathcal{N}(0,
        10^{-5} I_{n})$ is added inside the non-linearity. Left column: linear
        network. Center column: hardtanh-hardtanh network. Right column:
        ReLU-ReLU network.
       \label{fig:comparisons}}     
\end{figure}

{\bf 3.2 Learning experiments with linear networks---} In the following, and in
Section \ref{sec:moreresults} of the Supplementary Material\citesm, we discuss
training experiments of different instances of the deep learning models
defined in Section \ref{sec:dlmodels}. We seek to study the simplest possible training strategies
achieving good generalization. Hence for all experiments we use plain
stochastic gradient descent (SGD) with constant learning rates,  without
momentum and  without any explicit form of regularization. The sizes of the
training and testing sets are taken equal and scale typically as a few
hundreds times the size of the input layer. Unless otherwise stated, plots
correspond to single runs, yet we checked over a few repetitions that outcomes
of independent runs lead to identical qualitative behaviors. The values of
mutual informations $I(\X;\T_\ell)$ are computed by considering noisy
versions of the latent variables where an additive white Gaussian noise of
very small variance ($\sigma^2_{\rm noise}=10^{-5}$) is added right before the
activation function, as in the previous experiment. This noise is neither
present at training time, where it could act as a regularizer, nor at testing
time. Given the noise is only assumed at the last layer, the second to last
layer is a deterministic mapping of the input variable; hence the replica
formula yielding mutual informations between adjacent layers gives us directly
$I(\T_\ell; \T_{\ell-1}) = H(\T_\ell) - H(\T_\ell| \T_{\ell-1}) = H(\T_\ell) -
H(\T_\ell| \X) = I(\T_\ell; \X)$. We provide a second Python package
\cite{lsd} to implement in Keras learning experiments on synthetic datasets, using USV-
layers and interfacing the first Python package \cite{dnner} for replica
computations.

To start with we consider the training of a linear
network in the teacher-student scenario. The teacher has also to be linear to be learnable:
we consider a simple single-layer network with additive white Gaussian noise, $\Y =
\tilde{W}_{\rm teach} \X + \bm{\epsilon}$, with input $\bm{x} \sim \mathcal{N}(0, I_{n})$ of size $n$, teacher matrix $\tilde{W}_{\rm teach}$ i.i.d. normally distributed 
as $\mathcal{N}(0, 1/n)$ , noise $\bm{\epsilon}
\sim \mathcal{N}(0, 0.01 I_{n})$, and output of size $n_{Y} = 4$. We train a student
network of three USV-layers, plus one fully connected unconstrained layer $\X \rightarrow \T_1 \rightarrow \T_2 \rightarrow \T_3
\rightarrow \hat{\Y}$ on the regression task, using plain SGD for the MSE loss
$(\hat{\Y}-\Y)^2$. We recall that in the USV-layers \eqref{eq:usv} only the diagonal matrix is updated during learning. On the left
panel of Figure \ref{fig:linearnet}, we report the learning curve and the
mutual informations between the hidden layers and the input in the case where
all layers but outputs have size $n=1500$. Again this linear setting is analytically
tractable and does not require the replica formula, a similar situation was studied in
\cite{saxe_information_2018}. In agreement with their observations, we find that
the mutual informations $I(\X;\T_\ell)$ keep on increasing throughout the learning,
without compromising the generalization ability of the student.  Now, we also
use this linear setting to demonstrate (i) that the replica formula remains correct
throughout the learning of the USV-layers and (ii) that the replica method gets
closer and closer to the exact result in the limit of large networks, as
theoretically predicted \eqref{eq:asymptotic}. To this aim, we repeat the experiment for $n$ varying
between $100$ and $1500$, and report the maximum and the mean value of the
squared error on the estimation of the $I(\X;\T_\ell)$ over all epochs of 5
independent training runs. We find that even if errors tend to increase with the
number of layers, they remain objectively very small and decrease drastically as
the size of the layers increases.

\begin{figure}[t!]
    \begin{center}
      \includegraphics[width=1.01\linewidth]{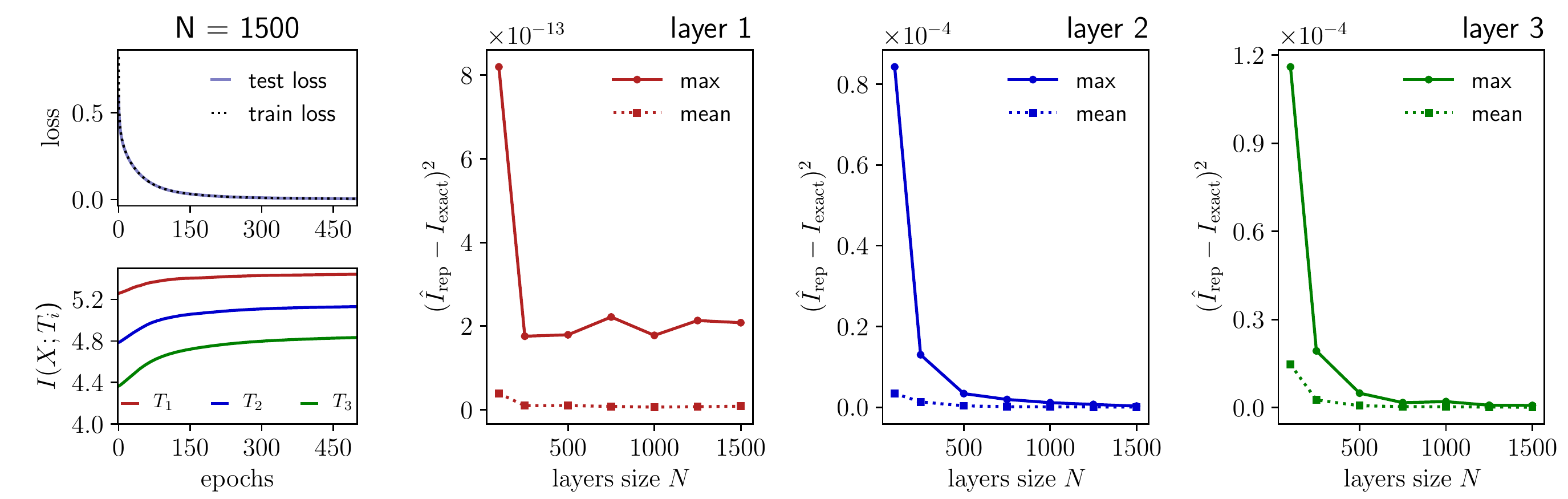}
     \caption{%
        Training of a 4-layer linear student of varying size on a regression
        task generated by a linear teacher of output size $n_{\Y}=4$.
        Upper-left: MSE loss on the training and testing sets during training
        by plain SGD for layers of size $n=1500$. Best training loss is 0.004735, 
        best testing loss is 0.004789. Lower-left: Corresponding
        mutual information evolution between hidden layers and input.
        Center-left, center-right, right: maximum and squared error of the
        replica estimation of the mutual information as a function of layers
        size $n$, over the course of 5 independent trainings for each value
        of $n$ for the first, second and third hidden layer.
       \label{fig:linearnet}}     
    \end{center}
\end{figure}

{\bf 3.3 Learning experiments with deep non-linear networks---}
Finally, we apply the replica formula to estimate mutual informations during the
training of 
non-linear networks on correlated input data.

We consider a simple single layer generative model $\X = \tilde{W}_{\rm gen} \Y +
\bm{\epsilon}$ with normally distributed code $\Y\sim \mathcal{N}(0, I_{n_Y})$ of
size $n_{Y} = 100$, data of size $n_{X} = 500$ generated with matrix $\tilde{W}_{\rm gen}$ i.i.d. normally distributed as $\mathcal{N}(0, 1/n_Y)$  and noise $\bm{\epsilon}
\sim \mathcal{N}(0, 0.01 I_{n_{X}})$. We then train a recognition model to
solve the binary classification problem of recovering the label $y =
\mathrm{sign}(Y_1)$, the sign of the first neuron in $\Y$, using plain
SGD but this time to minimize the cross-entropy loss. Note that the rest of
the initial code $(Y_2, .. Y_{n_{\Y}})$ acts as noise/nuisance with respect to
the learning task. We compare two 5-layers recognition models with 4 USV-
layers plus one unconstrained, of sizes 500-1000-500-250-100-2, and
activations either linear-ReLU-linear-ReLU-softmax (top row of Figure
\ref{fig:classif}) or linear-hardtanh-linear-hardtanh-softmax (bottom row).
Because USV-layers only feature $O(n)$ parameters instead of $O(n^2)$ we
observe that they require more iterations to train in general. In the case of
the ReLU network, adding interleaved linear layers was key to successful
training with 2 non-linearities, which explains the somewhat unusual architecture
proposed. For the recognition model using hardtanh, this was actually not an
issue (see Supplementary Material\citesm~for an experiment using only hardtanh
activations), however, we consider a similar architecture for fair comparison.
We discuss further the ability of learning of USV-layers in the Supplementary
Material\citesm.

This experiment is reminiscent of the setting of \cite{shwartz-ziv_opening_2017}, 
yet now tractable for networks of larger sizes. For both
types of non-linearities we observe that the mutual information between the
input and all hidden layers decrease during the learning, except for the very
beginning of training where we can sometimes observe a short phase of increase
(see zoom in insets). For the hardtanh layers this phase is longer and the
initial increase of noticeable amplitude.

In this particular experiment, the claim of
\cite{shwartz-ziv_opening_2017} that compression can occur during
training even with non double-saturated activation seems corroborated
(a phenomenon that was not observed by
\cite{saxe_information_2018}). Yet we do not observe that the
compression is more pronounced in deeper layers and its link to
generalization remains elusive. For instance, we do not see a delay in
the generalization w.r.t.  training accuracy/loss in the recognition
model with hardtanh despite of an initial phase without compression in
two layers. Further learning experiments, including a second run of this last experiment, are presented in the Supplementary
Material\citesm.

\begin{figure}[t!]
    \begin{center}
      \includegraphics[width=1.01\linewidth]{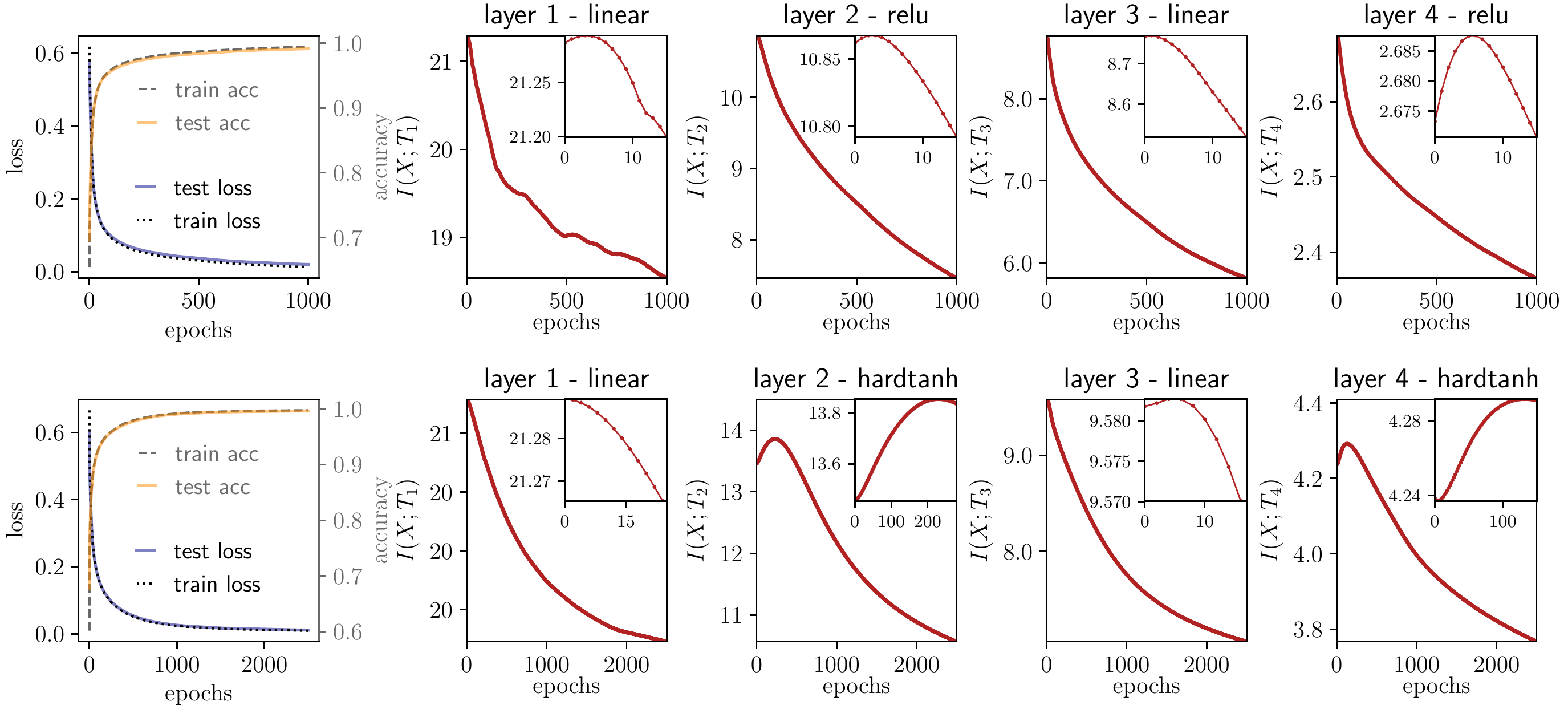}
     \caption{%
     	Training of two recognition models on a binary classification task with correlated input data and either ReLU (top) or hardtanh (bottom) non-linearities. Left: training and generalization cross-entropy loss (left axis) and accuracies (right axis) during learning. Best training-testing accuracies are 0.995 - 0.991 for ReLU version (top row) and 0.998 - 0.996 for hardtanh version (bottom row). Remaining colums: mutual information between the input and successive hidden layers. Insets zoom on the first epochs.
     	\label{fig:classif}}     
    \end{center}
\end{figure}


\section{Conclusion and perspectives}
\label{sec:conclusion}
We have presented a class of deep learning models together with a
tractable method to compute entropy and mutual information between
layers. This, we believe, offers a promising framework for further
investigations, and to this aim we provide Python packages that
facilitate both the computation of mutual informations and the training, 
for an arbitrary implementation of the model.  
In the future, allowing for biases by extending the proposed formula would 
improve the fitting power of the considered neural network models.

We observe in our high-dimensional experiments that compression can
happen during learning, even when using ReLU activations. While we
did not observe a clear link between generalization and compression in our
setting, there are many directions to be further explored within
the models presented in Section \ref{sec:dlmodels}. Studying the
entropic effect of regularizers is a natural step to formulate
an entropic interpretation to generalization. Furthermore, while our
experiments focused on the supervised learning, the replica formula
derived for multi-layer models is general and can be applied in
unsupervised contexts, for instance in the theory of VAEs. On
the rigorous side, the greater perspective remains proving the replica
formula in the general case of multi-layer models, and further confirm
that the replica formula stays true after the learning of the
USV-layers. Another question worth of future investigation is whether
the replica method can be used to describe not only entropies and
mutual informations for learned USV-layers, but also the optimal
learning of the weights itself.

\section*{Acknowledgments}
The authors would like to thank Léon Bottou, Antoine Maillard, Marc Mézard, Léo Miolane,
and Galen Reeves for insightful discussions. This work has been supported by
the ERC under the European Union’s FP7 Grant Agreement 307087-SPARCS and the
European Union's Horizon 2020 Research and Innovation Program 714608-SMiLe, as well as by the French Agence Nationale de la Recherche under grant
ANR-17-CE23-0023-01 PAIL. Additional funding is acknowledged by MG from
“Chaire de recherche sur les modèles et sciences des données”, Fondation CFM
pour la Recherche-ENS; by AM from Labex DigiCosme; and by CL from the Swiss
National Science Foundation under grant 200021E-175541. We gratefully
acknowledge the support of NVIDIA Corporation with the donation of the Titan
Xp GPU used for this research.

\putbib[main]
\end{bibunit}

\clearpage

\begin{center}
    \LARGE Supplementary Material
\end{center}

\setcounter{section}{0}
\addtocontents{toc}{\protect\setcounter{tocdepth}{2}}
\tableofcontents

\begin{bibunit}
\section{Replica formula for the entropy}

\subsection{Background}

The replica method \cite{sherrington_solvable_1975, mezard1990spin} was
first developed in the context of disordered physical systems where the
strength of interactions $J$ are randomly distributed, $J \sim P_J (J)$.
Given the distribution of microstates $\bm{x}$ at a fixed temperature
$\beta^{-1}$, $P (\bm{x} | \beta, J) = \frac{1}{\mathcal{Z} (\beta, J)} \,
e^{-\beta \mathcal{H}_J (\bm{x})}$, one is typically interested in the
average free energy
\begin{equation}
    \mathcal{F} (\beta) = -\lim_{n \to \infty} \frac{1}{\beta n} \mathbb{E}_J
    \log \mathcal{Z} (\beta, J),
    \label{eq:avg_free_energy}
\end{equation}
from which typical macroscopic behavior is obtained. Computing
(\ref{eq:avg_free_energy}) is hard in general, but can be done with the
use of specific techniques.
The replica method in particular employs the following mathematical identity
\begin{equation}
    \mathbb{E}_J \log \mathcal{Z} = \lim_{a \to 0} \, \frac{\mathbb{E}_J
    \mathcal{Z}^a - 1}{a}.
\end{equation}
Evaluating the average on the r.h.s. leads, under the
\emph{replica-symmetry} assumption, to an expression of the form
$\mathbb{E}_J \mathcal{Z}^a = e^{-\beta n\,a \cdot \extr_{\!\bm{q}} \phi (\beta,
\bm{q})}$, where $\phi (\beta, \bm{q})$ is known as the
\emph{replica-symmetric free energy}, and $\bm{q}$ are \emph{order
parameters} related to macroscopic quantities of the system. We then write
$\mathcal{F} (\beta) = \extr_{\!\bm{q}} \phi (\beta, \bm{q})$, so that
computing $\mathcal{F}$ depends on solving the saddle-point equations
$\nabla_{\bm{q}} \phi \big|_{\bm{q}^\ast} = 0$.

Computing (\ref{eq:avg_free_energy}) is of interest in many
problems outside of physics
\cite{nishimori_statistical_2001,mezard2009information}. Early applications
of the replica method in machine learning include the evaluation of the
optimal capacity and generalization error of the perceptron
\cite{gardner_space_1988,
gardner_optimal_1988,mezard_space_1989,seung_statistical_1992,engel_statistical_2001}.
More recently it has also been used in the study of problems in
telecommunications and signal processing, such as channel divison multiple
access \cite{tanaka_statistical-mechanics_2002} and compressed sensing
\cite{rangan_asymptotic_2009,
kabashima_typical_2009,ganguli_statistical_2010,krzakala_statistical-physics-based_2012}.
For a review of these developments see
\cite{zdeborova_statistical_2016}.

These particular examples all share the following common probabilistic
structure
\begin{equation}
    \left\{
        \begin{aligned}
            &\bm{y} \sim P_{Y | Z} (\bm{y} | W \bm{x}), \\
            &\bm{x} \sim P_X (\bm{x}),
        \end{aligned}
    \right.
    \label{eq:glm}
\end{equation}
for fixed $W$ and different choices of $P_{Y | Z}$ and $P_X$; in other
words, they are all specific instances of \emph{generalized linear models}
(GLMs).  Using Bayes theorem, one writes the posterior distribution of
$\bm{x}$ as $P(\bm{x} | W, \bm{y}) = \frac{1}{P(W, \bm{y})} \, P_{Y | Z}
(\bm{y} | W \bm{x}) \, P_X (\bm{x})$; the replica method is then employed to
evaluate the average log-marginal likelihood $\mathbb{E}_{W, \bm{y}} \log
P(W, \bm{y})$, which gives us typical properties of the model. Note this
quantity is nothing but the entropy of $\bm{y}$ given $W$, $H(\bm{y} | W)$.

The distribution $P_J$ (or $P_W$ in the notation above) is usually assumed
to be i.i.d. on the elements of the matrix $J$. However, one can also use
the same techniques to approach $J$ belonging to arbitrary
orthogonally-invariant ensembles. This approach was pioneered by
\cite{marinari_replica_1994,
parisi_mean-field_1995,opper_tractable_2001,cherrier_role_2003}, and in the
context of generalized linear models by
\cite{takeda_analysis_2006,muller_vector_2008,kabashima_inference_2008,shinzato_perceptron_2008,shinzato_learning_2009,
tulino_support_2013,kabashima_signal_2014}.

Generalizing the analysis of (\ref{eq:glm}) to multi-layer models has
first been considered by \cite{manoel_multi-layer_2017} in the context of
Gaussian i.i.d. matrices, and by \cite{fletcher_inference_2017,
reeves_additivity_2017} for orthogonally-invariant ensembles. In particular,
\cite{reeves_additivity_2017} has an expression for the replica free energy
which should be in principle equivalent to the one we present, although its
focus is in the derivation of this expression rather than applications or
explicit computations.

Finally, it is worth mentioning that even though the replica method is
usually considered to be non-rigorous, its results have been proven to be
exact for different classes of models, including GLMs
\cite{talagrand_spin_2003,panchenko2013sherrington,
barbier_mutual_2016,lelarge_fundamental_2016,reeves_replica-symmetric_2016,BarbierOneLayerGLM},
and are widely conjectured to be exact in general. In fact, in section
\ref{sec:partI} we show how to proove the formula in the particular
case of two-layer with Gaussian matrices.

\subsection{Entropy in single/multi-layer generalized linear models}
\label{sec:s2m}

\subsubsection{Single-layer}

For a single-layer generalized linear model
\begin{equation}
    \left\{ \begin{aligned}
        &\bm{x} \sim P_X (\bm{x}), \\
        &\bm{y} \sim P_{Y | Z} (\bm{y} | W \bm{x}).
    \end{aligned} \right.
\end{equation}
with $P_X$ and $P_{Y | Z}$ separable in the components
of $\bm{x} \in \mathbb{R}^n$ and $\bm{y} \in \mathbb{R}^m$, and
$W \in \mathbb{R}^{m \times n}$ Gaussian i.i.d., $W_{\mu i} \sim
\mathcal{N} (0, 1/n)$, define $\alpha = m/n$ and $\rho = \mathbb{E}_x
x^2$. Then the entropy of $\bm{y}$ in the limit $n \to \infty$ is given by \cite{zdeborova_statistical_2016,BarbierOneLayerGLM}
\begin{equation}
    \lim_{n \to \infty} \, n^{-1} H(\bm{y} | W) = \min \extr_{A, V} \phi (A, V),
\end{equation}
where
\begin{equation}
    \phi(A, V) = -\frac{1}{2} A V + I(x; x + \frac{\xi_0}{\sqrt{A}})
        + \alpha H(y | \xi_1; V, \rho),
    \label{eq:phi_sl}
\end{equation}
with $\xi_0$, $\xi_1$ both normally distributed with zero mean and unit
variance, and $P(y | \xi; V, \rho) = \int \mathcal{D}\tilde{z} \, P_{Y | Z} (y |
\sqrt{\rho - V} \xi + \sqrt{V} \tilde{z})$ (here $\mathcal{D}\tilde{z}$ denotes
integration over a standard Gaussian measure).

This can be adapted to orthogonally-invariant ensembles by using the
techniques described in \cite{kabashima_inference_2008}. Let $W = USV^T$,
where $U$ is orthogonal, $S$ diagonal and arbitrary and $V$ is Haar
distributed. We denote by $\pi_W(\lambda_W)$ the distribution of eigenvalues of
$W^T W$, and the second moment of $\bm{z} = W \bm{x}$ by $\tilde{\rho} =
\frac{\mathbb{E} \lambda_W}{\alpha} \rho$. The entropy is then written
as $n^{-1} H (\bm{y} | W) = \min \extr_{A, V, \tilde{A}, \tilde{V}} \phi (A, V,
\tilde{A}, \tilde{V})$, where
\begin{equation}
    \phi(A, V, \tilde{A}, \tilde{V}) = -\frac{1}{2} \big( \tilde{A} V
        + \alpha A \tilde{V} - F_W(AV) \big) + I(x; x +
        \frac{\xi_0}{\sqrt{\tilde{A}}}) + \alpha H(y | \xi_1; \tilde{V},
        \tilde{\rho}),
    \label{eq:phi_sl-kabashima}
\end{equation}
and
\begin{equation}
    F_W (x) = \min_{\theta} \; \big\{ 2 \alpha \theta + (\alpha -
        1) \log (1 - \theta) + \mathbb{E}_{\lambda_W} \log [ x \lambda_W +
        (1 - \theta) (1 - \alpha \theta) ] \big\}.
    \label{eq:F}
\end{equation}

If the matrix is Gaussian i.i.d., $\pi_W (\lambda_W)$ is Marchenko-Pastur
and $F_W(AV) = \alpha A V$. Extremizing over $A$ gives $\tilde{V} =
V$, so that (\ref{eq:phi_sl}) is recovered. In this precise case, it
has been proven rigorously in \cite{BarbierOneLayerGLM}.

\subsubsection{Multi-layer}

Consider the following multi-layer generalized linear model
\begin{equation}
    \left\{
        \begin{aligned}
            &t_{0,i} \equiv x_i \sim P_0 (x_i), \\
            &t_{1,i} \sim P_1 (t_{1,i} | W_1 \bm{x}), \\
            &t_{2,i} \sim P_2 (t_{2,i} | W_2 \bm{t}_{1}), \\
            &\vdots \\
            &t_{L,i} \equiv y_i \sim P_L (y | W_L \bm{t}_{L - 1}),
        \end{aligned}
    \right.
\end{equation}
where the $W_\ell \in \mathbb{R}^{n_{\ell} \times n_{\ell - 1}}$ are fixed,
and the $i$ index runs from $1$ to $n_\ell$. Using Bayes' theorem we can
write
\begin{equation}
    P (\bm{t}_0 | \bm{t}_L, \ul{W}) \!=\! \frac{1}{P(\bm{t}_L, \ul{W})}
        \int \prod_{\ell = 1}^{L - 1} d\bm{t}_\ell \prod_{\ell = 1}^L
        P(\bm{t}_\ell | W_\ell \bm{t}_{\ell - 1}) P(\bm{t}_0).
\end{equation}
with $\ul{W} = \{W_\ell\}_{\ell = 1, \dots, L}$.
Performing posterior inference requires one to evaluate the marginal
likelihood
\begin{equation}
    P (\bm{t}_L, \ul{W}) = \int \prod_{\ell = 0}^{L - 1} d\bm{t}_\ell \,
    \prod_{\ell = 1}^L P(\bm{t}_\ell | W_\ell \bm{t}_{\ell - 1}) \, P(\bm{t}_0),
\end{equation}
which is in general hard to do.
Our analysis employs the framework introduced in
\cite{manoel_multi-layer_2017} to compute the entropy of $\bm{t}_L$ in the
limit $n_0 \to \infty$ with $\tilde{\alpha}_\ell = n_\ell / n_0$ finite for
$\ell = 1, \dots, L$
\begin{equation}
    \lim_{n_0 \to \infty} n_0^{-1} H(\bm{t}_L | \ul{W}) = \min
    \extr_{\bm{A}, \bm{V}, \bm{\tilde{A}}, \bm{\tilde{V}}} \phi (\bm{A},
    \bm{V}, \bm{\tilde{A}}, \bm{\tilde{V}}),
\end{equation}
with the replica potential $\phi$ given by
\begin{align}
    &\phi (\bm{A}, \bm{V}, \bm{\tilde{A}}, \bm{\tilde{V}}) = -\frac12
        \sum_{\ell = 1}^L \tilde{\alpha}_{\ell - 1} \big[\tilde{A}_\ell V_\ell +
        \alpha_\ell A_\ell \tilde{V}_\ell - F_{W_\ell} (A_\ell V_\ell)\big] +
        I(t_0; t_0 + \frac{\xi_0}{\sqrt{\tilde{A}_1}}) \, + \label{eq:phi_ml} \\
    &\kern3em +\, \sum_{\ell = 1}^{L - 1} \tilde{\alpha}_\ell \bigg[
        H(t_\ell | \xi_\ell; \tilde{A}_{\ell + 1}, \tilde{V}_{\ell},
        \tilde{\rho}_\ell) - \frac12 \log (2 \pi e \tilde{A}^{-1}_{\ell + 1})
        \bigg] +\tilde{\alpha}_L H(t_L | \xi_L; \tilde{V}_L,
        \tilde{\rho}_L).  \notag
\end{align}
and the $\bm{\xi}$ normally distributed with zero mean and unit
variance. The $t_\ell$ in the expression above are distributed as
\begin{align}
    P(t_\ell | \xi_\ell; A, V, \rho) &= \int \mathcal{D}\tilde{\xi} \mathcal{D}\tilde{z} \, P_\ell
        (t_\ell + \sqrt{1 / A} \tilde{\xi} | \sqrt{\rho - V} \xi_\ell + \sqrt{V}
        \tilde{z}), \\
    \label{eq:ptxi}
    P(t_L | \xi_L; V, \rho) &= \int \mathcal{D}\tilde{z} \, P_L(t_L |
        \sqrt{\rho - V} \xi_L + \sqrt{V} \tilde{z}).
\end{align}
where $\int \mathcal{D}z \, (\cdot) = \int dz \, \mathcal{N} (z; 0, 1) \, (\cdot)$
denotes the integration over the standard Gaussian measure.

\subsection{A simple heuristic derivation of the multi-layer formula}

Formula (\ref{eq:phi_ml}) can be derived using a simple argument. Consider
the case $L = 2$, where the model reads
\begin{equation}
    \left\{ \begin{aligned}
        &\bm{t}_0 \sim P_0 (\bm{t}_0), \\
        &\bm{t}_1 \sim P_1 (\bm{t}_1 | W_1 \bm{t}_0), \\
        &\bm{t}_2 \sim P_2 (\bm{t}_2 | W_2 \bm{t}_1),
    \end{aligned} \right.
\end{equation}
with $\bm{t}_\ell \in \mathbb{R}^{n_\ell}$ and $W \in \mathbb{R}^{n_\ell
\times n_{\ell - 1}}$.  For the problem of estimating $\bm{t}_1$ given the
knowledge of $\bm{t}_2$, we compute $\lim_{n_1 \to \infty} n_1^{-1}
H(\bm{t}_2 | W_1)$ using the replica free energy
(\ref{eq:phi_sl-kabashima})
\begin{align}
    &\phi(A_2, V_2, \tilde{A}_2, \tilde{V}_2) \!=\!
        -\frac{1}{2} \big( \tilde{A}_2 V_2 + \alpha_2 A_2 \tilde{V_2} - F_{W_2}
        (A_2 V_2) \big) \, + \\
    &\kern5em + I(t_1; t_1 + \frac{\tilde{\xi}_1}{\sqrt{\tilde{A_2}}}) +
        \alpha_2 H(t_2 | \xi_2; \tilde{V}_2, \tilde{\rho}_2).
    \label{eq:deriv0}
\end{align}
Note that
\begin{equation}
    \begin{aligned}
        I(t_1; t_1 + \frac{\tilde{\xi}_1}{\sqrt{\tilde{A}_2}}) &= H\big( t_1 +
            \frac{\tilde{\xi}_1}{\sqrt{\tilde{A_2}}}\big) -
            H\big(\frac{\tilde{\xi}_1}{\sqrt{\tilde{A}_2}}\big) \\
        &= H\big( t_1 + \frac{\tilde{\xi}_1}{\sqrt{\tilde{A_2}}}\big) 
            - \frac{1}{2} \log(2 \pi e \tilde{A}^{-1}_2).
    \end{aligned}
    \label{eq:deriv1}
\end{equation}
Moreover, $H(t_1 + \tilde{\xi}_1 / \sqrt{\tilde{A}}_2)$ can be obtained
from the replica free energy of another problem: that of estimating
$\bm{t}_0$ given the knowledge of (noisy) $\bm{t}_1$, which can again be
written using (\ref{eq:phi_sl-kabashima})
\begin{equation}
    \lim_{n_0 \to \infty} n_0^{-1} \, H(t_1 +
        \frac{\tilde{\xi}_1}{\sqrt{\tilde{A}_2}}) = \min \extr_{A_1, V_1,
        \tilde{A}_1, \tilde{V}_1} \phi_1(A_1, V_1, \tilde{A}_1, \tilde{V}_1),
    \label{eq:deriv2}
\end{equation}
with
\begin{align}
    &\phi_1(A_1, V_1, \tilde{A}_1, \tilde{V}_1) \!=\!
        -\frac{1}{2} \big( \tilde{A}_1 V_1 + \alpha_1 A_1 \tilde{V_1} - F_{W_1}
        (A_1 V_1) \big) \, + \\
    &\kern5em + I(t_0; t_0 + \frac{\xi_0}{\sqrt{\tilde{A_1}}}) + \alpha_1
        H(t_1 | \xi_1; \tilde{A}_1, \tilde{V}_1, \tilde{\rho}_1),
    \label{eq:deriv3}
\end{align}
and the noise $\tilde{\xi}_1$ being integrated in the computation of $H(t_1 |
\xi_1)$, see (\ref{eq:ptxi}). Replacing (\ref{eq:deriv1})-(\ref{eq:deriv3})
in (\ref{eq:deriv0}) gives our formula (\ref{eq:phi_ml}) for $L = 2$;
further repeating this procedure allows one to write the equations for
arbitrary $L$.

\subsection{Formulation in terms of tractable integrals}
\label{sec:entropy}

While expression (\ref{eq:phi_ml}) is more easily written in terms of
conditional entropies and mutual informations, evaluating it requires us
to explicitely state it in terms of integrals, which we do below.
We first consider the Gaussian i.i.d.  In this case, the multi-layer
formula was derived with the cavity and replica method  by
\cite{manoel_multi-layer_2017}, and we shall use their results
here. Assuming that
$W_\ell \in \mathbb{R}^{n_{\ell}
\times n_{\ell - 1}}$ such that $W_{\ell,\mu i} \sim \mathcal{N} (0, 1 /
n_{\ell - 1})$ and using the replica formalism, Claim 1 from the main text
becomes, in this case
\begin{equation}
\lim_{n_0 \to \infty} n_0^{-1} H(\bm{t}_L | \ul{W})
= \min \extr_{\bm{A}, \bm{V}} \phi (\bm{A}, \bm{V}),
\end{equation}    
with the replica potential $\phi$ evaluated from
\begin{equation}
    \phi (\bm{A}, \bm{V}) = \frac{1}{2} \sum_{\ell = 1}^L
    \tilde{\alpha}_{\ell} A_\ell (\rho_\ell - V_\ell) - \mathcal{K}
    (\bm{A}, \bm{V}, \bm{\rho}),
    \label{eq:phi_mlamp}
\end{equation}
and
\begin{equation}
    \mathcal{K} (\bm{A}, \bm{V}, \bm{\rho}) = K_0 (A_1) +
        \sum_{\ell = 1}^{L - 1} \tilde{\alpha}_\ell K_\ell (A_{\ell + 1},
        V_\ell, \rho_\ell) + \tilde{\alpha_L} K_L (V_L, \rho_L).
    \label{eq:K}
\end{equation}
The constants $\alpha_\ell$, $\tilde{\alpha_\ell}$ and $\rho_\ell$ are
defined as following\footnote{Note that due to the central limit theorem,
$\rho_\ell$ can be evaluated from $\rho_{\ell - 1}$ using $\rho_\ell = \int
dt dz \, P_\ell(t | z) \mathcal{N} (z; 0, \rho_{\ell - 1}) \, t^2$.}:
$\alpha_\ell = n_\ell / n_{\ell - 1}$, $\tilde{\alpha}_\ell = n_{\ell} / n_0$,
$\rho_\ell = \int dt \, P_{\ell - 1}(t) \, t^2$.  Moreover
\begin{equation}
    K_\ell (A, V, \rho) = \mathbb{E}_{b, t, z, w | A, V, \rho} \, \log Z_\ell (A, b,
        V, w),
    \label{eq:Is1}
\end{equation}
for $1 \le \ell \le L-1$, and
\begin{equation}
    \begin{aligned}
        &K_0 (A) = \mathbb{E}_{b, x | A} \, \log Z_0 (A, b), \\
        &K_L (V, \rho) = \mathbb{E}_{y, z, w | V, \rho} \log Z_L (y, V, w).
    \end{aligned}
    \label{eq:Is2}
\end{equation}

where 
\begin{flalign}
        &Z_0 (A, B) \!=\! {\textstyle\int} dx \, P_0 (x) e^{-\frac12 A x^2 + B x}, \notag \\
        &Z_\ell (A, B, V, \omega) \!=\! {\textstyle\int} dt dz \, P_\ell(t | z)
            \mathcal{N} (z; \omega, V) e^{-\frac12 A t^2 + Bt}, \notag \\
        &Z_L (y, V, \omega) \!=\! {\textstyle\int} dz \, P_L(y | z) \mathcal{N} (z; \omega, V).
\end{flalign}
and the measures over which expectations are computed are
\begin{flalign}
        &p_0 (b, x ; A) \!=\! P_0 (x) \mathcal{N} (b; Ax, A), \notag \\
        &p_\ell (b, t, z, w ; A, V, \rho) \!=\! P_\ell (t | z) \mathcal{N} (b; At, A)
            \mathcal{N} (z; w, V) \mathcal{N} (w, 0, m), \label{eq:measures} \\
        &p_L (y, z, w ; V, \rho) \!=\! P_L (y | z) \mathcal{N} (z; w, V)
            \mathcal{N} (w; 0, \rho\!-\!V). \notag
\end{flalign}

We typically pick the likelihoods $P_\ell$ so that $Z_\ell$ can be
computed in closed-form, which allows for a number of activation functions
-- linear, probit, ReLU etc. However, our analysis is quite general and can
be done for arbitrary likelihoods, as long as evaluating (\ref{eq:Is1}) and
(\ref{eq:Is2}) is computationally feasible.

Finally, the replica potential above can be generalized to the
orthogonally-invariant case using the framework of
\cite{kabashima_inference_2008}, which we have described in
subsection \ref{sec:s2m}
\begin{equation}
    \phi (\bm{A}, \bm{V}, \tilde{\bm{A}}, \tilde{\bm{V}}) =
        \frac12 \sum_{\ell = 1}^L \tilde{\alpha}_{\ell - 1}
        \big[\tilde{A}_\ell (\rho_\ell - V_\ell) -
    \alpha_\ell A_\ell \tilde{V}_\ell + F_{W_\ell} (A_\ell
        V_\ell)\big] - \mathcal{K} (\tilde{\bm{A}}, \tilde{\bm{V}}; \bm{\tilde{\rho}}).
    \label{eq:phi_oi}
\end{equation}
If the matrix $W_\ell$ is Gaussian i.i.d., the distribution of eigenvalues
of $W^T_\ell W_\ell$ is Marchenko-Pastur and one gets $F_{W_\ell} (A_\ell
V_\ell) = \alpha_\ell A_\ell V_\ell$, $\tilde{A}_\ell = \alpha_\ell A_\ell$,
$\tilde{V}_\ell = V_\ell$, so that (\ref{eq:phi_mlamp}) is recovered.
Moreover, for $L = 1$, one obtains the replica free energy proposed by
\cite{kabashima_inference_2008,shinzato_perceptron_2008,shinzato_learning_2009}.

\subsubsection{Recovering the formulation in terms of conditional entropies}

One can rewrite the formulas above in a simpler way. By manipulating the
measures (\ref{eq:measures}) one obtains
\begin{equation}
    K_0 (A, \rho) = -I(x; b) + \frac{1}{2} A \rho,
\end{equation}
for $x \sim P_0(x)$ and $b \sim \mathcal{N} (b; Ax, A)$. Introducing a
standard normal variable $\xi_0$ and using the invariance of mutual
informations, this can be written as
\begin{equation}
    K_0 (A, \rho) = -I(x; x + \sqrt{1/A} \xi_0) + \frac{1}{2} A \rho.
\end{equation}

Similarly
\begin{equation}
    K_L (V, \rho) = -H(y | w; V),
\end{equation}
for $P(y | w; V) = \int dz P_L (y | z) \mathcal{N} (z; w, V)$ and $P(w;
V, \rho) = \mathcal{N} (w; 0, \rho - V)$. Introducing standard normal
$\xi_L$ \begin{equation}
    K_L(V, \rho) = -H(y | \xi_L; V, \rho).
\end{equation}
where 
\begin{equation}
    P(y | \xi_L; V, \rho) = \int \mathcal{D}\tilde{z} \, P_L(y | \sqrt{\rho - V}
        \xi_L + \sqrt{V} \tilde{z}),
\end{equation}
and $\int \mathcal{D}\tilde{z} \, (\cdot) = \int d\tilde{z} \, \mathcal{N} (z; 0, 1) \,
(\cdot)$ denotes integration over the standard Gaussian measure.

Finally, for the $K_\ell$
\begin{equation}
    K_\ell (A, V, \rho) = -H(b | w; A, V, \rho) + \frac12 A \rho +
    \frac{1}{2} \log(2 \pi e A^{-1}),
\end{equation}
for $P(b | w; A, V) = \int dt dz \, \mathcal{N} (b; At, A) P_\ell(t | z)
\mathcal{N} (z; w, V)$ and $P(w; V, \rho) = \mathcal{N} (w; 0, \rho - V)$.
Introducing standard normal $\xi_\ell$
\begin{equation}
    K_\ell (A, V, \rho) = -H(t_\ell | \xi_\ell; A, V) + \frac{1}{2} A
        \rho + \frac{1}{2} \log (2 \pi e A^{-1})\,,
\end{equation}
where
\begin{equation}
    P(t_\ell | \xi_\ell; A, V, \rho) = \int \mathcal{D}\tilde{\xi} \mathcal{D}\tilde{z} \,
        P_\ell (t_\ell + \sqrt{1 / A} \tilde{\xi} | \sqrt{\rho - V} \xi_\ell +
        \sqrt{V} \tilde{z}).
\end{equation}

We can then rewrite (\ref{eq:K}) as
\begin{equation}
    \begin{aligned}
        &\mathcal{K} (\bm{A}, \bm{V}, \bm{\rho}) = \frac12 \sum_{\ell = 1}^L
            \tilde{\alpha}_{\ell - 1} A_\ell \rho_\ell - I(t_0; t_0 +
            \frac{\xi_0}{\sqrt{A_1}}) - \, \\
        &\kern2em -\, \sum_{\ell = 1}^{L - 1} \tilde{\alpha}_\ell \bigg[
            H(t_\ell | \xi_\ell; A_{\ell + 1}, V_\ell, \rho_\ell) - \frac12 \log
            (2 \pi e A^{-1}_{\ell + 1}) \bigg] -\tilde{\alpha}_L H(t_L | \xi_L; V_L, \rho_L).
    \end{aligned}
\end{equation}
Replacing in (\ref{eq:phi_mlamp}) yields
\begin{equation}
    \begin{aligned}
        &\phi (\bm{A}, \bm{V}) = -\frac12 \sum_{\ell = 1}^L
            \tilde{\alpha}_{\ell - 1} A_\ell V_\ell + I(t_0; t_0 +
            \frac{\xi_0}{\sqrt{A_1}}) + \, \\
        &\kern2em +\, \sum_{\ell = 1}^{L - 1} \tilde{\alpha}_\ell \bigg[
            H(t_\ell | \xi_\ell; A_{\ell + 1}, V_\ell, \rho_\ell) - \frac12
            \log (2 \pi e A^{-1}_{\ell + 1}) \bigg] +\tilde{\alpha}_L H(t_L | \xi_L;
            V_L, \rho_L).
    \end{aligned}
\end{equation}

\subsection{Solving saddle-point equations}

In order to deal with the extremization problem in 
\begin{equation}
    \lim_{n_0 \to \infty} n_0^{-1} H (\bm{t}_L | \ul{W}) = \min \extr_{\bm{A},
        \bm{V}, \tilde{\bm{A}}, \tilde{\bm{V}}} \phi (\bm{A}, \bm{V},
        \tilde{\bm{A}}, \tilde{\bm{V}}),
\end{equation}
one needs to solve the saddle-point equations $\nabla_{\{\bm{A}, \bm{V},
\tilde{\bm{A}}, \tilde{\bm{V}}\}} \phi = 0$. In what follows we propose two
different methods to do that: a fixed-point iteration, and the state evolution
of the ML-VAMP algorithm \cite{fletcher_inference_2017}.

\subsubsection{Method 1: fixed-point iteration}

We first introduce the following function, which is related to the
derivatives of $F_{W_\ell}$
\begin{equation}
    \psi_\ell (\theta, \gamma) = 1 - \gamma \big[ \mathcal{S}_\ell \big( -\gamma^{-1}
        (1 - \theta) (1 - \alpha_\ell \theta) \big) \big]^{-1},
\end{equation}
where $\mathcal{S}_\ell (z) = \mathbb{E}_{\lambda_\ell}
\frac{1}{\lambda_\ell - z}$ is the Stieltjes transform of $W_\ell^T W_\ell$,
see e.g. \cite{tulino_random_2004}. In our experiments we have evaluated
$\mathcal{S}$ approximately by using the empirical distribution of
eigenvalues.

The fixed point iteration consist in looping through layers $L$ to $1$,
first computing the $\vartheta_\ell$ which minimizes (\ref{eq:F}), and
$\tilde{V}_\ell$
\begin{equation}
    \begin{aligned}
    &\vartheta_\ell^{(t)} = \argmin_{\theta} \left[
        \theta - \psi_\ell (\theta, A_\ell^{(t)} V_\ell^{(t)}) \right]^2, \\
    &\tilde{V}_\ell^{(t)} = \vartheta_\ell^{(t)} / A_\ell^{(t)},
    \end{aligned}
    \label{eq:fp1}
\end{equation}
then $A_\ell^{(t + 1)}$, which for layers $1 \leq \ell \leq L-1$ comes from
\begin{equation}
    A_\ell^{(t + 1)} = -\mathbb{E}_{b, t, z, w | \tilde{\rho}_\ell,
    \tilde{A}_{\ell + 1}, \tilde{V}_{\ell}} \partial_w^2 \log Z_\ell
    (\tilde{A}_{\ell + 1}^{(t)}, b, \tilde{V}_\ell^{(t)}, w),
    \label{eq:fp2_1}
\end{equation}
and for the $L$-th layer, from
\begin{equation}
    A_L^{(t + 1)} = -\mathbb{E}_{y, z, w | \tilde{\rho}_L, \tilde{V}_{L}}
    \partial_w^2 \log Z_L (y, \tilde{V}_{L}, w).
    \label{eq:fp2_2}
\end{equation}
Finally, we recompute $\vartheta_\ell$ using $A_\ell^{(t + 1)}$, and
$\tilde{A}_\ell$
\begin{equation}
    \begin{aligned}
    &\vartheta_\ell^{(t + \frac12)} = \argmin_{\theta} \left[
        \theta - \psi_\ell (\theta, A_\ell^{(t + 1)} V_\ell^{(t)}) \right]^2, \\
    &\tilde{A}_\ell^{(t)} = \alpha_\ell \vartheta_\ell^{(t + \frac12)} /
        V_\ell^{(t)}.
    \end{aligned}
    \label{eq:fp3}
\end{equation}
and move on to the next layer. After these quantities are computed for all
layers, we compute all the $V_\ell$; for $2 \leq \ell \leq L$
\begin{equation}
    V_\ell^{(t + 1)} = \mathbb{E}_{b, t, z, w | \tilde{\rho}_{\ell - 1},
    \tilde{A}_{\ell}, \tilde{V}_{\ell - 1}} \partial_b^2 \log Z_\ell
    (\tilde{A}_{\ell}^{(t)}, b, \tilde{V}_{\ell - 1}^{(t)}, w),
    \label{eq:fp4_1}
\end{equation}
and for the 1st layer
\begin{equation}
    V_1^{(t + 1)} = \mathbb{E}_{b, x | \tilde{A}_1} \partial_b \log Z_1
        (\tilde{A}_{1}^{(t)}, b).
    \label{eq:fp4_2}
\end{equation}

This particular order has been chosen so that if $W_\ell$ is Gaussian
i.i.d., $\theta_\ell^{(t)} = A_\ell^{(t)} V_\ell^{(t)}$ and one recovers the
state evolution equations in \cite{manoel_multi-layer_2017}.

\begin{algorithm}
\caption{Compute entropy $H(\bm{y}_L | \ul{W})$}
\label{alg1}
\begin{algorithmic}
    \REQUIRE $\{\bm{A}^{(1)}_i, \bm{V}^{(1)}_i\}_{i = 1}^{n_{\rm init}}$,
        $\epsilon$, $t_{\rm max}$
    \FOR[loop through initial conditions]{$i = 1 \to n_{\rm init}$}
        \STATE $t \leftarrow 0$
        \WHILE[at each time step \dots]{$D < \epsilon$ or $t < t_{\rm max}$}
            \FOR[\dots loop through layers]{$\ell = L \to 1$}
                \STATE compute $\vartheta_\ell^{(t)}$, $\tilde{V}_\ell^{(t)}$ using (\ref{eq:fp1})
                \STATE compute $A_\ell^{(t + 1)}$ using (\ref{eq:fp2_1}) or (\ref{eq:fp2_2})
                \STATE compute $\vartheta_\ell^{(t + \frac12)}$, $\tilde{A}_\ell^{(t)}$ using (\ref{eq:fp3})
            \ENDFOR
            \STATE compute $V_\ell^{(t + 1)} \, \forall \ell$ using (\ref{eq:fp4_1}) or (\ref{eq:fp4_2})
            \STATE $D \leftarrow \sum_\ell |V_\ell^{(t + 1)} - V_\ell^{(t)}|$
            \STATE $t \leftarrow t + 1$
        \ENDWHILE
        \STATE $H_i \leftarrow \phi(\bm{A}^{(t)}, \bm{V}^{(t)},
            \bm{\tilde{A}}^{(t)}, \bm{\tilde{V}}^{(t)})$
    \ENDFOR
    \STATE \textbf{return} $\min_i H_i$
\end{algorithmic}
\end{algorithm}

The set of initial conditions is picked so as to cover the basin of
attraction of typical fixed points. In our experiments we have chosen
$(A^{(0)}_{i, \ell}, V^{(0)}_{i, \ell}) \in \{(\rho_\ell^{-1}, \rho_\ell),
(\delta^{-1}, \delta)\}$, with $\delta = 10^{-10}$.

\subsubsection{Method 2: ML-VAMP state evolution}

While the fixed-point iteration above works well in most cases, it is not
provably convergent. In particular, it relies on a solution for $\theta =
\psi_\ell (\theta, A_\ell^{(t)} V_\ell^{(t)})$ being found, which might not
happen throughout the iteration.

An alternative is to employ the state evolution (SE) of the ML-VAMP
algorithm \cite{fletcher_inference_2017}, which leads to the same fixed
points as the scheme above under certain conditions. Let us first look at
the single-layer case; the ML-VAMP SE equations read
\begin{align}
    &A^+_x = \frac{1}{V^+_x (A^-_x)} - A^-_x, &&\qquad
    A^+_z = \frac{1}{V^+_z (A^+_x, 1 / A_z^-)} - A^-_z, \\
    &A^-_x = \frac{1}{V^-_x (A^+_x, 1 / A_z^-)} - A^+_x, &&\qquad
    A^-_z = \frac{1}{V^-_z (A^+_z)} - A^+_z, \label{eq:gvamp_se4}
\end{align}
where
\begin{align}
    &V^+_x (A) = \mathbb{E}_{x, z} \partial_B^2 \log Z_0 (A, Ax + \sqrt{A}z), \label{eq:gvamp_v1} \\
    &V^+_z (A, \sigma^2) = \sigma^2 \lim_{m \to \infty} \, \frac{1}{m} \operatorname{Tr}
    \big[ \Phi (\Phi^T \Phi + A \sigma^2)^{-1} \Phi^T \big] = \alpha^{-1} \sigma^2 \big(1 - A \sigma^2 \, \mathcal{S}
    (-A \sigma^2)\big), \\
    &V^-_x (A, \sigma^2) = \sigma^2 \lim_{n \to \infty} \, \frac{1}{n} \operatorname{Tr}
    \big[ (\Phi^T \Phi + A \sigma^2)^{-1} \big] = \sigma^2 \, \mathcal{S}
    (-A \sigma^2), \\
    &V^-_z (A) = \frac{1}{A} + \frac{1}{A^2} \, \underbrace{\mathbb{E}_{y,
        w, z} \partial_w^2 \log Z_1 (y, w, 1 / A)}_{-\bar{g} (A)}.
        \label{eq:gvamp_v4}
\end{align}
Combining (\ref{eq:gvamp_se4}) and (\ref{eq:gvamp_v4}) yields
\begin{equation}
    1 / A_z^- = \frac{1}{\bar{g} (A_z^+)} - \frac{1}{A_z^+}.
\end{equation}

At the fixed points
\begin{equation}
    V_x \equiv V_x^+ = V_x^- = \frac{1}{A_x^+ + A_x^-}\,, \qquad
    V_z \equiv V_z^+ = V_z^- = \frac{1}{A_z^+ + A_z^-}\,,
\end{equation}
as well as
\begin{equation}
    V_z = \frac{1 - A_x^+ V_x}{\alpha A_z^-} = \frac{A_x^- V_x}{\alpha A_z^-}
    \, \Rightarrow \,
    \alpha \frac{A_z^-}{A_z^- + A_z^+} = \frac{A_x^-}{A_x^- + A_x^+}\,.
\end{equation}

One can show these conditions also hold for the above scheme, under the following
mapping of variables:
\begin{align}
    V = V_x, \quad \tilde{A} = A_x^-, \quad
    \tilde{V} = 1 / A_z^+, \quad A = A_z^- A_z^+ V_z = \bar{g} (A_z^+), \quad
    \theta = A_z^- V_z = \frac{1}{1 + \frac{A_z^+}{A_z^-}}.
\end{align}

These equations are easily generalizable to the multi-layer case; the equations for $A_z^+$
and $A_x^-$ remain the same, while the equations for $A_x^+$ and $A_z^-$ become
\begin{align}
    &A_{x_\ell}^+ = \frac{1}{V_{x_\ell}^+ (A_{x_\ell}^-, 1 / A_{z_{\ell - 1}}^+)} - A_{x_\ell}^-, \\
    &A_{z_{\ell - 1}}^- = \frac{1}{V_{z_{\ell - 1}}^- (A_{x_\ell}^-, 1 /
    A_{z_{\ell - 1}}^+)} - A_{z_{\ell - 1}}^+,
\end{align}
where
\begin{align}
    &V_{x_\ell}^+ (A, V) = \mathbb{E}_{b, t, z, w} \, \partial^2_B \log
        Z_\ell (A, b, V, w), \label{eq:gvamp_v5} \\
    &V_{z_{\ell - 1}}^- (A, V) = \frac{1}{A} + \frac{1}{A^2} \mathbb{E}_{b,
        t, z, w} \, \partial^2_w \log Z_\ell (A, b, V, w). \label{eq:gvamp_v6}
\end{align}
Note that the quantities in (\ref{eq:gvamp_v1}), (\ref{eq:gvamp_v4}),
(\ref{eq:gvamp_v5}) and (\ref{eq:gvamp_v6}) were already being evaluated in
the scheme described in the previous subsection.

\subsection{Further considerations}

\subsubsection{Mutual information from entropy}
\label{sec:further1}

While in our computations we focus on the entropy $H(\T_\ell)$, the mutual
information $I(\T_\ell; \T_{\ell - 1})$ can be easily obtained from the chain
rule relation
\begin{align}
    I(\T_\ell; \T_{\ell-1}) &= H(\T_\ell) + \eTT{\ell}{\ell-1} \log
    P_{\T_\ell | \T_{\ell - 1}} (\bm{t}_\ell | \bm{t}_{\ell-1}) \notag \\
    &= H(\T_\ell) + \int dz \, \mathcal{N} (z; 0, \tilde{\rho}_\ell) \int dh
    \, P_\ell (h | z) \log P_\ell (h | z),
\end{align}
where in order to go from the first to the second line we have used the
central limit theorem. In particular if the mapping $\X \to \T_{\ell - 1}$
is deterministic, as typically enforced in the models we use in the
experiments, then $I(\T_\ell; \T_{\ell - 1}) = I(\T_\ell; \X)$.

\subsubsection{Equivalence in linear case}
\label{sec:further2}

In the linear case, $\Y = W_L W_{L - 1} \cdots W_1 \X + \mathcal{N} (0,
\Delta)$, our formula reduces to
\cite{kabashima_inference_2008,tulino_support_2013,toappear}
\begin{equation}
    \lim_{n_ \to \infty} \, n^{-1} I(\Y; \X) = \min \extr_{A, V} \, \left\{
        -\frac12 A V - \frac12 G(-V / \Delta) + I (x; x + \sqrt{1 / A} \xi) \right\},
\end{equation}
where
\begin{equation}
    G(x) = \extr_{\Lambda} \left\{ -\mathbb{E}_\lambda \log |\lambda -
    \Lambda| + \Lambda x \right\} - (\log |x| + 1),
\end{equation}
is also known as the integrated R-transform, with $\lambda$ the
eigenvalues of $W^T W$, $W \equiv W_L W_{L - 1} \cdots W_1$. If $P_0$ is Gaussian,
then $I(x; x + \sqrt{1 / A} \xi) = \frac{1}{2} \log(1 + A)$; extremizing over $A$
and $V$ then gives
\begin{equation}
    A = 1/V - 1, \qquad V = \Delta \mathcal{S} (-\Delta),
\end{equation}
where $\mathcal{S} (z)$ is the Stieltjes transform of $W^T W$. The mutual information
can then be rewritten as
\begin{equation}
    \lim_{n \to \infty} \, n^{-1} I(\Y; \X) = \frac12 \mathbb{E}_\lambda \log(\lambda + \Delta)
    - \frac12 \log \Delta.
\end{equation}
This same result can be achieved analytically with much less effort, since
in this case $P_{\Y} (\bm{y}) = \mathcal{N} (\bm{y}; 0, \Delta I_m +
W W^T)$.

\section{Proof of the replica formula by the adaptive interpolation method}\label{sec:partI}
In this section we prove Theorem \ref{th:RS_2layer_main} (that we re-write below more explicitely) using the adaptive
interpolation method of \cite{barbier_stoInt,BarbierOneLayerGLM} in a multi-layer
setting, as first developed in \cite{2017arXiv170910368B_sm}.
\subsection{Two-layer generalized linear estimation: Problem statement}
One gives here a generic description of the observation model, that is a two-layer generalized linear model (GLM).
Let $n_0, n_1, n_2 \in \N^*$ and define the triplet $\mathbf{n} = (n_0,n_1,n_2)$.
Let $P_0$ be a probability distribution over $\R$ and let $(X^0_i)_{i=1}^{n_0} \iid P_0$ be the components of a signal vector $\bX^0$.
One fixes two functions $\varphi_1: \R \times \R^{k_1} \to \R$ and $\varphi_2: \R \times \R^{k_2}\to \R$, $k_1$, $k_2 \in \mathbb{N}$.
They act component-wise, i.e.\ if $\bx \in \R^m$ and $\bA \in \R^{m \times k_i}$ then $\varphi_i(\bx,\bA) \in \R^m$ is a vector with entries $[\varphi_i(\bx,\bA)]_{\mu} \defeq \varphi_i(x_{\mu},\bA_{\mu})$, $\bA_{\mu}$ being the $\mu$-th row of $\bA$.
For $i \in \{1,2\}$, consider $(\bA_{i,\mu})_{\mu = 1}^{n_i} \iid P_{\! A_i}$ where $P_{A_i}$ is a probability distribution over $\R^{k_{i}}$. One acquires $n_2$ measurements through
\begin{align}\label{measurements}
Y_\mu = \varphi_2\bigg(\frac{1}{\sqrt{n_1}} \bigg[\bW{2}
	\varphi_1\bigg(\frac{\bW{1} \bX^0}{\sqrt{n_0}}, \bA_1\bigg)\bigg]_{\mu}, \bA_{2,\mu}\bigg) + \sqrt{\Delta} Z_\mu\,,
	\qquad 1\leq \mu \leq n_2\,.
\end{align}
Here $(Z_\mu)_{\mu=1}^{n_2} \iid \mathcal{N}(0,1)$ is an additive Gaussian noise, $\Delta > 0$, and $\bW{1} \in \R^{n_1 \times n_0}$, $\bW{2} \in \R^{n_2 \times n_1}$ are measurement matrices whose entries are i.i.d.\ with distribution $\mathcal{N}(0,1)$.
Equivalently,
\begin{equation}\label{eq:channel}
Y_\mu \sim P_{\rm out,2}\bigg( \cdot \, \bigg| 
	\frac{1}{\sqrt{n_1}} \bigg[\bW{2} \varphi_1\bigg(\frac{\bW{1} \bX^0}{\sqrt{n_0}}, \bA_1\bigg)\bigg]_{\mu}\bigg)\,
\end{equation}
where the transition density, w.r.t. Lebesgue's measure, is
\begin{equation}\label{defPout2}
P_{\rm out,2}\big( y \big| x \big) 
= \int dP_{A_2}(\ba)\frac{1}{\sqrt{2\pi \Delta}} e^{-\frac{1}{2 \Delta}( y - \varphi_2(x,\ba))^2}\,.
\end{equation}

Our analysis uses both representations \eqref{measurements} and \eqref{eq:channel}. The estimation problem is to recover $\bX^0$ from the knowledge of $\bY=(Y_\mu)_{\mu=1}^{n_2}$, $\varphi_1$, $\varphi_2$, $\bW{1}$, $\bW{2}$ and $\Delta$, $P_0$.

In the language of statistical mechanics, the random variables $\bY$, $\bW{1}$, $\bW{2}$, $\bX^0$, $\bA_1$, $\bA_2$, $\bZ$ are called {\it quenched} variables because once the measurements are acquired they have a ``fixed realization''.
An expectation taken w.r.t.\ {\it all} quenched random variables appearing in an expression will simply be denoted by $\E$ {\it without} subscript. Subscripts are only used when the expectation carries over a subset of random variables appearing in an expression or when some confusion could arise. 

Defining the {\it Hamiltonian}
\begin{align} \label{Hamil}
\cH(\bx, \ba_1;\bY,\bW{1}, \bW{2})
	&\defeq - \sum_{\mu = 1}^{n_2} \ln P_{\rm out,2}\bigg( Y_{\mu} \bigg|
		\frac{1}{\sqrt{n_1}} \bigg[\bW{2} \varphi_1\bigg(\frac{\bW{1} \bx}{\sqrt{n_0}},\ba_1\bigg)\bigg]_{\mu} \bigg)	\,,
\end{align}
the joint posterior distribution of $(\bx$, $\ba_1)$ given the quenched variables $\bY$, $\bW{1}$, $\bW{2}$ reads (Bayes formula)
\begin{align}\label{bayes}
dP(\bx, \ba_1| \bY,\bW{1}, \bW{2}) &= \frac{1}{\cZ(\bY,\bW{1}, \bW{2})} dP_0(\bx)dP_{A_1}(\ba_1)e^{-\cH(\bx, \ba_1;\bY,\bW{1}, \bW{2})} \, ,
\end{align}
$dP_0(\bx) = \prod_{i=1}^{n_0} dP_0(x_i)$ being the prior over the signal and $dP_{A_1}(\ba_1) \defeq \prod_{i=1}^{n_1} dP_{A_1}(\ba_{1,i})$. The  {\it partition function} is defined as 
\begin{multline}
\cZ(\bY,\bW{1}, \bW{2}) \\
\defeq \!\int \!\! dP_0(\bx)dP_{A_1}(\ba_1)dP_{A_2}(\ba_2)
\prod_{\mu=1}^{n_2} \frac{1}{\sqrt{2\pi\Delta}}e^{-\frac{1}{2 \Delta}\Big( Y_\mu - \varphi_2\Big(\frac{1}{\sqrt{n_1}} \Big[\bW{2} \varphi_1\Big(\frac{\bW{1} \bx}{\sqrt{n_0}},\ba_1\Big)\Big]_{\mu},\ba_{2,\mu}\Big)\Big)^2}.	\label{Z_1}
\end{multline}
One introduces a standard statistical mechanics notation for the expectation w.r.t.\ the posterior \eqref{bayes}, the so called {\it Gibbs bracket} $\langle - \rangle$ defined for any continuous bounded function $g$ as
\begin{equation}
\big\langle g(\bx, \ba_1)\big\rangle \defeq \int dP(\bx,\ba_1|\bY,\bW{1}, \bW{2}) g(\bx,\ba_1)
\end{equation}
One important quantity is the associated {\it averaged  free entropy} (or minus the {\it averaged free energy})
\begin{align}
f_{\mathbf{n}} \defeq \frac{1}{n_0} \E\ln \cZ(\bY,\bW{1}, \bW{2}) \,. \label{f}
\end{align}
It is perhaps useful to stress that $\mathcal{Z}(\bY,\bW{1}, \bW{2})$ is nothing else than the density of $\bY$ conditioned on $\bW{1}, \bW{2}$; so we have the explicit representation (used later on)
\begin{align}\label{fff}
f_{\mathbf{n}} &= \frac{1}{n_0} \E_{\bW{1}, \bW{2}}\int d\bY \mathcal{Z}(\bY,\bW{1}, \bW{2}) \ln \cZ(\bY,\bW{1}, \bW{2}) \nn
&= \frac{1}{n_0} \E_{\bW{1}, \bW{2}}\Bigg[\int d\bY dP_0(\bX^0)dP_{A_1}(\bA_1) e^{-\cH(\bX^0, \bA_1;\bY,\bW{1}, \bW{2})}\nn
&\qquad\qquad\qquad\qquad \cdot \ln\int dP_0(\bx)dP_{A_1}(\ba_1)  e^{-\cH(\bx,\ba_1;\bY,\bW{1}, \bW{2})}\Bigg],
\end{align}
where $d\bY = \prod\limits_{\mu=1}^{n_2} dY_\mu$.
Thus $f_{\mathbf{n}}$ is minus the conditional entropy $-H(\bY|\bW{1},\bW{2})/n_0$ of the measurements.

This appendix presents the derivation, thanks to the adaptive interpolation method, of the thermodynamic limit $\lim_{\mathbf{n} \to \infty}f_{\mathbf{n}}$ in the ``high-dimensional regime'', namely when $n_0,n_1,n_2 \to +\infty$ such that $\nicefrac{n_2}{n_1} \to \alpha_2 > 0$, $\nicefrac{n_1}{n_0} \to \alpha_1 > 0$. In this high-dimensional regime, the ``measurement rate'' satisfies $\nicefrac{n_2}{n_0} \to \alpha \defeq \alpha_1 \cdot \alpha_2$.

\subsection{Important scalar inference channels} \label{sec:scalar_inf}
The thermodynamic limit of the free entropy will be expressed in terms of the free entropy of simple \textit{scalar} inference channels. This ``decoupling property'' results from the mean-field approach in statistical physics, used through in the replica method to perform a formal calculation of the free entropy of the model \cite{mezard1990spin,mezard2009information}. This section presents these three scalar denoising models.

{\bf a)} The first channel is an additive Gaussian one. Let $r \geq 0$ play the role of a signal-to-noise ratio.
Consider the inference problem consisting of retrieving $X_0 \sim P_0$ from the observation ${Y_0 = \sqrt{r}\, X_0 + Z_0}$, where $Z_0 \sim \cN(0,1)$ independently of $X_0$.
The associated posterior distribution is
\begin{align}
dP(x|Y_0) = \frac{dP_0(x) e^{\sqrt r\, Y_0 x - r x^2/2}}{\int dP_0(x)e^{\sqrt r\, Y_0 x - r x^2/2}}	\,.
\end{align}
The free entropy associated to this channel is just the expectation of the logarithm of the normalization factor
\begin{align} \label{psi0}
\psi_{P_0}(r) \defeq \E \ln \int dP_0(x)e^{\sqrt r \,Y_0 x - r x^2/2} \,.
\end{align}

{\bf b)} The second scalar channel appearing naturally in the problem is linked to $P_{\rm out,2}$ through the following inference model. Suppose that $V,U \iid \cN(0,1)$ where $V$ is {\it known}, while the inference problem is to recover the {\it unknown} $U$ from the observation
\begin{equation}\label{eq:Pout_scalar_channel}
\widetilde{Y}_0 \sim P_{\rm out, 2} \big( \cdot \, | \sqrt{q}\, V + \sqrt{\rho - q} \,U \big)\,,
\end{equation}
where $\rho >0$, $q \in [0, \rho]$.
The free entropy for this model, again related to the normalization factor of the posterior 
\begin{align}
dP(u|\widetilde Y_0,V) = \frac{{\cal D}u P_{\rm out,2}\big(\widetilde{Y}_0 | \sqrt{q}\, V + \sqrt{\rho - q}\, u\big)}{\int {\cal D}u P_{\rm out,2}\big(\widetilde{Y}_0 | \sqrt{q}\, V + \sqrt{\rho - q}\, u\big)}\, ,
\end{align}
namely,
\begin{align}\label{PsiPout}
\Psi_{P_{\rm out,2}}(q;\rho) \defeq
\E \ln \int {\cal D}u P_{\rm out,2}\big(\widetilde{Y}_0 | \sqrt{q}\, V + \sqrt{\rho - q}\, u\big)\, ,
\end{align}
where ${\cal D}u =du(2\pi)^{-\nicefrac{1}{2}} e^{-\nicefrac{u^2}{2}}$ is the standard Gaussian measure.

{\bf c)} The third scalar channel to play a role is linked to the hidden layer $\bX^1 \defeq \varphi_1\Big(\nicefrac{\bW{1} \bX^0}{\sqrt{n_0}}, \bA_1\Big)$ of the two-layer GLM. Suppose that $V,U \iid \cN(0,1)$, where $V$ is \textit{known}. Consider the problem of recovering the {\it unknown} $U$ from the observation $Y_0^{\prime} = \sqrt{r}\varphi_1(\sqrt{q}\, V + \sqrt{\rho - q} \,U,\bA_1) + Z'$ where $r \geq 0$, $\rho >0$, $q \in [0, \rho]$, $Z' \sim \cN(0,1)$ and $\bA_1 \sim P_{A_1}$. Equivalently,
$Y_0^{\prime} \sim P_{\rm out,1}^{(r)}( \cdot \vert \sqrt{q}\, V + \sqrt{\rho - q} \, U)$ with
\begin{equation}
P_{\rm out,1}^{(r)}(y \vert x) \defeq  \int dP_{A_1}(\ba)\frac{1}{\sqrt{2\pi}} e^{-\frac{1}{2}(y - \sqrt{r}\varphi_1(x,\ba))^2} \; .
\end{equation}
From this last description, it is easy to see that the free entropy for this model is given by a formula similar to \eqref{PsiPout}.
Introducing $\delta(\, \cdot - \varphi_1(x,\ba))$, the Dirac measure centred on $\varphi_1(x,\ba)$, it reads
\begin{align*}
\Psi_{P_{\rm out,1}^{(r)}}\!\!(q;\rho)
&= \E \ln \int {\cal D}u P_{\rm out,1}^{(r)}\big(Y_0^{\prime} | \sqrt{q}\, V + \sqrt{\rho - q}\, u\big) \\
&= \E \ln \int \! {\cal D}u dP_{A_1}(\ba) dh \frac{1}{\sqrt{2\pi}} e^{-\frac{1}{2}( Y_0^{\prime} - \sqrt{r}h)^2} \delta\big(h -\varphi_1(\sqrt{q}\, V + \sqrt{\rho - q}\, u, \ba)\big)\\
&= -\frac{\ln(2\pi) + \E[(Y_0^{\prime})^2]}{2} \\
&\qquad+ \E \ln \!\! \int \!\! {\cal D}u dP_{A_1}(\ba) dh e^{\sqrt{r}h Y_0^{\prime} -\frac{r h^2}{2}} \delta\big(h -\varphi_1(\sqrt{q}\, V + \sqrt{\rho - q}\, u, \ba)\big) \; .
\end{align*}
The second moment of $Y_0^{\prime}$ is simply $\E[(Y_0^{\prime})^2] = r \E[\varphi_1^2(T,\bA_1)] + 1$ with $T \sim \cN(0,\rho)$, $\bA_1 \sim P_{A_1}$. Hence 
\begin{equation}\label{rewritingPsiOutThirdChannel}
\Psi_{P_{\rm out,1}^{(r)}}(q;\rho) = -\frac{1 + \ln(2\pi) + r \E[\varphi_1^2(T,\bA_1)]}{2} + \Psi_{\varphi_1}(q, r;\rho)
\end{equation}
where
\begin{equation}
\Psi_{\varphi_1}(q, r;\rho) \defeq \E \ln \int \! {\cal D}u dP_{A_1}(\ba) dh e^{\sqrt{r}h Y_0^{\prime} -\frac{r h^2}{2}} \delta\big(h -\varphi_1(\sqrt{q}\, V + \sqrt{\rho - q}\, u, \ba)\big) \; .
\end{equation}

\subsection{Replica-symmetric formula and mutual information}\label{RSformula-andhyp}
Our goal is to prove Theorem~\ref{th:RS_2layer_main} that gives a single-letter {\it replica-symmetric formula} for the asymptotic free entropy 
of model \eqref{measurements}, \eqref{eq:channel}. The result holds under the following hypotheses:
\begin{enumerate}[label=(H\arabic*),noitemsep]
	\item \label{hyp:bounded} The prior distribution $P_0$ has a bounded support.
	\item \label{hyp:c2} $\varphi_1$, $\varphi_2$ are bounded $\mathcal{C}^2$ functions with bounded first and second derivatives w.r.t.\ their first argument.
	\item \label{hyp:phi_gauss2} $\bW{1}$, $\bW{2}$ have entries i.i.d.\ with respect to $\cN(0,1)$.
\end{enumerate}

Let $\rho_0 \defeq \E[(X^0)^2]$ where $X^0\sim P_0$ and $\rho_1 \defeq \E[\varphi_1^2(T,\bA_1)]$ where $T \sim \cN(0,\rho_0)$, $\bA_1 \sim P_{A_1}$.
The {\it replica-symmetric potential} (or just potential) is
\begin{align} \label{frs}
f_{\rm{RS}}(q_0,r_0,q_1,r_1;\rho_0,\rho_1) \defeq  \psi_{P_0}(r_0) + \alpha_1 \Psi_{\varphi_1}(q_0, r_1;\rho_0)
+ \alpha \Psi_{P_{\rm out,2}}(q_1;\rho_1) - \frac{r_0 q_0}{2} - \alpha_1 \frac{r_1 q_1}{2} \,.  
\end{align}
\begin{thm}[Replica-symmetric formula] \label{th:RS_2layer}
	Suppose that hypotheses~\ref{hyp:bounded},~\ref{hyp:c2},~\ref{hyp:phi_gauss2} hold. Then, the thermodynamic limit of the free entropy \eqref{f} for the two-layer generalized linear estimation model \eqref{measurements}, \eqref{eq:channel} satisfies
	\begin{equation}\label{eq:rs_formula}
	f_{\infty} := \lim_{\mathbf{n} \to\infty }f_{\mathbf{n}} 
	=  \adjustlimits{\sup}_{q_1 \in [0,\rho_1]} {\inf}_{r_1 \geq 0}  \adjustlimits{\sup}_{q_0 \in [0,\rho_0]} {\inf}_{r_0 \geq 0} f_{\rm RS}(q_0,r_0,q_1,r_1;\rho_0,\rho_1)\,.
	\end{equation}	
\end{thm}

\subsection{Interpolating estimation problem}\label{interp-est-problem}
The proof of Theorem~\ref{th:RS_2layer} follows the same steps than the proof of the replica formula for a one-layer GLM in \cite{BarbierOneLayerGLM}. Let $t\in[0,1]$ be an interpolation parameter. We introduce an \textit{interpolating estimation problem} that interpolates between the original problem \eqref{eq:channel} at $t=0$ and two analytically tractable problems at $t=1$.

Prior to establishing this interpolation scheme, we set up some important notations. The output of the hidden layer is denoted $\bX^1 \in \mathbb{R}^{n_1}$, i.e.\
\begin{equation}\label{defX1}
\forall i \in \{1,\dots,n_1\}: X_i^1 \defeq \varphi_1\Big(  \Big[\frac{\bW{1} \bX^0}{\sqrt{n_0}}\Big]_{i}, \bA_{1,i}\Big)\,.
\end{equation}
The expected power of this signal is
\begin{equation}\label{expectedPowerX1}
\rho_1(n_0) := \E\Bigg[ \frac{\big\Vert \bX^1 \big\Vert^2}{n_1}\Bigg] = \frac{1}{n_1} \sum_{i=1}^{n_1} \E\big[(X_i^1)^2 \big] = \E\Big[ \varphi_1^2\big([\nicefrac{\bW{1} \bX^0}{\sqrt{n_0}}]_{1}, \bA_{1,1} \big) \Big] \,.
\end{equation}
In Appendix \ref{app:convergenceRho} we prove
\begin{proposition}[Convergence of $\rho_1(n_0)$ to $\rho_1$]\label{prop:convergenceRho}
	Under the hypotheses~\ref{hyp:bounded},~\ref{hyp:c2},~\ref{hyp:phi_gauss2}
	\begin{equation}
	\lim_{n_0 \to +\infty} \rho_1(n_0) = \rho_1 \:.
	\end{equation}
\end{proposition}
Denote $\psi_{P_{\rm out, 2}}^{\prime}(\,\cdot\,;\rho)$ the derivative of the function $q \in [0,\rho] \mapsto \psi_{P_{\rm out, 2}}\big(q;\rho\big)$. We define the bounded sequence $r^*(n_0) \defeq 2\alpha_2 \psi_{P_{\rm out, 2}}^{\prime}\big(\rho_1(n_0);\rho_1(n_0)\big)$ and its limit $r^* \defeq 2\alpha_2 \psi_{P_{\rm out, 2}}^{\prime}(\rho_1;\rho_1)$. We also introduce $\rmax$ as an upper bound on the whole sequence, i.e.\ $\forall n_0 \geq 1: r^*(n_0) \leq \rmax$ and $r^*\leq \rmax$.

Let $\{s_{n_0}\}_{n_0 \geq 1} \! \in (0,\nicefrac{1}{2}]^{\mathbb{N}^*}$ be some bounded sequence that converges to $0$  positively. This sequence will be explicitly specified later and defines a sequence of intervals $\mathcal{B}_{n_0} = [s_{n_0}, 2 s_{n_0}]^2$. For a fixed $n_0$, we pick a pair $\epsilon = (\epsilon_1,\epsilon_2) \in \mathcal{B}_{n_0}$.
Let $q_{\epsilon}: [0,1] \to [0,\rho_1(n_0)]$ and $r_{\epsilon}: [0,1] \to [0,\rmax]$ be two continuous ``interpolation functions''. Their dependence on $\epsilon$ will also be specified later. 
It is useful to define
\begin{align}\label{R1R2}
R_1(t,\epsilon) \defeq \epsilon_1 + \int_0^t r_{\epsilon}(v)dv\,,\qquad 
R_2(t,\epsilon) \defeq \epsilon_2 + \int_0^t q_{\epsilon}(v)dv\,.
\end{align}
We will be interested in functions $r_{\epsilon}$, $q_{\epsilon}$ that satisfy some regularity properties:
\begin{definition}[Regularity]\label{def:reg}
The families of functions $(q_{\epsilon})_{\epsilon \in \mathcal{B}_{n_0}}$ and $(r_{\epsilon})_{\epsilon \in \mathcal{B}_{n_0}}$, taking values in $[0,\rho_1(n_0)]$ and $[0,\rmax]$ respectively, are said \textit{regular} if for all $t \in [0,1]$ the mapping
	\begin{equation*}
	R(t,\cdot):
	\left\vert
	\begin{array}{ccc}
	(s_{n_0},2 s_{n_0})^2 & \to & R\big(t,(s_{n_0},2 s_{n_0})^2\big) \\
	\epsilon & \mapsto & \big(R_1(t,\epsilon), R_2(t,\epsilon )\big)
	\end{array}
	\right.
	\end{equation*}
	is a $\mathcal{C}^1$-diffeomorphism, whose Jacobian is greater than or equal to $1$.
\end{definition}
Let $\bS_{t,\epsilon} \in \mathbb{R}^{n_2}$ be the vector with entries
\begin{equation}\label{defS_t_mu}
S_{t,\epsilon,\mu} \defeq \sqrt{\frac{1-t}{n_1}}\, \bigg[\bW{2} \varphi_1\bigg(\frac{\bW{1} \bX^0}{\sqrt{n_0}}, \bA_1\bigg)\bigg]_\mu
+ \sqrt{R_2(t,\epsilon)} \,V_{\mu} + \sqrt{\rho_1(n_0) t - R_2(t,\epsilon) + 2s_{n_0}} \,U_{\mu}	
\end{equation}
where $V_{\mu}, U_{\mu} \iid \cN(0,1)$.
Assume  $\bV=(V_{\mu})_{\mu=1}^{n_2}$ is {\it known}. The inference problem is to estimate both unknowns $\bX^0$ and $\bU$ from the knowledge of $\bV$, $\bW{1}$, $\bW{2}$ and the two kinds of observations
\begin{eqnarray}
	\begin{cases}
			Y_{t,\epsilon,\mu}  &\sim P_{\rm out, 2}(\ \cdot \ | \, S_{t,\epsilon,\mu})\,, \qquad\qquad\qquad\qquad\;\;\,  1 \leq \mu \leq n_2, \\
			Y'_{t,\epsilon,i} &= \sqrt{R_1(t,\epsilon)}\, \varphi_1\Big(\Big[\frac{\bW{1} \bX^0}{\sqrt{n_0}}\Big]_i, \bA_{1,i}\Big) + Z'_i\,, \;\: 1 \leq i \leq n_1,
	\end{cases}
	\label{2channels}
\end{eqnarray}
where $(Z_i')_{i=1}^{n_1} \iid {\cal N}(0,1)$.
$\bY_{t,\epsilon}=(Y_{t,\epsilon,\mu})_{\mu=1}^{n_2}$ and $\bY'_{t,\epsilon}=(Y_{t,\epsilon,i}')_{i=1}^{n_1}$ are the ``time-and perturbation-dependent'' observations.

Define, with a slight abuse of notations, $s_{t, \epsilon, \mu}(\bx,\ba_1, u_\mu) \equiv s_{t, \epsilon, \mu}$ as
\begin{align}
	s_{t,\epsilon,\mu}
		= \sqrt{\frac{1-t}{n_1}}\, \bigg[\bW{2} \varphi_1\bigg(\frac{\bW{1} \bx}{\sqrt{n_0}},\ba_1\bigg)\bigg]_\mu
		+ \sqrt{R_2(t,\epsilon)} \,V_{\mu} + \sqrt{\rho_1(n_0) t - R_2(t,\epsilon) + 2s_n} \,u_{\mu}\,.
	\label{stmu}
\end{align}
We now introduce the {\it interpolating Hamiltonian}
\begin{multline}\label{interpolating-ham-general}
\cH_{t,\epsilon}(\bx,\ba_1, \bu;\bY,\bY',\bW{1},\bW{2},\bV)
\defeq - \sum_{\mu=1}^{n_2} \ln P_{\rm out,2} ( Y_{\mu} |s_{t, \epsilon, \mu})\\
+ \frac{1}{2} \sum_{i=1}^{n_1}\bigg(Y'_{i}
- \sqrt{R_1(t,\epsilon)}\, \varphi_1\bigg(\bigg[\frac{\bW{1} \bx}{\sqrt{n_0}}\bigg]_i,\ba_{1,i}\bigg)\bigg)^2 \,.
\end{multline}
It depends on $\bW{2}$ and $\bV$ through the terms $(s_{t,\epsilon,\mu})_{\mu=1}^{n_2}$, and on $\bW{1}$ through both $(s_{t,\epsilon,\mu})_{\mu=1}^{n_2}$ and the sum over $i \in \{1,\dots,n_1\}$. When the $(t,\epsilon)$-dependent observations \eqref{2channels} are considered, it reads
\begin{multline}\label{interpolating-ham}
\cH_{t,\epsilon}(\bx,\ba_1, \bu;\bY_{t,\epsilon},\bY_{t,\epsilon}^{'},\bW{1},\bW{2},\bV)
	\defeq - \sum_{\mu=1}^{n_2} \ln P_{\rm out,2} ( Y_{t,\epsilon,\mu} |s_{t, \epsilon,\mu})\\
	+ \frac{1}{2} \sum_{i=1}^{n_1}\Bigg[
	\sqrt{R_1(t,\epsilon)}
	\bigg( \varphi_1\bigg(\bigg[\frac{\bW{1} \bX^0}{\sqrt{n_0}}\bigg]_i,\bA_{1,i}\bigg)
		- \varphi_1\bigg(\bigg[\frac{\bW{1} \bx}{\sqrt{n_0}}\bigg]_i,\ba_{1,i}\bigg)\bigg)
		+ Z'_i\Bigg]^2 \,.
\end{multline}
The corresponding Gibbs bracket $\langle - \rangle_{\mathbf{n},t,\epsilon}$, which is the expectation operator w.r.t.\ the $(t,\epsilon)$-dependent joint posterior distribution of $(\bx,\ba_1, \bu)$ given $(\bY_{t,\epsilon},\bY_{t,\epsilon}',\bW{1},\bW{2},\bV)$ is defined for every continuous bounded function $g$ on $\R^{n_0} \times \R^{n_1 \times k_1} \times \R^{n_2}$ as: 
\begin{align}
	\label{gibbs}
	\big\langle g(\bx,\ba_1,\bu) \big\rangle_{\mathbf{n},t,\epsilon}
	\defeq \frac{1}{\cZ_{\mathbf{n},t,\epsilon}} \int \! dP_0(\bx)dP_{A_1}(\ba_1){\cal D}\bu \, g(\bx,\ba_1,\bu)\,e^{-\cH_{t,\epsilon}(\bx,\ba_1,\bu;\bY_{t,\epsilon},\bY_{t,\epsilon}',\bW{1},\bW{2},\bV)} \, .
\end{align}
In \eqref{gibbs}, ${\cal D}\bu = (2\pi)^{-\nicefrac{n_2}{2}}\prod_{\mu=1}^{n_2} du_\mu e^{-\nicefrac{u_\mu^2}{2}}$ is the $n_2$-dimensional standard Gaussian distribution and $\cZ_{\mathbf{n},t,\epsilon} \equiv \cZ_{\mathbf{n},t,\epsilon}(\bY_{t,\epsilon},\bY_{t,\epsilon}',\bW{1},\bW{2},\bV)$ is the appropriate normalization, i.e.\
\begin{equation} \label{Zt}
	\cZ_{\mathbf{n},t,\epsilon}(\bY_{t,\epsilon},\bY_{t,\epsilon}',\bW{1},\bW{2},\bV) \defeq \int \! dP_0(\bx)dP_{A_1}(\ba_1){\cal D}\bu \, e^{-\cH_{t,\epsilon}(\bx,\ba_1,\bu;\bY_{t,\epsilon},\bY_{t,\epsilon}',\bW{1},\bW{2},\bV)}\,.
\end{equation}
Finally, the {\it interpolating free entropy} is 
\begin{equation}
	f_{\mathbf{n},\epsilon}(t) \defeq \frac{1}{n_0} \E \ln \cZ_{\mathbf{n},t,\epsilon}(\bY_{t,\epsilon},\bY'_{t,\epsilon},\bW{1},\bW{2},\bV) 
	\,.	\label{ft}
\end{equation}
Note that the {\it perturbation} $\epsilon=(\epsilon_1,\epsilon_2)$ induces only a small change in the free entropy, namely of the order of $s_{n_0}$:
\begin{lemma}[Small free entropy variation under perturbation]\label{lem:perturbation_f}
There exists a constant $C$ such that
\begin{equation}
	\forall \epsilon \in \mathcal{B}_{n_0}: \: \vert f_{\mathbf{n},\epsilon}(0) - f_{\mathbf{n},\epsilon=(0,0)}(0) \vert \leq C s_{n_0} \,,
\end{equation}	
uniformly w.r.t.\ $\mathbf{n}$ and the choice of $(q_{\epsilon})_{\epsilon \in {\cal B}_{n_0}}$, $(r_{\epsilon})_{\epsilon \in {\cal B}_{n_0}}$.
\end{lemma}
\begin{proof}
	A simple computation shows that $\frac{\partial f_{\mathbf{n},\epsilon}(0)}{\partial \epsilon_1} = - \frac{n_1}{n_0} \E \langle \mathcal{L} \rangle_{\mathbf{n},0,\epsilon}$, where 
	\begin{multline*}
	\mathcal{L} \defeq \frac{1}{n_1} \sum_{i=1}^{n_1}
	\frac{1}{2}\varphi_1^2\bigg(\bigg[\frac{\bW{1} \bx}{\sqrt{n_0}}\bigg]_i, \ba_{1,i}\bigg)
	- \varphi_1\bigg(\bigg[\frac{\bW{1} \bx}{\sqrt{n_0}}\bigg]_i,\ba_{1,i}\bigg) \varphi_1\bigg(\bigg[\frac{\bW{1} \bX^0}{\sqrt{n_0}}\bigg]_i, \bA_{1,i}\bigg)\\
	- \frac{1}{2\sqrt{\epsilon_1}} \varphi_1\bigg(\bigg[\frac{\bW{1} \bx}{\sqrt{n_0}}\bigg]_i, \bA_{1,i}\bigg) Z'_i\; .
	\end{multline*}
	Lemma~\ref{lem:expectationL} (see Appendix~\ref{appendix-overlap}) applied for $t=0$ reads $\E \langle \mathcal{L} \rangle_{\mathbf{n},0,\epsilon} = -\frac{1}{2} \E \langle \widehat{Q} \rangle_{\mathbf{n},0,\epsilon}$ where $\widehat{Q}$ being the overlap
	\begin{align*}
	\widehat{Q} \defeq \frac{1}{n_1}\sum_{i=1}^{n_1}
		\varphi_1\bigg(\bigg[\frac{\bW{1} \bx}{\sqrt{n_0}}\bigg]_i,\ba_{1,i}\bigg) \varphi_1\bigg(\bigg[\frac{\bW{1} \bX^0}{\sqrt{n_0}}\bigg]_i, \bA_{1,i}\bigg)\,.
	\end{align*}
	$\vert \widehat{Q} \vert$ is trivially bounded by $\sup \vert\varphi_1\vert^2$, hence
	\begin{align}
	\Big|\frac{\partial f_{\mathbf{n},\epsilon}(0)}{\partial {\epsilon_1}} \Big|
	\leq \underbrace{\frac{n_1}{n_0}}_{\to \alpha_{1}} \frac{\sup \varphi_1^2}{2} \,.	
	\end{align}
	Let $u_{y}(x) \defeq \ln P_{\rm out,2}(y|x)$ and $u'_{y}(x)$ its $x$-derivative. In a similar fashion to what is done in Appendix~\ref{app:computationDerivative}, we compute
	\begin{align}
	\Big\vert \frac{\partial f_{\mathbf{n},\epsilon}(0)}{\partial {\epsilon_2}}\Big\vert
	= \frac{1}{2 n_0}\sum_{\mu=1}^{n_2}
	\big\vert \E \big[u'_{Y_{0,\epsilon,\mu}}(S_{0,\epsilon,\mu})\langle u'_{Y_{0,\epsilon,\mu}}(s_{0,\epsilon,\mu}) \rangle_{\mathbf{n},0,\epsilon}\big]\big\vert\,.
	\end{align}
	This quantity is bounded uniformly under the hypothesis \ref{hyp:c2} (see the first part of the proof of Proposition~\ref{prop:cancel_remainder}). Then, by the mean value theorem,
	$\vert f_{\mathbf{n},\epsilon}(0) - f_{\mathbf{n},(0,0)}(0) \vert \leq K \Vert\epsilon\Vert \leq C s_{n_0}$ for appropriate numerical constants $K$ and $C$.
\end{proof}

\subsection{Interpolating free entropy at t=0 and t=1}
We will denote by $\smallO_{n_0}(1)$ any quantity that vanishes uniformly in $t \in [0,1]$ and $\epsilon$ when $n_0 \to +\infty$.
It is easily shown (the first equality uses Lemma \ref{lem:perturbation_f}) that
\begin{equation}\label{eq:f0_f1}
	\left\{
		\begin{array}{lll}
			f_{\mathbf{n},\epsilon}(0)
			&=& f_{\mathbf{n},\epsilon=(0,0)}(0) + \smallO_{n_0}(1) 
			= f_{\mathbf{n}} - \frac{1}{2}\frac{n_1}{n_0} + \smallO_{n_0}(1)\,,\\
			f_{\mathbf{n},\epsilon}(1)
			&=& \tilde{f}_{(n_0,n_1),\epsilon} + \frac{n_2}{n_0}\Psi_{P_{\rm out, 2}}\Big(\epsilon_2 + \int_0^1 q_{\epsilon}(t) dt;\rho_1(n_0) + 2 s_{n_0}\Big) + \frac{n_1}{n_0}\frac{\ln 2\pi}{2} \, .
		\end{array}
	\right.
\end{equation}
In the last expression $\tilde{f}_{(n_0,n_1),\epsilon}$ denotes the free entropy of the \textit{one-layer} GLM
\begin{equation}
Y'_{i} = \sqrt{R_1(1,\epsilon)}\, \varphi_1\bigg(\bigg[\frac{\bW{1} \bX^0}{\sqrt{n_0}}\bigg]_i,\bA_{1,i}\bigg) + Z'_i\,, \qquad 1 \leq i \leq n_1,
\end{equation}
with $(X_i^0)_{i=1}^{n_0} \iid P_0$, $(\bA_{1,i})_{i=1}^{n_1} \iid P_{A_1}$ and $(Z_i')_{i=1}^{n_1} \iid {\cal N}(0,1)$.
Applying Theorem 1 of \cite{BarbierOneLayerGLM}, then \eqref{rewritingPsiOutThirdChannel}, the free entropy $\tilde{f}_{(n_0,n_1),\epsilon}$ in the thermodynamic limit $n_0, n_1 \to +\infty$ such that $\nicefrac{n_1}{n_0} \to \alpha_1$ satisfies
\begin{align}
\tilde{f}_{n_0,n_1}
&= \smallO_{n_0}(1) + 
\adjustlimits{\sup}_{q_0 \in [0,\rho_0]} {\inf}_{r_0 \geq 0} \Big\{ \psi_{P_0}(r_0) + \alpha_1 \Psi_{P_{\rm out,1}^{(R_1(1,\epsilon))}}(q_0;\rho_0)- \frac{r_0 q_0}{2}\Big\} \nn
&= \smallO_{n_0}(1) -\alpha_1 \frac{1+\ln 2\pi + R_1(1,\epsilon) \rho_1}{2}\nn
&\qquad\qquad+ \!\! \adjustlimits{\sup}_{q_0 \in [0,\rho_0]} {\inf}_{r_0 \geq 0} \Big\{ \psi_{P_0}(r_0) + \alpha_1 \Psi_{\varphi_1}(q_0, R_1(1,\epsilon);\rho_0)- \frac{r_0 q_0}{2}\Big\}\nn
&= \smallO_{n_0}(1) -\frac{\alpha_1}{2} \bigg(1+\ln 2\pi + \rho_1 \int_0^1 \!\! r_{\epsilon}(t) dt \bigg)\nn
&\qquad\qquad+ \!\! \adjustlimits{\sup}_{q_0 \in [0,\rho_0]} {\inf}_{r_0 \geq 0} \Big\{ \psi_{P_0}(r_0) + \alpha_1 \Psi_{\varphi_1}\bigg(q_0, \int_0^1 \!\! r_{\epsilon}(t) dt;\rho_0\bigg)- \frac{r_0 q_0}{2}\Big\} \,. \label{limitOneLayerFreeEntropy}
\end{align}
The last line follows from $R_1(1,\epsilon) = \int_0^1 \! r_{\epsilon}(t) + \smallO_{n_0}(1)$ and Proposition \ref{app:propertiesThirdChannel} that shows the Lipschitzianity of the mapping
\begin{equation*}
r \mapsto \adjustlimits{\sup}_{q_0 \in [0,\rho_0]} {\inf}_{r_0 \geq 0}  \Big\{\psi_{P_0}(r_0) + \alpha_1 \Psi_{\varphi_1}(q_0, r;\rho_0) - \frac{r_0 q_0}{2}\Big\} \,. 
\end{equation*}
The second summand appearing in $f_{\mathbf{n},\epsilon}(1)$ can also be simplified in the large $n_0$ limit:
\begin{align}
\frac{n_2}{n_0}\Psi_{P_{\rm out, 2}}\bigg(\epsilon_2 + \int_0^1 \!\! q_{\epsilon}(t) dt;\rho_1(n_0) + 2 s_{n_0}\bigg)
&= \alpha \Psi_{P_{\rm out, 2}}\bigg(\epsilon_2 + \int_0^1 \!\! q_{\epsilon}(t) dt;\rho_1(n_0) + 2 s_{n_0}\bigg) + \smallO_{n_0}(1)\nn
&= \alpha \Psi_{P_{\rm out, 2}}\bigg(\int_0^1 \!\! q_{\epsilon}(t) dt;\rho_1(n_0)\bigg) + \smallO_{n_0}(1)\:.\label{secondSummand_f_t=1}
\end{align}
The first line follows from Lemma \ref{lem:psiPoutUniformlyBounded} stated at the end of this section. The second line follows from the Lipschitzianity in both its arguments of the continuous mapping $(q,\rho) \mapsto \Psi_{P_{\rm out},2}\big(q;\rho\big)$ on the compact $\big\{(q,\rho):  0 \leq \rho \leq 1 + \rho_{\mathrm{max}}, 0 \leq q \leq \rho \big\}$, with $\rho_{\mathrm{max}}$ a upper bound on the sequence $\{\rho_1(n_0)\}_{n_0 \geq 1}$.
Reporting \eqref{limitOneLayerFreeEntropy} and \eqref{secondSummand_f_t=1} in \eqref{eq:f0_f1}, we obtain:
\begin{multline}\label{eq:f1_thermoLimit}
f_{\mathbf{n},\epsilon}(1) = -\frac{\alpha_1}{2}\bigg(1 + \rho_1 \int_0^1 \!\! r_{\epsilon}(t) dt\bigg)
+ \alpha \Psi_{P_{\rm out, 2}}\bigg(\int_0^1 q_{\epsilon}(t) dt;\rho_1(n_0)\bigg)\\
+ \adjustlimits{\sup}_{q_0 \in [0,\rho_0]} {\inf}_{r_0 \geq 0} \; \bigg\{\psi_{P_0}(r_0) + \alpha_1 \Psi_{\varphi_1}\bigg(q_0, \int_0^1 \!\! r_{\epsilon}(t);\rho_0\bigg)
- \frac{r_0 q_0}{2} \bigg\} + \smallO_{n_0}(1) \; .
\end{multline}
\begin{lemma}[Uniform upper bound on $\Psi_{P_{\rm out, 2}}$]\label{lem:psiPoutUniformlyBounded}
Assuming $\varphi_2$ is bounded, one has for all $\rho \geq 0$ and $q \in [0,\rho]$
\begin{equation*}
\vert \Psi_{P_{\rm out, 2}}(q;\rho)\vert \leq \frac{1 + \ln(2\pi \Delta)}{2} + \frac{2 \sup \vert \varphi_2 \vert^2}{\Delta}\; .
\end{equation*}
\begin{proof}
The upper bound $P_{\rm out,2}\big( y \big| x \big) \leq \nicefrac{1}{\sqrt{2\pi \Delta}}$ directly implies
\begin{equation*}
\Psi_{P_{\rm out, 2}}(q;\rho) \leq -\frac{1}{2} \ln(2\pi \Delta) \; .
\end{equation*}
By Jensen's inequality, one also has the lowerbound
\begin{align*}
\Psi_{P_{\rm out, 2}}(q;\rho) &\geq \E \!\int \! {\cal D}u dP_{A_2}(\ba)\ln \frac{1}{\sqrt{2\pi \Delta}} e^{-\frac{1}{2\Delta}(\widetilde{Y}_0 - \varphi_2(\sqrt{q}\, V + \sqrt{\rho - q}\, u,\ba))^2}\\
&\geq -\frac{1}{2 \Delta} \E \!\int \!\!{\cal D}u dP_{A_2}(\ba) \big(\varphi_2(\sqrt{q}\, V + \sqrt{\rho - q}\, U, \ba) - \varphi_2(\sqrt{q}\, V + \sqrt{\rho - q}\, u, \ba)\big)^2\\
&\quad\,-\frac{1 + \ln 2\pi \Delta}{2}\,.
\end{align*}
Put together, these lower and upper bounds give the lemma.
\end{proof}
\end{lemma}
\noindent To conclude on that section, the interpolating model is such that:
\begin{itemize}
\item \textbf{at t=0}, it reduces to the two-layer GLM;
\item \textbf{at t=1}, it reduces to one scalar inference channel associated to the term $\Psi_{P_{\rm out,2}}$ plus one-layer GLM whose formula for the free entropy $\tilde{f}_{n_0,n_1}$ in the thermodynamic limit is already known from \cite{BarbierOneLayerGLM}.
\end{itemize}

\subsection{Free entropy variation along the interpolation path}
From the Fundamental Theorem of Calculus and $\eqref{eq:f0_f1}$, $\eqref{eq:f1_thermoLimit}$
\begin{align}
f_{\mathbf{n}} &= f_{\mathbf{n},\epsilon}(0) + \frac{1}{2}\frac{n_1}{n_0} + \smallO_{n_0}(1)\nonumber\\
&= f_{\mathbf{n},\epsilon}(1) - \int_0^1\frac{df_{\mathbf{n},\epsilon}(t)}{dt} dt
+ \frac{1}{2}\alpha_1 + \smallO_{n_0}(1)\nonumber\\
&= -\frac{\alpha_1}{2} \rho_1 \int_0^1 \!\! r_{\epsilon}(t) dt
+ \alpha \Psi_{P_{\rm out, 2}}\bigg(\int_0^1 q_{\epsilon}(t) dt;\rho_1(n_0)\bigg)
- \int_0^1\frac{df_{\mathbf{n},\epsilon}(t)}{dt} dt\nonumber\\
&\qquad\qquad + \adjustlimits{\sup}_{q_0 \in [0,\rho_0]} {\inf}_{r_0 \geq 0} \; \bigg\{\psi_{P_0}(r_0) + \alpha_1 \Psi_{\varphi_1}\bigg(q_0, \int_0^1 \!\! r_{\epsilon}(t) dt;\rho_0\bigg)- \frac{r_0 q_0}{2} \bigg\}
+ \smallO_{n_0}(1)\,. \label{f0_f1_int}
\end{align}
Most of the terms that form the potential $\eqref{frs}$ can already be identified in the expression $\eqref{f0_f1_int}$. For the missing terms to appear, the t-derivative of the free entropy has to be computed first.\\

Recall the definitions:
\begin{itemize}
	\item $u_y(x) \defeq \ln P_{\rm out,2}(y|x)$ and $u'_y(x)$ is its derivative (w.r.t.\ $x$).
	\item The overlap $\widehat{Q}$:
	\begin{align*}
	\widehat{Q} \defeq \frac{1}{n_1}\sum_{i=1}^{n_1}
	\varphi_1\bigg(\bigg[\frac{\bW{1} \bx}{\sqrt{n_0}}\bigg]_i,\ba_{1,i}\bigg) \varphi_1\bigg(\bigg[\frac{\bW{1} \bX^0}{\sqrt{n_0}}\bigg]_i, \bA_{1,i}\bigg)\,.
	\end{align*}
\end{itemize}
In Appendix \ref{appendix_interpolation} we show
\begin{proposition}[Free entropy variation] \label{prop:der_f_t}
	The derivative of the free entropy \eqref{ft} verifies, for all $t \in (0,1)$,
	\begin{multline}\label{eq:der_f_t}
		\frac{df_{\mathbf{n},\epsilon}(t)}{dt} = 
-\frac{1}{2} \frac{n_1}{n_0}\E\Bigg[\Bigg\langle
\Bigg(\frac{1}{n_1}\sum_{\mu=1}^{n_2} u_{Y_{t,\epsilon,\mu}}'( S_{t,\epsilon,\mu} )u_{Y_{t,\epsilon,\mu}}'(s_{t,\epsilon,\mu}) - r_{\epsilon}(t)\Bigg)
\Big( \widehat{Q}- q_{\epsilon}(t)\Big)
\Bigg\rangle_{\! \mathbf{n},t,\epsilon}\,\Bigg]\\
+ \frac{n_1}{n_0} \frac{r_{\epsilon}(t)}{2} \big(q_{\epsilon}(t)-\rho_1(n_0)\big) + \smallO_{n_0}(1)\,,
	\end{multline}
	where $\smallO_{n_0}(1)$ is a quantity that goes to $0$ in the limit $n_0, n_1, n_2 \to +\infty$, \textbf{uniformly} in $t \in [0,1]$ and $\epsilon \in \mathcal{B}_{n_0}$.
\end{proposition}

\subsection{Overlap concentration}
An important quantity appearing naturally in the $t$-derivative of the average free entropy is $\widehat{Q}$, the overlap between the hidden output $\bX^1$ and the sample
$\bx^1  \defeq \varphi_{1}\big(\nicefrac{\bW{1} \bx}{\sqrt{n_0}}, \ba_{1}\big)$ where the triplet $(\bx,\bu,\ba_1)$ is sampled from the posterior distribution associated to the Gibbs bracket $\langle - \rangle_{\mathbf{n},t, \epsilon}$.
The next proposition mirrors Proposition 4 in \cite{BarbierOneLayerGLM} and states that the overlap concentrates around its mean.
\begin{proposition}[Overlap concentration] \label{prop:concentrationOverlap}
Let $s_{n_0} = \frac{1}{2} n_0^{-\nicefrac{1}{16}}$.
Assume that ~\ref{hyp:bounded},~\ref{hyp:c2},~\ref{hyp:phi_gauss2} hold and that $(q_{\epsilon})_{\epsilon \in {\cal B}_{n_0}}$, $(r_{\epsilon})_{\epsilon \in {\cal B}_{n_0}}$ are regular.
There exists a constant $C(\varphi_1,\varphi_2,\alpha_1,\alpha_2, S)$ independent of $t$ such that
\begin{equation}
\frac{1}{s_{n_0}^2}\int_{{\cal B}_{n_0}} \!\!\!\! d\epsilon \int_0^1 \!\! dt\,
\E\big[\big\langle \big(\widehat{Q} - \E\big[\langle \widehat{Q} \rangle_{\mathbf{n},t,\epsilon}\big]\big)^2\big\rangle_{\mathbf{n},t,\epsilon}\big]
\leq \frac{C(\varphi_1,\varphi_2,\alpha_1,\alpha_2, S)}{n_0^{\nicefrac{1}{8}}}\;.
\end{equation}
\begin{proof}
It is a simple consequence of Proposition \ref{L-concentration} and the bound \eqref{boundConcentrationOverlapL}, both in Appendix~\ref{appendix-overlap}, combined with Fubini's theorem.
\end{proof}
\end{proposition}
Note from \eqref{frs} and \eqref{f0_f1_int} that the second summand in \eqref{eq:der_f_t} is precisely the term needed to obtain the expression of the potential on the r.h.s.\ of \eqref{f0_f1_int}. We would like to ``cancel'' the Gibbs bracket in \eqref{eq:der_f_t} to prove Theorem~\ref{th:RS_2layer}.
One way to cancel this so-called remainder is to choose $q_\epsilon(t) = \E\big[\langle \widehat{Q} \rangle_{\mathbf{n},t,\epsilon}\big]$ on which $\widehat{Q}$ concentrates by Proposition~\ref{prop:concentrationOverlap}.
However, $\E\big[\langle \widehat{Q} \rangle_{\mathbf{n},t,\epsilon}\big]$ depends on $\int_0^t q_\epsilon(v)dv$, as well as $t$, $\epsilon$ and $\int_0^t r_\epsilon(v)dv$. The equation $q_\epsilon(t) = \E\big[\langle \widehat{Q} \rangle_{\mathbf{n},t,\epsilon}\big]$ is therefore a first order differential equation over $t \mapsto \int_0^t q_\epsilon(v)dv$.
This will be addressed in details in the next section.
For now we assume that we can take $q_\epsilon(t) = \E\big[\langle \widehat{Q} \rangle_{\mathbf{n},t,\epsilon}\big]$ and prove
\begin{proposition}\label{prop:cancel_remainder}
Assume that ~\ref{hyp:bounded},~\ref{hyp:c2},~\ref{hyp:phi_gauss2} hold, that the interpolation functions $(q_{\epsilon})_{\epsilon \in {\cal B}_{n_0}}$, $(r_{\epsilon})_{\epsilon \in {\cal B}_{n_0}}$ are regular, and that $\forall (t,\epsilon) \in [0,1] \times \mathcal{B}_{n_0}: q_{\epsilon}(t) = \E \langle \widehat{Q} \rangle_{n,t,\epsilon}$. Then
\begin{multline}
f_{\mathbf{n}}
= \smallO_{n_0}(1) + \frac{1}{s_{n_0}^2}\int_{{\cal B}_{n_0}} \!\!\!\! d\epsilon
	\Bigg[\alpha \Psi_{P_{\rm out, 2}}\bigg(\int_0^1 q_{\epsilon}(t) dt;\rho_1(n_0)\bigg)
	- \frac{\alpha_1}{2} \int_0^1 \!\! dt \,
	r_{\epsilon}(t)q_{\epsilon}(t)\\
+ \adjustlimits{\sup}_{q_0 \in [0,\rho_0]} {\inf}_{r_0 \geq 0} \; \bigg\{\psi_{P_0}(r_0) + \alpha_1 \Psi_{\varphi_1}\bigg(q_0, \int_0^1 \!\! r_{\epsilon}(t) dt;\rho_0\bigg)- \frac{r_0 q_0}{2} \bigg\} \Bigg]
\end{multline}
where $\smallO_{n_0}(1)$ is a quantity that vanishes as $n_0 \to \infty$, uniformly w.r.t.\ the choice of the interpolation functions.
\end{proposition}
\begin{proof}
By the Cauchy-Schwarz inequality
\begin{multline}\label{CS_remainder}
	\Bigg(\frac{1}{s_{n_0}^2}
	\int_{{\cal B}_{n_0}} \!\!\!\! d\epsilon\int_0^1 \!\! dt\,
	\E\Bigg[\Bigg\langle
	\Bigg(\frac{1}{n_1}\sum_{\mu=1}^{n_2} u_{Y_{t,\epsilon,\mu}}'( S_{t,\epsilon,\mu} )u_{Y_{t,\epsilon,\mu}}'(s_{t,\epsilon,\mu}) - r_{\epsilon}(t)\Bigg)
	\Big( \widehat{Q}- q_{\epsilon}(t)\Big)
	\Bigg\rangle_{\! \mathbf{n},t,\epsilon}\,\Bigg]\Bigg)^{\! 2}\\
	\leq 
	\frac{1}{s_{n_0}^2} \int_{{\cal B}_{n_0}} \!\!\!\! d\epsilon\int_0^1 \!\!dt\, 
	\E\Bigg[\Bigg\langle
	\Bigg(\frac{1}{n_1}\sum_{\mu=1}^{n_2} u_{Y_{t,\epsilon,\mu}}'( S_{t,\epsilon,\mu} )u_{Y_{t,\epsilon,\mu}}'(s_{t,\epsilon,\mu}) - r_{\epsilon}(t)\Bigg)^{\! 2}
	\Bigg\rangle_{\! \mathbf{n},t,\epsilon}\Bigg]\\
	\cdot \frac{1}{s_{n_0}^2}\int_{{\cal B}_{n_0}} \!\!\!\! d\epsilon \int_0^1 \!\! dt\,
	\E\Big[\big\langle \big(\widehat{Q} - q_{\epsilon}(t)\big)^2\big\rangle_{\mathbf{n},t,\epsilon}\Big]\,.
\end{multline}
First we will look at the first term of this product. We have
\begin{multline*}
\E\Bigg[\Bigg\langle
\Bigg(\frac{1}{n_1}\sum_{\mu=1}^{n_2} u_{Y_{t,\epsilon,\mu}}'( S_{t,\epsilon,\mu} )u_{Y_{t,\epsilon,\mu}}'(s_{t,\epsilon,\mu}) - r_{\epsilon}(t)\Bigg)^{\! 2}
\Bigg\rangle_{\! \mathbf{n},t,\epsilon}\Bigg]\\
\leq 2 \frac{n_2}{n_1}\rmax + 2\E\Bigg[\Bigg\langle
\Bigg(\frac{1}{n_1}\sum_{\mu=1}^{n_2} u_{Y_{t,\epsilon,\mu}}'( S_{t,\epsilon,\mu} )u_{Y_{t,\epsilon,\mu}}'(s_{t,\epsilon,\mu})\Bigg)^{\! 2}
\Bigg\rangle_{\! \mathbf{n},t,\epsilon}\Bigg]\,.
\end{multline*}
The second summand on the left-hand side is further upper bounded as:
\begin{align*}
&\E\Bigg[\Bigg\langle
\Bigg(\frac{1}{n_1}\sum_{\mu=1}^{n_2} u_{Y_{t,\epsilon,\mu}}'( S_{t,\epsilon,\mu} )u_{Y_{t,\epsilon,\mu}}'(s_{t,\epsilon,\mu})\Bigg)^{\! 2}
\Bigg\rangle_{\! \mathbf{n},t,\epsilon}\Bigg]\\
&\qquad\qquad\qquad\qquad\qquad\qquad\leq \frac{1}{n_1^2}\E\Big[
\Big\Vert \big\{u_{Y_{t,\epsilon,\mu}}'( S_{t,\epsilon,\mu})\big\}_{\mu=1}^{n_2} \Big\Vert\cdot
\Big\langle\Big\Vert \big\{u_{Y_{t,\epsilon,\mu}}'(s_{t,\epsilon,\mu})\big\}_{\mu=1}^{n_2}\Big\Vert
\Big\rangle_{\! \mathbf{n},t,\epsilon}\Big]\\
&\qquad\qquad\qquad\qquad\qquad\qquad= \frac{1}{n_1^2}\E\Big[
\Big\langle\Big\Vert \big\{u_{Y_{t,\epsilon,\mu}}'(s_{t,\epsilon,\mu})\big\}_{\mu=1}^{n_2}\Big\Vert
\Big\rangle_{\! \mathbf{n},t,\epsilon}^2 \Big]\\
&\qquad\qquad\qquad\qquad\qquad\qquad\leq \frac{1}{n_1^2}\E\Big[
\Big\langle\Big\Vert \big\{u_{Y_{t,\epsilon,\mu}}'(s_{t,\epsilon,\mu})\big\}_{\mu=1}^{n_2}\Big\Vert^2
\Big\rangle_{\! \mathbf{n},t,\epsilon}\Big]\\
&\qquad\qquad\qquad\qquad\qquad\qquad\leq \frac{1}{n_1^2}\E\Big[
\Big\Vert \big\{u_{Y_{t,\epsilon,\mu}}'(S_{t,\epsilon,\mu})\big\}_{\mu=1}^{n_2}\Big\Vert^2
\Big]
= \frac{n_2}{n_1^2}\E\Big[ u_{Y_{t,\epsilon,1}}'(S_{t,\epsilon,1})^2\Big]\,.
\end{align*}
The first inequality follows from the Cauchy-Schwartz inequality, the subsequent equality from the Nishimori identity (Proposition~\ref{prop:nishimori} in Appendix~\ref{app:nishimori}), the second inequality from Jensen's inequality and the last inequality from the Nishimori identity again.

Remember $u_y(x) \defeq \ln P_{\rm out,2}(y|x) = \ln \int dP_{A_2}(\ba)\frac{1}{\sqrt{2\pi \Delta}} e^{-\frac{1}{2 \Delta}( y - \varphi_2(x,\ba))^2}$, so
\begin{align*}
\big\vert u_{Y_{t,\epsilon,1}}^\prime(x) \big\vert
&= \frac{\big\vert \int dP_{A_2}(\ba) \varphi_2^\prime(x,\ba) (Y_{t,\epsilon,1} - \varphi_2(x, \ba)) e^{-\frac{1}{2\Delta} (Y_{t,\epsilon,1} - \varphi_2(x, \ba))^2}\big\vert}
{\Delta \int dP_A(\ba_\mu) e^{-\frac{1}{2} (Y_{t,\mu} - \varphi(s, \ba_\mu))^2}}\\
&\leq \sup \vert\varphi_2^\prime\vert \frac{\vert Z_1 \vert + 2\sup \vert\varphi_2\vert}{\Delta} 
\end{align*}
and $\E\big[ u_{Y_{t,\epsilon,1}}'(S_{t,\epsilon,1})^2\big] \leq \sup \vert\varphi_2^\prime\vert^2 \frac{2 + 8\sup \vert\varphi_2\vert^2}{\Delta^2}$. Putting everything together, we obtain the following bound uniform w.r.t.\ the choice of the interpolating functions:
\begin{multline*}
\frac{1}{s_{n_0}^2} \int_{{\cal B}_{n_0}} \!\!\!\! d\epsilon\int_0^1 \!\!dt\, 
\E\Bigg[\Bigg\langle
\Bigg(\frac{1}{n_1}\sum_{\mu=1}^{n_2} u_{Y_{t,\epsilon,\mu}}'( S_{t,\epsilon,\mu} )u_{Y_{t,\epsilon,\mu}}'(s_{t,\epsilon,\mu}) - r_{\epsilon}(t)\Bigg)^{\! 2}
\Bigg\rangle_{\! \mathbf{n},t,\epsilon}\Bigg]\\
\leq 2 \underbrace{\frac{n_2}{n_1}}_{\to \alpha_1}\Big(\rmax + \frac{1 + 4\sup \vert\varphi_2\vert^2}{\Delta^2 n_1} \sup \vert\varphi_2^\prime\vert^2\Big) \,.
\end{multline*}

With our assumption on $q_{\epsilon}$ and Proposition \ref{prop:concentrationOverlap}, the second term in \ref{CS_remainder} is easily bounded by $C(\varphi_1,\varphi_2,\alpha_1,\alpha_2, S) n^{-\nicefrac{1}{8}}$. Thus
\begin{multline*}
\Bigg\vert\frac{1}{s_{n_0}^2}
\int_{{\cal B}_{n_0}} \!\!\!\! d\epsilon\int_0^1 \!\! dt\,
\E\Bigg[\Bigg\langle
\Bigg(\frac{1}{n_1}\sum_{\mu=1}^{n_2} u_{Y_{t,\epsilon,\mu}}'( S_{t,\epsilon,\mu} )u_{Y_{t,\epsilon,\mu}}'(s_{t,\epsilon,\mu}) - r_{\epsilon}(t)\Bigg)
\Big( \widehat{Q}- q_{\epsilon}(t)\Big)
\Bigg\rangle_{\! \mathbf{n},t,\epsilon}\,\Bigg]\Bigg\vert\\
\leq \frac{C(\varphi_1,\varphi_2,\alpha_1,\alpha_2, S)}{n^{\nicefrac{1}{16}}}\,.
\end{multline*}
Therefore the integral of \eqref{eq:der_f_t} over $(t,\epsilon)$ reads
\begin{align}
	\frac{1}{s_{n_0}^2}\int_{{\cal B}_{n_0}} \!\!\!\! d\epsilon \int_0^1 \!\! dt\, \frac{df_{\mathbf{n},\epsilon}(t)}{dt}
	&= \frac{1}{s_{n_0}^2}\frac{n_1}{2 n_0} \int_{{\cal B}_{n_0}} \!\!\!\! d\epsilon \int_0^1 \!\! dt \,
	\big(r_{\epsilon}(t)q_{\epsilon}(t)-r_{\epsilon}(t)\rho_1(n_0)\big) + \smallO_{n_0}(1)\nn
	&= \frac{1}{s_{n_0}^2}\frac{\alpha_1}{2} \int_{{\cal B}_{n_0}} \!\!\!\! d\epsilon \int_0^1 \!\! dt \,
	\big(r_{\epsilon}(t)q_{\epsilon}(t)-r_{\epsilon}(t)\rho_1\big) + \smallO_{n_0}(1) \,.
	\label{eq:id_fluctuation}
\end{align}
Here the summand $\smallO_{n_0}(1)$ vanishes uniformly w.r.t.\ to the choice of $q_{\epsilon}$ and $r_{\epsilon}$. Replacing \eqref{eq:id_fluctuation} in \eqref{f0_f1_int} leads to the claimed identity.
\end{proof}

\subsection{Lower and upper matching bounds}\label{subsec:lower-upper}
To end the proof of Theorem~\ref{th:RS_2layer} one has to go through the following two steps:  
\begin{enumerate}[label=(\roman*)]
	\item\label{item:step1} Prove that under the assumptions~\ref{hyp:bounded},~\ref{hyp:c2} and \ref{hyp:phi_gauss2}
	\begin{equation*}
	\lim_{\mathbf{n} \to\infty }f_{\mathbf{n}} 
	=  \adjustlimits{\sup}_{r_1 \geq 0} {\inf}_{q_1 \in [0,\rho_1]} \adjustlimits{\sup}_{q_0 \in [0,\rho_0]} {\inf}_{r_0 \geq 0} f_{\rm RS}(q_0,r_0,q_1,r_1;\rho_0,\rho_1)\,.
	\end{equation*}
	\item\label{item:step2}Invert the order of the optimizations on $r_1$ and $q_1$.
\end{enumerate}
To tackle~\ref{item:step1}, we prove that $\liminf_{n_0 \to \infty} f_{\mathbf{n}}$ and $\limsup_{n_0 \to \infty} f_{\mathbf{n}}$ are -- respectively -- lower and upper bounded by the same quantity $\sup\limits_{r_1 \geq 0} \inf\limits_{q_1 \in [0,\rho_1]}  \sup\limits_{q_0 \in [0,\rho_0]} \inf\limits_{r_0 \geq 0} f_{\rm RS}(q_0,r_0,q_1,r_1;\rho_0,\rho_1)$.

In these proofs we will need the following useful proposition. For $t \in [0,1]$ and $\epsilon \in \mathcal{B}_{n_0}$, we write $R(t,\epsilon) = (R_1(t,\epsilon),R_2(t,\epsilon))$.
$\E\big[\langle \widehat{Q} \rangle_{\mathbf{n},t,\epsilon}\big]$ can be written as a function $\E\big[\langle \widehat{Q} \rangle_{\mathbf{n},t,\epsilon}\big] = F_{\mathbf{n}}\big(t,R(t,\epsilon)\big)$ of $\mathbf{n},t,R(t,\epsilon)$ with $F_{\mathbf{n}}$ a function defined on
\begin{equation}
D_{\mathbf{n}} \defeq \Big\{ (t,r_1,r_2) \in [0,1] \times [0,+\infty)^2 \, \Big\vert \, r_2 \leq \rho_1(n_0) t + 2s_{n_0}\Big\}\,.
\end{equation}
\begin{proposition}\label{prop:F_equadiff}
Let $D_{\mathbf{n}}^{\circ}$ be the interior of the set $D_{\mathbf{n}}$.
$F_\mathbf{n}$ is a continuous function from $D_{\mathbf{n}}$ to $[0,\rho_1(n_0)]$, and admits partial derivatives with respect to its second and third arguments on $D_{\mathbf{n}}^{\circ}$. These partial derivatives are both continuous and non-negative on $D_{\mathbf{n}}^{\circ}$.
\begin{proof}
Let define explicitly the function $F_{\mathbf{n}}: D_{\mathbf{n}}^{\circ} \to \mathbb{R}$.
Fix $(t,r_1,r_2) \in D_{\mathbf{n}}^{\circ}$.
Consider the two sets of observations
\begin{align*}
\begin{cases}
Y_{t,r_2,\mu}  &\sim P_{\rm out, 2}(\ \cdot \ | \, S_{t,r_2,\mu})\,, \qquad\qquad\;\;\;  1 \leq \mu \leq n_2, \\
Y'_{t,r_1,i} &= \sqrt{r_1}\, \varphi_1\Big(\Big[\frac{\bW{1} \bX^0}{\sqrt{n_0}}\Big]_i, \bA_{1,i}\Big) + Z'_i\,, \;\: 1 \leq i \leq n_1,
\end{cases}
\end{align*}
with $\bS_{t,r_2} \defeq \sqrt{\frac{1-t}{n_1}}\, \bW{2} \varphi_1\Big(\frac{\bW{1} \bX^0}{\sqrt{n_0}}, \bA_1\Big)
+ \sqrt{r_2} \,\bV + \sqrt{\rho_1(n_0) t + 2s_{n_0} - r_2} \, \bU \in \mathbb{R}^{n_2}$.
Now, in complete analogy with \eqref{interpolating-ham-general}, we define the Hamiltonian
\begin{multline*}
\cH_{t,r_1,r_2}(\bx,\ba_1, \bu;\bY,\bY,\bW{1},\bW{2},\bV)\\
\defeq -\sum_{\mu=1}^{n_2} \ln P_{\rm out,2} ( Y_{\mu} |s_{t, r_2, \mu})
+ \frac{1}{2} \sum_{i=1}^{n_1}\bigg(Y'_{i}
- \sqrt{r_1}\, \varphi_1\bigg(\bigg[\frac{\bW{1} \bx}{\sqrt{n_0}}\bigg]_i,\ba_{1,i}\bigg)\bigg)^{\! 2} \,.
\end{multline*}
where $\mathbf{s}_{t,r_2} \defeq \sqrt{\frac{1-t}{n_1}}\, \bW{2} \varphi_1\Big(\frac{\bW{1} \bx}{\sqrt{n_0}}, \bA_1\Big)
+ \sqrt{r_2} \,\bV + \sqrt{\rho_1(n_0) t + 2s_{n_0} - r_2}$. The Gibbs bracket corresponding to this Hamiltonian is denoted $\langle - \rangle_{\mathbf{n},t,r_1,r_2}$ and is used to define $F_{\mathbf{n}}$:
\begin{equation*}
F_{\mathbf{n}}(t,r_1,r_2)
\defeq \E\big[\langle \widehat{Q} \rangle_{\mathbf{n},t,r_1,r_2}\big]
= \frac{1}{n_1} \sum_{i=1}^{n_1} \E\big[X_i^1 \langle x_i^1 \rangle_{\mathbf{n},t,r_1,r_2}\big].
\end{equation*}

Now that $F_{\mathbf{n}}$ is defined, its continuity and differentiability properties follow from the standard theorems of continuity and differentiation under the integral sign. The domination hypotheses are easily verified under assumptions~\ref{hyp:bounded},~\ref{hyp:c2}.

The overlap $\E\big[\langle \widehat{Q} \rangle_{\mathbf{n},t,r_1,r_2}\big]$ is related to the minimal-mean-squared error
(using a Nishimori identity to get the second equality)
\begin{align*}
\mathrm{MMSE}\big( \bX^1 \, \big\vert \, \bY_{t,r_2}, \, \bY_{t,r_1}^{\prime},\bW{1},\bW{2},\bV \big)
&= \E \Big[\big\Vert \bX^1 - \E[\bX^1 \, \vert \, \bY_{t,r_2}, \, \bY_{t,r_1}^{\prime},\bW{1},\bW{2},\bV]
\big\Vert^2\Big]\\
&= \E \Big[\big\Vert \bX^1 - \langle \bx_1 \rangle_{\mathbf{n},t,r_1,r_2} \big\Vert^2\Big]\\
&= n_1\big(\rho_1(n_0) - \E\big[\langle \widehat{Q} \rangle_{\mathbf{n},t,r_1,r_2}\big] \big)\,.
\end{align*}
Since $\mathrm{MMSE}\big( \bX^1 \, \big\vert \, \bY_{t,r_2}, \, \bY_{t,r_1}^{\prime},\bW{1},\bW{2},\bV \big) \geq 0$, we obtain that $\E\big[\langle \widehat{Q} \rangle_{\mathbf{n},t,\epsilon}\big] \in [0,\rho_1(n_0)]$.
It remains to show that $\mathrm{MMSE}\big( \bX^1 \, \big\vert \, \bY_{t,r_2}, \, \bY_{t,r_1}^{\prime},\bW{1},\bW{2},\bV \big)$ is separately non-increasing in $r_1$ and $r_2$.

$r_1$ only appears in the definition of $\bY'_{t,r_1} \defeq \sqrt{r_1} \bX^1 + \bZ'$, where $\bZ'$ is a standard Gaussian vector.
Then $\mathrm{MMSE}\big( \bX^1 \, \big\vert \, \bY_{t,r_2}, \, \bY_{t,r_1}^{\prime},\bW{1},\bW{2},\bV \big)$ is clearly a non-increasing function of $r_1$, so $F_{\mathbf{n}}(t,r_1,r_2)$ is a non-decreasing function of $r_1$.

$r_2$'s only role is in the generation process of $\bY_{t,r_2}$:
\begin{align}
\bY_{t,r_2} \sim
P_{\rm out}\Big( \cdot \, \Big\vert \,  \sqrt{\frac{1-t}{n_1}} \bW{2}\bX^1 + \sqrt{r_2} \bV + \sqrt{\rho_1(n_0) t + 2 s_{n_0} - r_2} \bU \Big) .
\end{align}
Let $0 \leq r_2 \leq r_2^{\prime} <\rho_1(n_0)t + 2s_{n_0}$ and $\bV' \sim \cN(0,\mathbf{I}_m)$ independent of everything else.
Define
\begin{equation*}
\widetilde{\bY}_{t,r_2,r_2'}
\sim
P_{\rm out}\Big( \cdot \, \Big\vert \,  \sqrt{\frac{1-t}{n_1}} \bW{2}\bX^1 + \sqrt{r_2} \bV + \sqrt{r_2'-r_2} \bV'
+ \sqrt{\rho_1(n_0) t + 2 s_{n_0} - r'_2} \bU
\Big),
\end{equation*}
independently of everything else. Now notice that
\begin{align*}
\mathrm{MMSE}\big( \bX^1 \, \big\vert \, \bY_{t,r_2}, \, \bY_{t,r_1}^{\prime},\bW{1},\bW{2},\bV \big)
&\geq \mathrm{MMSE}\big( \bX^1 \, \big\vert \, \widetilde{\bY}_{t,r_2,r_2'}, \bY_{t,r_1}',\bW{1},\bW{2},\bV\big)\,;\\
\mathrm{MMSE}\big( \bX^1 \, \big\vert \, \bY_{t,r_2'}, \, \bY_{t,r_1}^{\prime},\bW{1},\bW{2},\bV \big)
&= \mathrm{MMSE}\big( \bX^1 \, \big\vert \, \widetilde{\bY}_{t,r_2,r_2'}, \, \bY_{t,r_1}',\bW{1},\bW{2},\bV,\bV'\big)\,.
\end{align*}
It directly implies $\mathrm{MMSE}\big( \bX^1 \, \big\vert \, \bY_{t,r_2}, \, \bY_{t,r_1}^{\prime},\bW{1},\bW{2},\bV \big)
\geq \mathrm{MMSE}\big( \bX^1 \, \big\vert \, \bY_{t,r_2'}, \, \bY_{t,r_1}^{\prime},\bW{1},\bW{2},\bV \big)$, i.e.\ $F_{\mathbf{n}}(t,r_1,r_2) \leq F_{\mathbf{n}}(t,r_1,r_2')$. It proves $F_{\mathbf{n}}(t,r_1,r_2)$ is a non-decreasing function of $r_2$.	
\end{proof}
\end{proposition}
\begin{proposition}[Lower bound] \label{prop:lower_bound} The free entropy \eqref{f} verifies 
	\begin{align}
		\liminf_{n_0 \to \infty} f_{\mathbf{n}} \geq \adjustlimits{\sup}_{r_1 \geq 0} {\inf}_{q_1 \in [0,\rho_1]} \adjustlimits{\sup}_{q_0 \in [0,\rho_0]} {\inf}_{r_0 \geq 0} f_{\rm RS}(q_0,r_0,q_1,r_1;\rho_0,\rho_1)\,.
	\end{align}
\end{proposition}
\begin{proof}
For $(\epsilon,r) \in \mathcal{B}_{n_0} \times [0,\rmax]$, consider the order-1 differential equation
\begin{equation}\label{eq:equadiff1}
\forall t \in [0,1]: \; y'(t) = \big(r, F_{\mathbf{n}}(t,y(t))\big) \quad \text{with} \quad y(0) = \epsilon \,.
\end{equation}
It admits a unique solution $R(\cdot,\epsilon)=\big(R_1(\cdot,\epsilon), R_2(\cdot, \epsilon)\big)$ by the Cauchy-Lipschitz theorem (see Theorem~3.1, Chapter~V \cite{ode_hartman}), whose hypotheses are satisfied thanks to Proposition \ref{prop:F_equadiff}. Define for $t \in [0,1]$
\begin{equation*}
r_{\epsilon}(t) \defeq R_1'(t,\epsilon) = r
\qquad \text{and} \qquad
q_{\epsilon}(t) \defeq R_2'(t,\epsilon) = F_{\mathbf{n}}(t,R(t,\epsilon)) \in [0,\rho_1(n_0)].
\end{equation*}
Clearly $R_1(t,\epsilon) = \epsilon_1 + \int_0^t r_\epsilon(s) ds$ and $R_2(t,\epsilon) = \epsilon_2 + \int_0^t q_\epsilon(s) ds$. We obtain that for all $t\in [0,1]$
\begin{equation*}
q_{\epsilon}(t)
= F_\mathbf{n}\big(t,(R_1(t,\epsilon),R_2(t,\epsilon))\big)
= \E\big[\langle \widehat{Q} \rangle_{\mathbf{n},t,\epsilon}\big]\,.
\end{equation*}
This is one of the assumption needed to apply Proposition~\ref{prop:cancel_remainder}.

Now we show that the functions $(q_{\epsilon})_{\epsilon \in {\cal B}_{n_0}}$ and $(r_{\epsilon})_{\epsilon \in {\cal B}_{n_0}}$ are regular (see Definition \ref{def:reg}). Fix $t \in [0,1]$. $R \equiv R(t,\cdot):\epsilon \mapsto (R_1(t,\epsilon),R_2(t,\epsilon))$ is the flow of \eqref{eq:equadiff1}, thus it is injective by unicity of the solution, and $\mathcal{C}^1$ by $F_{\mathbf{n}}$'s regularity properties (see Proposition~\ref{prop:F_equadiff}). The determinant of the Jacobian of the flow is given by the Liouville formula (see Corollary~3.1, Chapter~V \cite{ode_hartman}):
\begin{equation*}
\det\bigg(\frac{\partial R}{\partial \epsilon}(\epsilon)\bigg)
= \exp\bigg(
\int_{0}^{t} \!\! ds \frac{\partial F_\mathbf{n}}{\partial r_2}(s,R(s,\epsilon)) \bigg)
\geq 1.
\end{equation*}
and is greater or equal to $1$ because $\nicefrac{\partial F_\mathbf{n}}{\partial r_2}$ is non-negative (see Proposition \ref{prop:F_equadiff}). Since the Jacobian determinant does not vanish the local inversion theorem and the injectivity of the flow imply that $R$ is a $\mathcal{C}^1$-diffeomorphism. Moreover, since also the Jacobian determinant is greater than or equal to $1$ the functions $(q_{\epsilon})_{\epsilon \in {\cal B}_{n_0}}$ and $(r_{\epsilon})_{\epsilon \in {\cal B}_{n_0}}$ are regular.

Proposition~\ref{prop:cancel_remainder} can now be applied to this special choice of regular functions:
\begin{align}
f_{\mathbf{n}}
&= \smallO_{n_0}(1) + \frac{1}{s_{n_0}^2}\int_{{\cal B}_{n_0}} \!\!\!\! d\epsilon
\Bigg[\alpha \Psi_{P_{\rm out, 2}}\bigg(\int_0^1 q_{\epsilon}(t) dt;\rho_1(n_0)\bigg)
- \frac{\alpha_1 r}{2} \int_0^1 \!\! dt \, q_{\epsilon}(t)\nn
&\qquad\qquad\qquad\qquad\qquad\qquad+ \adjustlimits{\sup}_{q_0 \in [0,\rho_0]} {\inf}_{r_0 \geq 0} \; \bigg\{\psi_{P_0}(r_0) + \alpha_1 \Psi_{\varphi_1}\bigg(q_0,r;\rho_0\bigg)- \frac{r_0 q_0}{2} \bigg\} \Bigg]\nn
&= \smallO_{n_0}(1) + \adjustlimits{\sup}_{q_0 \in [0,\rho_0]} {\inf}_{r_0 \geq 0} \; f_{\rm RS}\bigg(q_0,r_0, \int_0^1 q_{\epsilon}(t) dt,r;\rho_0,\rho_1(n_0)\bigg)\nn
&= \smallO_{n_0}(1) + \mkern-64mu\adjustlimits{\phantom{\sup}}_{\phantom{q_0 \in [0,\rho_0]}} {\inf}_{q_1 \in [0,\rho_1(n_0)]} \adjustlimits{\sup}_{q_0 \in [0,\rho_0]} {\inf}_{r_0 \geq 0} \; f_{\rm RS}\big(q_0,r_0,q_1,r;\rho_0,\rho_1(n_0)\big) \label{lowerboundFreeEntropy}
\end{align}
	By a continuity argument
\begin{multline*}
\lim_{n_0 \to \infty} \mkern-54mu\adjustlimits{\phantom{\sup}}_{\phantom{q_0 \in [0,\rho_0]}} {\inf}_{q_1 \in [0,\rho_1(n_0)]} \adjustlimits{\sup}_{q_0 \in [0,\rho_0]} {\inf}_{r_0 \geq 0} \; f_{\rm RS}\Big(q_0,r_0, q_1,r;\rho_0,\rho_1(n_0)\Big) \\
= \mkern-60mu\adjustlimits{\phantom{\sup}}_{\phantom{q_0 \in [0,\rho_0]}} {\inf}_{q_1 \in [0,\rho_1]} \adjustlimits{\sup}_{q_0 \in [0,\rho_0]} {\inf}_{r_0 \geq 0} \; f_{\rm RS}\Big(q_0,r_0, q_1,r;\rho_0,\rho_1\Big) \,.
\end{multline*}
This limit, combined with \eqref{lowerboundFreeEntropy}, gives
\begin{equation}
\liminf_{n_0 \to \infty} f_{\mathbf{n}} \geq \mkern-64mu\adjustlimits{\phantom{\sup}}_{\phantom{q_0 \in [0,\rho_0]}} {\inf}_{q_1 \in [0,\rho_1]} \adjustlimits{\sup}_{q_0 \in [0,\rho_0]} {\inf}_{r_0 \geq 0} \; f_{\rm RS}\Big(q_0,r_0,q_1,r;\rho_0,\rho_1\Big)	\,.
\end{equation}
This is true for all $r \in [0,\rmax]$, thus
\begin{equation}\label{liminf_supOnInterval}
\liminf_{n_0 \to \infty} f_{\mathbf{n}} \geq \adjustlimits{\sup}_{r_1 \in [0,\rmax]} {\inf}_{q_1 \in [0,\rho_1]} \adjustlimits{\sup}_{q_0 \in [0,\rho_0]} {\inf}_{r_0 \geq 0} \; f_{\rm RS}\Big(q_0,r_0,q_1,r_1;\rho_0,\rho_1\Big)	\,.
\end{equation}
It remains to extend the supremum over $r_1 \in [0,\rmax]$ to a supremum over $r_1$ non-negative. Define on $[0,+\infty[ \times [0,\rho_1]$ the function $\psi(r_1, q_1)\defeq f(r_1) + g(q_1) - \frac{\alpha_1}{2} r_1 q_1$ where $f: [0,+\infty[ \to \mathbb{R}$ and $g: [0,\rho_1] \to \mathbb{R}$ are such that
\begin{equation*}
f(r_1) \defeq  \!\!\adjustlimits{\sup}_{q_0 \in [0,\rho_0]} {\inf}_{r_0 \geq 0}  \Big\{\psi_{P_0}(r_0) + \alpha_1 \Psi_{\varphi_1}(q_0, r_1;\rho_0) - \frac{r_0 q_0}{2} \Big\}\,,\;
g(q_1) \defeq  \alpha \Psi_{P_{\rm out,2}}(q_1;\rho_1)\,.
\end{equation*}
Notice that $\psi(r_1, q_1) = \sup_{q_0 \in [0,\rho_0]} \inf_{r_0 \geq 0} f_{\rm RS}(q_0,r_0,q_1,r_1;\rho_0,\rho_1)$. For all $q_1 \in [0,\rho_1]$ we have
$\nicefrac{\partial \psi}{\partial q_1}(r_1,q_1) = \alpha \psi_{P_{\rm out, 2}}^{\prime}(q_1;\rho_1) - \nicefrac{\alpha_1 r_1}{2}$, that is non-positive if $r_1 \geq 2\alpha_2 \psi_{P_{\rm out, 2}}^{\prime}(q_1;\rho_1)$.
The function $q_1 \mapsto \psi_{P_{\rm out, 2}}(q_1;\rho_1)$ is convex (see Proposition 18, Appendix B.2 \cite{BarbierOneLayerGLM}). Therefore its $q_1$-derivative is non-decreasing and $\rmax$, which is greater than $r^* \defeq  2\alpha_2 \psi_{P_{\rm out, 2}}^{\prime}(\rho_1;\rho_1)$ by definition, satisfies $\forall q_1 \in [0,\rho_1]: \rmax \geq 2\alpha_2 \psi_{P_{\rm out, 2}}^{\prime}(q_1;\rho_1)$. Hence
\begin{equation}
\forall r_1 \geq \rmax: \inf_{q_1 \geq [0,\rho_1]} \psi(r_1, q_1) = \psi(r_1, \rho_1) \,.
\end{equation}
The latter implies that for all $r_1 \geq \rmax$ we have
\begin{align*}
\inf_{q_1 \geq [0,\rho_1]} \psi(r_1, q_1) - \inf_{q_1 \geq [0,\rho_1]} \psi(\rmax, q_1)
&= \psi(r_1, \rho_1) - \psi(\rmax, \rho_1)\\
&= f(r_1) - f(\rmax) - \frac{\alpha_1 \rho_1}{2}(r_1 - \rmax)\\
&\leq 0 \,.
\end{align*}
The last inequality follows from an application of Proposition~\ref{propertiesThirdChannel} in the Appendix~\ref{app:propertiesThirdChannel}: $f$ is non-decreasing and $\nicefrac{\alpha_1 \rho_1}{2}$-Lipschitz, thus $f(r_1) - f(\rmax) = \vert f(r_1) - f(\rmax) \vert \leq \frac{\alpha_1 \rho_1}{2}(r_1 - \rmax)$ for $r_1 \geq \rmax$. We proved that $r_1 \mapsto \inf_{q_1 \geq [0,\rho_1]} \psi(r_1, q_1)$ is non-increasing on $[\rmax,+\infty)$. Going back to \eqref{liminf_supOnInterval}:
\begin{align*}
\liminf_{n_0 \to \infty} f_{\mathbf{n}}
&\geq \adjustlimits{\sup}_{r_1 \in [0,\rmax]} {\inf}_{q_1 \in [0,\rho_1]} \underbrace{\adjustlimits{\sup}_{q_0 \in [0,\rho_0]} {\inf}_{r_0 \geq 0} \; f_{\rm RS}\Big(q_0,r_0,q_1,r_1;\rho_0,\rho_1\Big)}_{= \psi(r_1, q_1)} \\
&= \adjustlimits{\sup}_{r_1 \geq 0} {\inf}_{q_1 \in [0,\rho_1]} \adjustlimits{\sup}_{q_0 \in [0,\rho_0]} {\inf}_{r_0 \geq 0} \; f_{\rm RS}\Big(q_0,r_0,q_1,r_1;\rho_0,\rho_1\Big) \,.
\end{align*}
\end{proof}

\begin{proposition}[Upper bound]\label{prop:upper_bound} The free entropy \eqref{f} verifies
\begin{align}\label{upperbound_limsup}
	\limsup_{n_0 \to \infty} f_{\mathbf{n}}
	\leq \adjustlimits{\sup}_{r_1 \geq 0} {\inf}_{q_1 \in [0,\rho_1]} \adjustlimits{\sup}_{q_0 \in [0,\rho_0]} {\inf}_{r_0 \geq 0} \; f_{\rm RS}\Big(q_0,r_0,q_1, r_1;\rho_0,\rho_1\Big)\,.
\end{align}		
\end{proposition}
\begin{proof}
For $(\epsilon,r) \in \mathcal{B}_{n_0} \times [0,\rmax]$, consider the order-1 differential equation
\begin{equation}\label{eq:equadiff1}
\forall t \in [0,1]: \; y'(t) = \left(
\begin{array}{c}
2\alpha_2 \Psi'_{P_{\rm out, 2}}\big(F_{\mathbf{n}}(t,y(t)),\rho_1(n_0)\big) \\
F_{\mathbf{n}}(t,y(t))
\end{array}\right)
\quad \text{with} \quad y(0) = \epsilon \,.
\end{equation}
The function $\Psi'_{P_{\rm out,2}}$ is $\mathcal{C}^1$ (see Proposition 18, Appendix B.2 \cite{BarbierOneLayerGLM}) and takes values in $[0,r^*(n_0)]$. The later combined with the properties of $F_\mathbf{n}$ (see Proposition \ref{prop:F_equadiff}) allow us to apply the Cauchy-Lipschitz Theorem. The differential equation \eqref{eq:equadiff1} admits a unique solution $R(\cdot,\epsilon)=\big(R_1(\cdot,\epsilon), R_2(\cdot, \epsilon)\big)$. Define for $t \in [0,1]$
\begin{align*}
r_{\epsilon}(t) &\defeq R_1'(t,\epsilon) = 2\alpha_2 \Psi'_{P_{\rm out, 2}}\big(F_{\mathbf{n}}(t,y(t)),\rho_1(n_0)\big) \in [0, r^*(n_0)] \subseteq [0,\rmax]\,;\\
q_{\epsilon}(t) &\defeq R_2'(t,\epsilon) = F_{\mathbf{n}}(t,R(t,\epsilon)) \in [0,\rho_1(n_0)]\,.
\end{align*}
Clearly $R_1(t,\epsilon) = \epsilon_1 + \int_0^t q_\epsilon(s) ds$ and $R_2(t,\epsilon) = \epsilon_2 + \int_0^t r_\epsilon(s) ds$. We obtain that for all $t\in [0,1]$
\begin{equation*}
q_{\epsilon}(t)
= F_\mathbf{n}\big(t,(R_1(t,\epsilon),R_2(t,\epsilon))\big)
= \E\big[\langle \widehat{Q} \rangle_{\mathbf{n},t,\epsilon}\big]\,.
\end{equation*}
This is one of the assumption needed to apply Proposition~\ref{prop:cancel_remainder}.

Now we show that the functions $(q_{\epsilon})_{\epsilon \in {\cal B}_{n_0}}$ and $(r_{\epsilon})_{\epsilon \in {\cal B}_{n_0}}$ are regular. Fix $t \in [0,1]$. $R \equiv R(t,\cdot):\epsilon \mapsto (R_1(t,\epsilon),R_2(t,\epsilon))$ is the flow of \eqref{eq:equadiff1}, thus it is injective and $\mathcal{C}^1$ by $F_{\mathbf{n}}$'s regularity properties (see Proposition~\ref{prop:F_equadiff}). The Jacobian of the flow is given by the Liouville formula (see Corollary~3.1, Chapter~V \cite{ode_hartman}):
\begin{equation*}
\det\bigg(\frac{\partial R}{\partial \epsilon}(\epsilon)\bigg)
= \exp\bigg(
\int_{0}^{t} \!\!\! ds\, 
2\alpha_2 \Psi''_{P_{\rm out,2}}\big(F_\mathbf{n}(s,R(s,\epsilon))\big) \frac{\partial F_\mathbf{n}}{\partial r_1}(s,R(s,\epsilon))
+
\frac{\partial F_\mathbf{n}}{\partial r_2}(s,R(s,\epsilon)) \bigg) \,.
\end{equation*}
$\nicefrac{\partial F_\mathbf{n}}{\partial r_1}$ and $\nicefrac{\partial F_\mathbf{n}}{\partial r_2}$ are both non negative (see Proposition \ref{prop:F_equadiff}) as well as $\Psi_{P_{\rm out,2}}''$ (see Proposition 18, Appendix B.2 \cite{BarbierOneLayerGLM}). Therefore the Jacobian $\det\big(\nicefrac{\partial R}{\partial \epsilon}\big)$ is greater than or equal to 1.
We obtain that $R$ is a $\mathcal{C}^1$-diffeomorphism (by the local inversion Theorem), and since its Jacobian is greater than or equal to $1$ the functions $(q_{\epsilon})_{\epsilon \in {\cal B}_{n_0}}$ and $(r_{\epsilon})_{\epsilon \in {\cal B}_{n_0}}$ are regular.
	
Proposition~\ref{prop:cancel_remainder} applied to this special choice of regular functions $(q_{\epsilon},r_{\epsilon})_{\epsilon \in {\cal B}_{n_0}}$ gives
\begin{align}
f_{\mathbf{n}}
&= \smallO_{n_0}(1) + \frac{1}{s_{n_0}^2}\int_{{\cal B}_{n_0}} \!\!\!\! d\epsilon
\Bigg[\alpha \Psi_{P_{\rm out, 2}}\bigg(\int_0^1 q_{\epsilon}(t) dt;\rho_1(n_0)\bigg)
- \frac{\alpha_1}{2} \int_0^1 \!\! dt \,
r_{\epsilon}(t)q_{\epsilon}(t)\nn
&\qquad\qquad\qquad\qquad\quad+ \adjustlimits{\sup}_{q_0 \in [0,\rho_0]} {\inf}_{r_0 \geq 0} \; \bigg\{\psi_{P_0}(r_0) + \alpha_1 \Psi_{\varphi_1}\bigg(q_0, \int_0^1 \!\! r_{\epsilon}(t) dt;\rho_0\bigg)- \frac{r_0 q_0}{2} \bigg\} \Bigg]\nn
&\leq \smallO_{n_0}(1) + \frac{1}{s_{n_0}^2}\int_{{\cal B}_{n_0}} \!\!\!\! d\epsilon
\int_0^1 \!\! dt
\Bigg[\alpha \Psi_{P_{\rm out, 2}}\big(q_{\epsilon}(t);\rho_1(n_0)\big)
- \frac{\alpha_1}{2}r_{\epsilon}(t)q_{\epsilon}(t)\nn
&\qquad\qquad\qquad\qquad\quad+ \adjustlimits{\sup}_{q_0 \in [0,\rho_0]} {\inf}_{r_0 \geq 0} \; \bigg\{\psi_{P_0}(r_0) + \alpha_1 \Psi_{\varphi_1}\big(q_0, r_{\epsilon}(t);\rho_0\big)- \frac{r_0 q_0}{2} \bigg\} \Bigg]\,.
\label{upperbound_fn_integral}
\end{align}
To obtain this last inequality we applied Jensen's inequality to the convex functions:
\begin{align*}
q \in [0,\rho_1(n_0)] &\mapsto \alpha \Psi_{P_{\rm out, 2}}\big(q;\rho_1(n_0)\big) \;,\\
r \in [0,\rmax] &\mapsto \adjustlimits{\sup}_{q_0 \in [0,\rho_0]} {\inf}_{r_0 \geq 0} \; \bigg\{\psi_{P_0}(r_0) + \alpha_1 \Psi_{\varphi_1}\big(q_0, r;\rho_0\big)- \frac{r_0 q_0}{2} \bigg\}\,.
\end{align*}
Fix $(t,\epsilon) \in [0,1] \times \mathcal{B}_{n_0}$. We chose the interpolation functions $r_\epsilon$ and $q_\epsilon$ such that
\begin{equation*}
r_\epsilon(t) = 2 \alpha_2 \psi_{P_{\rm out, 2}}^{\prime}(q_\epsilon(t),\rho_1(n_0)) \,.
\end{equation*}
Therefore $q_\epsilon(t)$ is a critical point of the convex function
\begin{equation*}
q_1 \in [0,\rho_1(n_0)] \mapsto \alpha \psi_{P_{\rm out, 2}}(q_1,\rho_1(n_0)) - \frac{\alpha_1}{2}r_{\epsilon}(t)q_1 \,,
\end{equation*}
so that we have
\begin{equation*}
\alpha \psi_{P_{\rm out, 2}}(q_\epsilon(t),\rho_1(n_0)) - \frac{\alpha_1}{2}r_{\epsilon}(t)q_\epsilon(t) = \inf_{q_1 \in [0,\rho_1(n_0)]} \Big\{\alpha \psi_{P_{\rm out, 2}}(q_1,\rho_1(n_0)) - \frac{\alpha_1}{2}r_{\epsilon}(t)q_1 \Big\} \:.
\end{equation*}
Plugging this back in \eqref{upperbound_fn_integral}, we get
\begin{align*}
f_{\mathbf{n}}
&\leq \smallO_{n_0}(1) + \frac{1}{s_{n_0}^2}\int_{{\cal B}_{n_0}} \!\!\!\! d\epsilon
\int_0^1 \!\! dt
\Bigg[\inf_{q_1 \in [0,\rho_1(n_0)]} \Big\{\alpha \psi_{P_{\rm out, 2}}(q_1,\rho_1(n_0)) - \frac{\alpha_1}{2}r_{\epsilon}(t)q_1 \Big\}\nn
&\qquad\qquad\qquad\qquad\quad+ \adjustlimits{\sup}_{q_0 \in [0,\rho_0]} {\inf}_{r_0 \geq 0} \; \bigg\{\psi_{P_0}(r_0) + \alpha_1 \Psi_{\varphi_1}\big(q_0, r_{\epsilon}(t);\rho_0\big)- \frac{r_0 q_0}{2} \bigg\} \Bigg]\nn
&\leq \smallO_{n_0}(1) + \frac{1}{s_{n_0}^2}\int_{{\cal B}_{n_0}} \!\!\!\! d\epsilon
\int_0^1 \!\! dt
\mkern-59mu\adjustlimits{\phantom{\sup}}_{\phantom{q_0 \in [0,\rho_0]}} {\inf}_{q_1 \in [0,\rho_1(n_0)]} \adjustlimits{\sup}_{q_0 \in [0,\rho_0]} {\inf}_{r_0 \geq 0}
f_{\rm RS}\Big(q_0,r_0,q_1, r_{\epsilon}(t);\rho_0,\rho_1(n_0)\Big)\nn
&\leq \smallO_{n_0}(1) +
\adjustlimits{\sup}_{r_1 \geq 0} {\inf}_{q_1 \in [0,\rho_1(n_0)]} \adjustlimits{\sup}_{q_0 \in [0,\rho_0]} {\inf}_{r_0 \geq 0} f_{\rm RS}\Big(q_0,r_0,q_1, r_1;\rho_0,\rho_1(n_0)\Big) \,.
\end{align*}
The upperbound \eqref{upperbound_limsup} follows.
\end{proof}
\noindent It remains to prove \ref{item:step2}, i.e.\
\begin{proposition}[Switch the optimization order] Under the assumptions~\ref{hyp:bounded},~\ref{hyp:c2}, we have
		\begin{align*}
			\adjustlimits{\sup}_{r_1 \geq 0} {\inf}_{q_1 \in [0,\rho_1]}  \adjustlimits{\sup}_{q_0 \in [0,\rho_0]} {\inf}_{r_0 \geq 0} &f_{\rm RS}(q_0,r_0,q_1,r_1;\rho_0,\rho_1)\\
			&=\adjustlimits{\sup}_{q_1 \in [0,\rho_1]} {\inf}_{r_1 \geq 0} \adjustlimits{\sup}_{q_0 \in [0,\rho_0]} {\inf}_{r_0 \geq 0} f_{\rm RS}(q_0,r_0,q_1,r_1;\rho_0,\rho_1) \,.
		\end{align*}
for any positive real numbers $\rho_0$ and $\rho_1$.
\begin{proof}
Let $f: [0,+\infty[ \to \mathbb{R}$ and $g: [0,\rho_1] \to \mathbb{R}$ be the two functions
\begin{equation*}
	f(r_1) \defeq  \!\!\adjustlimits{\sup}_{q_0 \in [0,\rho_0]} {\inf}_{r_0 \geq 0}  \Big\{\psi_{P_0}(r_0) + \alpha_1 \Psi_{\varphi_1}(q_0, r_1;\rho_0) - \frac{r_0 q_0}{2} \Big\}\,,\;
	g(q_1) \defeq  \alpha \Psi_{P_{\rm out,2}}(q_1;\rho_1)\,,
\end{equation*}
such that $\psi(r_1, q_1)
\defeq \sup_{q_0 \in [0,\rho_0]} \inf_{r_0 \geq 0} f_{\rm RS}(q_0,r_0,q_1,r_1;\rho_0,\rho_1)
= f(r_1) + g(q_1) - \frac{\alpha_1}{2} r_1 q_1$.\newline
In Appendix~\ref{app:propertiesThirdChannel} it is shown that, under ~\ref{hyp:c2}, $f$ is convex, Lipschitz and non-decreasing on $[0,+\infty[$. 
Proposition 18 in Appendix B.2 of \cite{BarbierOneLayerGLM} states that, under~\ref{hyp:c2}, $g$ is convex, Lipschitz and non-decreasing on $[0,\rho_1]$. The desired result is then obtained by applying Corollary 7 in Appendix D of \cite{BarbierOneLayerGLM}:
\begin{equation*}
	\adjustlimits{\sup}_{r_1 \geq 0} {\inf}_{q_1 \in [0,\rho_1]}  \psi(r_1, q_1)
	= \adjustlimits{\sup}_{q_1 \in [0,\rho_1]} {\inf}_{r_1 \geq 0} \psi(r_1, q_1) \;.
\end{equation*}
\end{proof}
\end{proposition}
\section{Numerical experiments}


\subsection{Activations comparison in terms of mutual informations}
\label{sec:micomp}
Here we assume the exact same setting as the one presented in the main text to compare activation functions on a two-layer random weights network. We compare here the mutual information estimated with the proposed replica formula instead of the entropy behaviors discussed in the main text. As it was the case for entropies, we can see that the saturation of the double-side saturated hardtanh leads to a loss of information for large weights, while the mutual informations are always increasing for linear and ReLU activations.

\begin{figure}[t!]
      \includegraphics[width=1.01\linewidth]{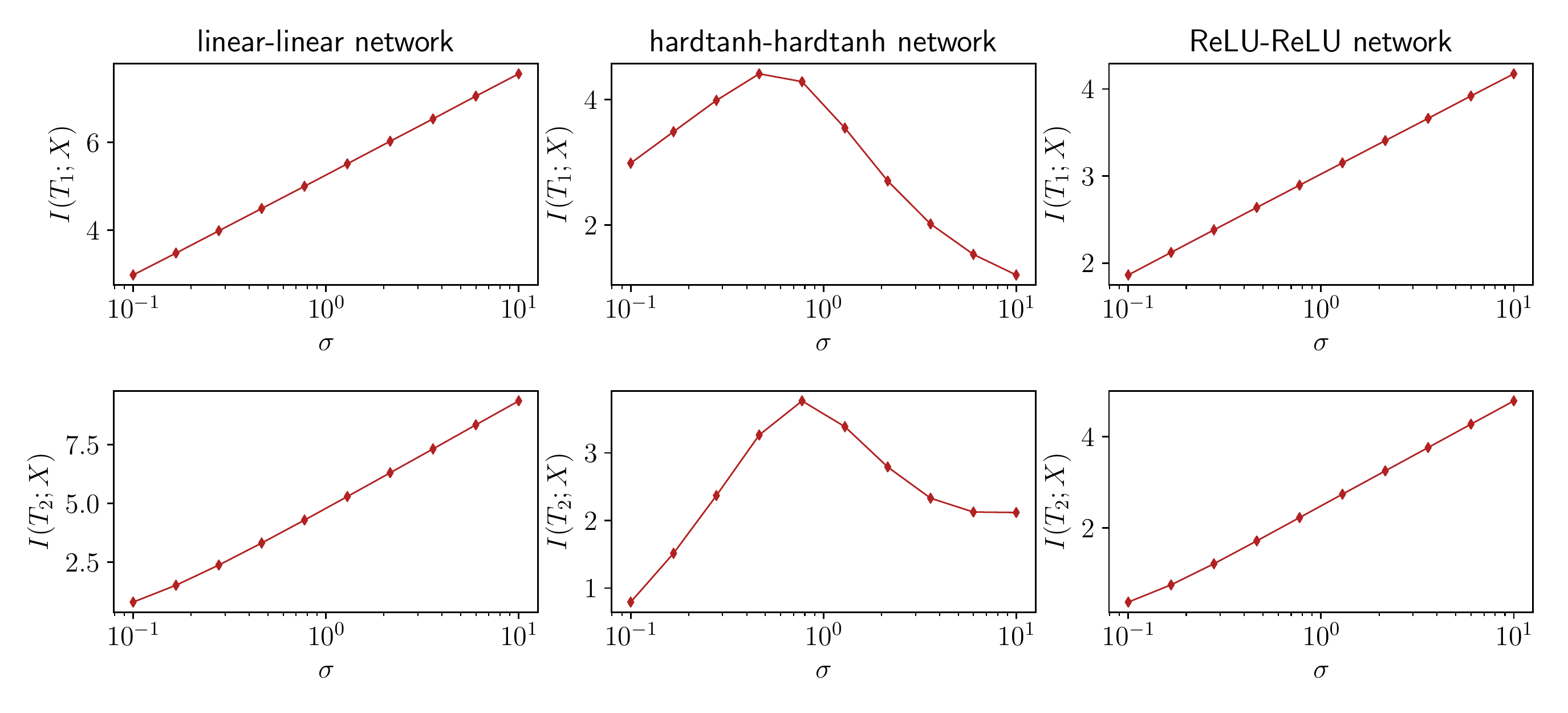}
     \caption{
        Replica mutual informations between latent and input variables in stochastic networks $\X \rightarrow
        \T_1 \rightarrow \T_2 $, with equally sized layers $n=1000$, inputs
        drawn from $\mathcal{N}(0, I_{n})$, weights from  $\mathcal{N}(0,
        \sigma^2 I_{n^2} / n)$, as a function of the weight scaling
        parameter $\sigma$. An additive white Gaussian noise $\mathcal{N}(0,
        10^{-5} I_{n})$ is added inside the non-linearity. Left column: linear
        network. Center column: hardtanh-hardtanh network. Right column:
        ReLU-ReLU network.
       \label{fig:comparisons_sm}}     
\end{figure}

\subsection{Learning ability of USV-layers}
\label{sec:mnist}

To ensure weight matrices remain close enough to being independent during learning we introduce USV-layers, corresponding to a custom type of weight constraint. We recall that in such layers, weight matrices are decomposed in the manner of a singular value decomposition, $W_{\ell} = U_{\ell}S_{\ell}V_{\ell}$, with $U_{\ell}$ and $V_{\ell}$ drawn from the corresponding Haar measures (i.e. uniformly among the orthogonal matrices of given size), and $S_{\ell}$ contrained to be diagonal, being the only matrix being learned.
In the main text, we demonstrate on a linear network that the USV-layers ensure that the assumptions necessary to our replica formula are met with learned matrices in the case of linear networks. 
Nevertheless, a USV-layer of size $n \times n$ has only $n$ trainable parameters, which implies that they are harder to train than usual fully connected layers. In practice, we notice that they tend to require more parameter updates and that interleaving linear USV-layers to increase the number of parameters between non-linearities can significantly improve the final result of training. 

To convince ourselves that the training ability of USV-layers is still relevant to study learning dynamics on real data we conduct an experiment on the MNIST dataset. We study the classification problem of the classical MNIST data set ($60\,000$ training images and $10\,000$ testing images) with a simple fully-connected network featuring one non-linear (ReLU) hidden layer of 500 neurones. On top of the ReLU-layer, we place a softmax output layer where the $500 \times 10$ parameters of the weight matrix are all being learned in all the versions of the experiments. Conversely, before the ReLU layer, we either (1) do not learn at all the $784 \times 500$ parameters which then define random projections, (2) learn all of them as a traditional fully connected network, (3) use a combination of 2 (3a), 3 (3b) or 6 (3c) consecutive USV-layers (without any intermediate non-linearity). The best train and test, losses and accuracies, for the different architectures are given in Table \ref{tab:mnist} and some learning curves are displayed on Figure \ref{fig:mnist}. As expected we observe that USV-layers are achieving better classification success than the random projections, yet worse than the unconstrained fully connected layer. Interestingly, stacking USV-layers to increase the number of trainable parameters allows to reach very good training accuracies, nevertheless, the testing accuracies do not benefit to the same extent from these additional parameters.
On Figure \ref{fig:mnist}, we can actually see that the version of the experiment with 6 USV-layers overfits the training set (green curves with testing losses growing towards the end of learning). Therefore, particularly in this case, adding regularizers might allow to improve the generalization performances of models with USV-layers.

\begin{figure}[t!] 
  \includegraphics[width=\linewidth]{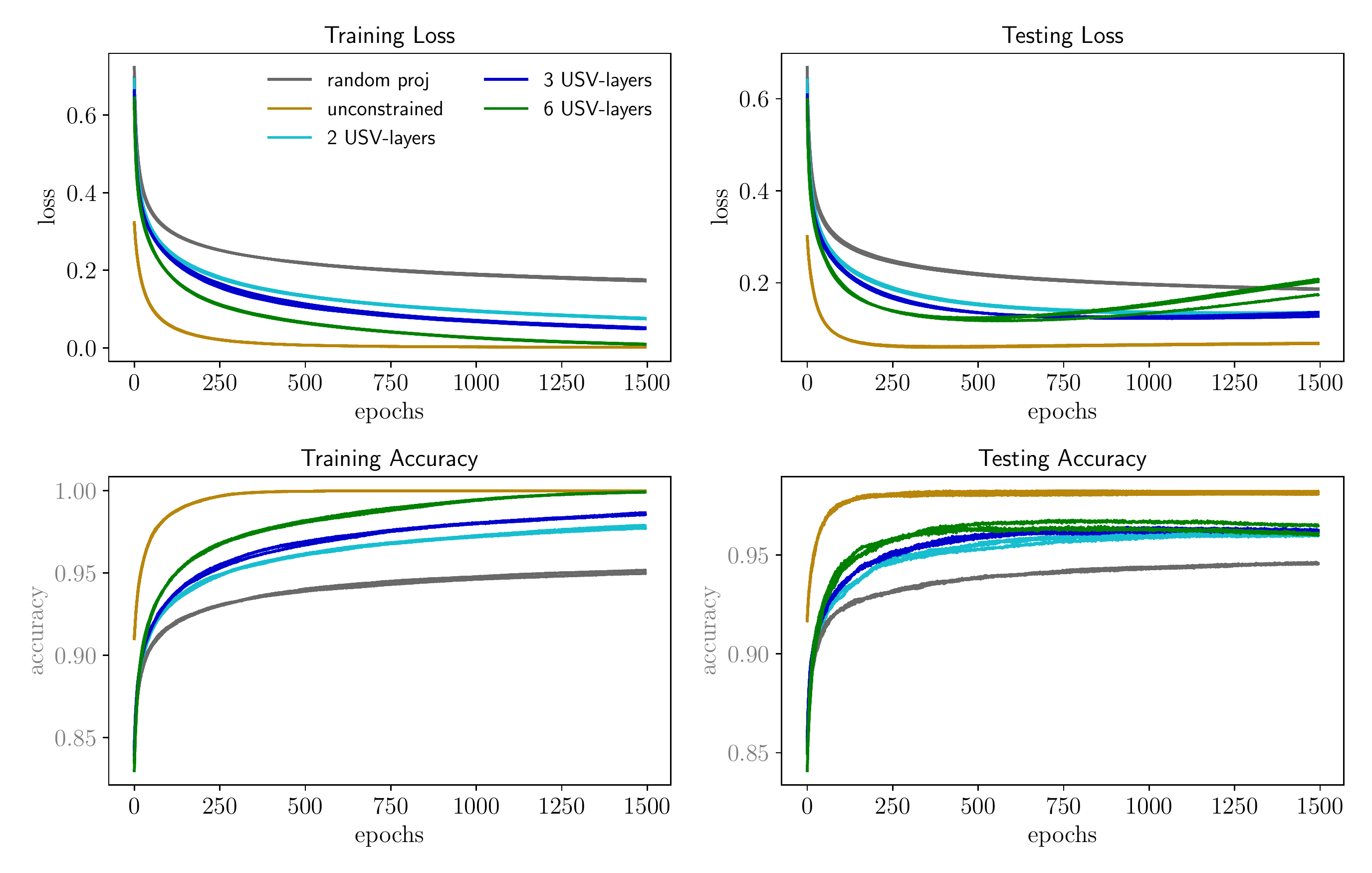}
  \caption{Training and testing curves for the training of a two-layer neural net on the classification of MNIST for different constraints on the first layer (further details are given in Section\ref{sec:mnist}). For each version of the experiment the outcomes of two independent runs are plotted with the same color, it is not always possible to distinguish the two runs as they overlap.}
  \label{fig:mnist}   
\end{figure}

\begin{table*}
\begin{center}
\begin{tabular}{p{0.18\linewidth} |
p{0.14\linewidth}
p{0.1\linewidth}
p{0.1\linewidth}
p{0.12\linewidth}
p{0.12\linewidth}}
\hline
First layer type
& $\#$train params
& Train loss
& Test loss
& Train acc
& Test acc\\
\hline\hline \\
Random (1)
& 0
& 0.1745 
& 0.1860 
& 95.05 (0.09)
& 94.61 (0.02)
\\
Unconstrained (2)
& $784 \times \; 500$
& 0.0012 
& 0.0605 
& 100. (0.00)
& 98.18 (0.06)
\\
2-USV (3a)
& $2 \times \; 500$
& 0.0758 
& 0.1326 
& 97.80 (0.07)
& 96.10 (0.03)
\\
3-USV (3b)
& $3 \times \; 500$
& 0.0501 
& 0.1238 
& 98.62 (0.05)
& 96.35 (0.04)
\\
6-USV (3c)
& $6 \times \; 500$
& 0.0092 
& 0.1211 
& 99.93 (0.01)
& 96.54 (0.17)
\\
\end{tabular}
\end{center}
\caption{Training results for MNIST classification of a fully connected 784-500-10 neural net with a ReLU non linearity. The different rows correspond to different specifications of trainable parameters in the first layer (1, 2, 3a, 3b, 3c) describe in the paragraph. We use plain SGD to minimize the cross-entropy loss. All experiments use the same learning rate 0.01 and batchsize of 100 samples. Results are averaged over 5 independent runs, and standard deviations are reported in parentheses.}
\label{tab:mnist}
\end{table*}

\subsection{Additional learning experiments on synthetic data} 
\label{sec:moreresults}

Similarly to the experiments of the main text, we consider simple training
schemes with constant learning rates, no momentum, and no explicit
regularization.

We first include a second version of Figure 4 of the main text, corresponding to the exact same experiment with a different random seed and check that results are qualitatively identical.

\begin{figure}[t!]
    \begin{center}
      \includegraphics[width=1.01\linewidth]{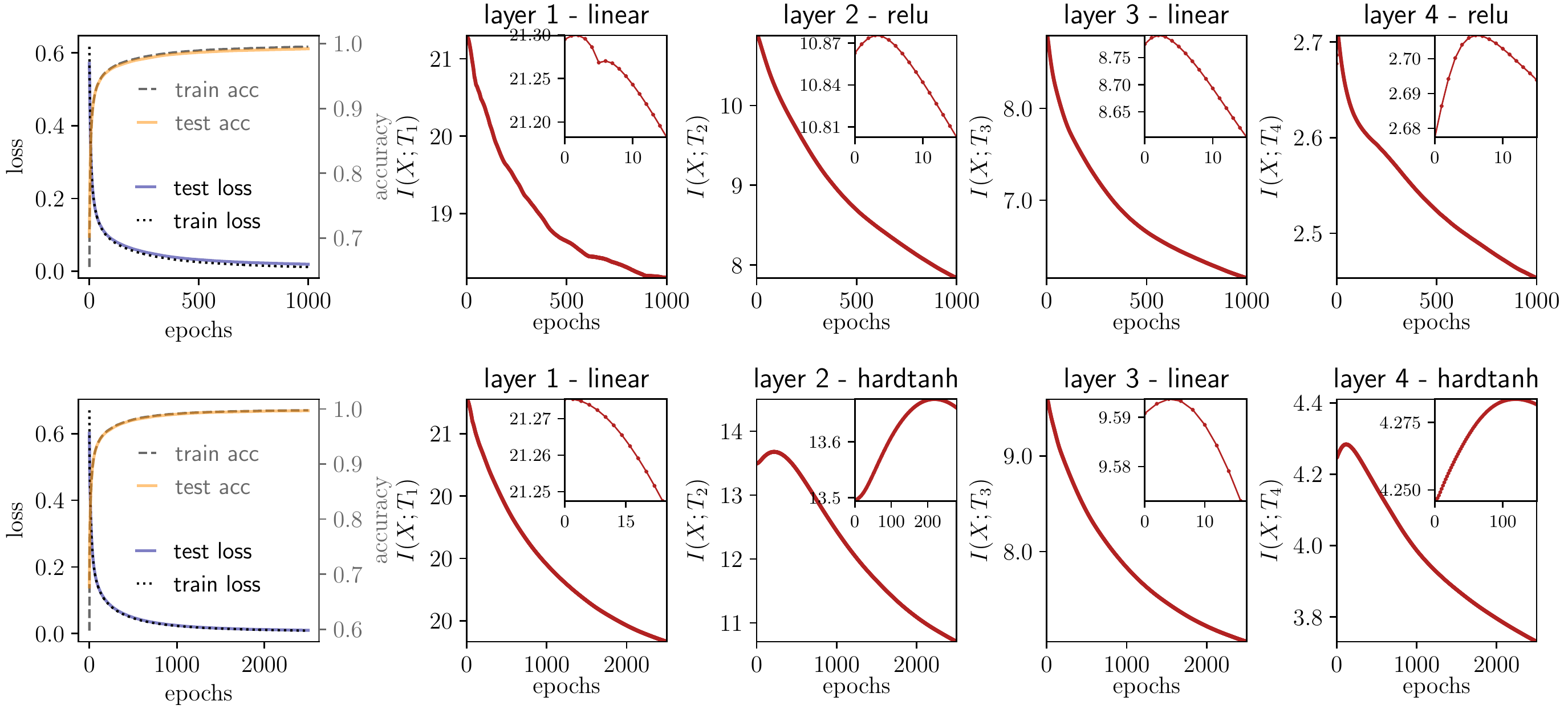}
     \caption{%
      Independent run outcome for Figure 4 of the main text. Training of two recognition models on a binary classification task with correlated input data and either ReLU (top) or hardtanh (bottom) non-linearities. Left: training and generalization cross-entropy loss (left axis) and accuracies (right axis) during learning. Best training-testing accuracies are 0.995 - 0.992 for ReLU version (top row) and 0.998 - 0.997 for hardtanh version (bottom row). Remaining colums: mutual information between the input and successive hidden layers. Insets zoom on the first epochs.
      \label{fig:classifbis}}     
    \end{center}
\end{figure}

We consider then a regression task created by a 2-layer teacher network
of sizes 500-3-3, activations ReLU-linear, uncorrelated input data
distribution $\mathcal{N}(0,I_{n_X})$ and additive white Gaussian noise at the
output of variance $0.01$. The matrices of the teacher network are i.i.d. normally distributed with a variance equal to the inverse of the layer input dimension. We train a student network with 2 ReLU layers of
sizes $2500$ and $1000$, each featuring 5 stacked USV-layers of same size
before the non linear activation, and with one final fully-connected linear
layer. We use plain SGD with a constant learning rate of 0.01 and a batchsize
of 50. In Figure \ref{fig:regression} we plot the mutual informations with the
input at the effective 10-hidden layers along the training. Except for the
very first layer where we observe a slight initial increase, all mutual
informations appear to only decrease during the learning, at least at this
resolution (i.e. after the first epoch). We thus observe a compression even in
the absence of double-saturated non-linearities. We further note that in this
case we observe an accuentuation of the amount of compression with layer depth
as observed by \cite{shwartz-ziv_opening_2017} (see second plot of first row of 
Figure \ref{fig:regression}), but which we did not observe
in the binary classification experiment presented in the main text. 
On Figure \ref{fig:regressionbis}, we reproduce the figure for a different seed.

\begin{figure}[t!]
    \begin{center}
      \includegraphics[width=\linewidth]{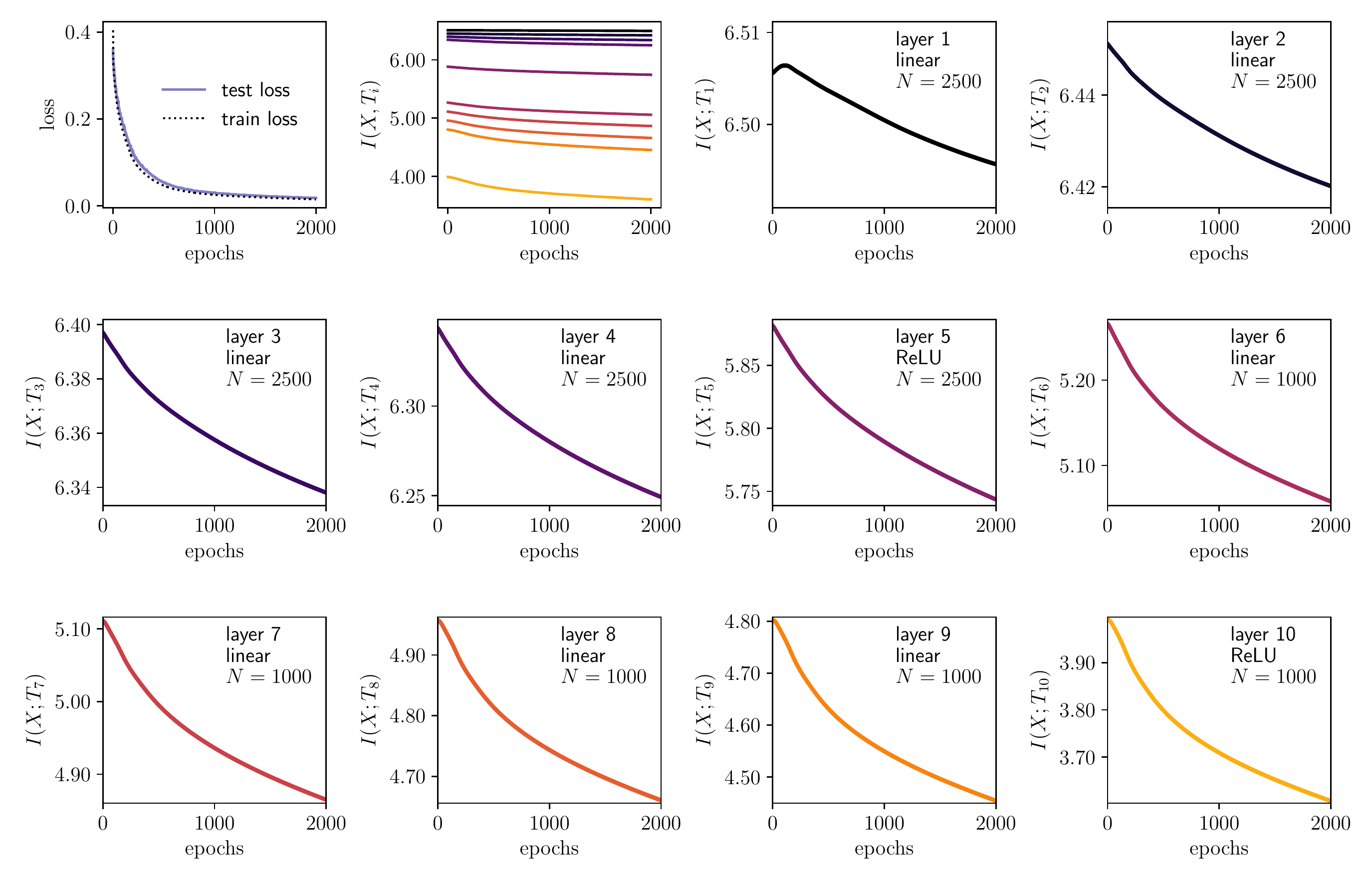}
     \caption{
        Example of regression with a 10 hidden-layer student network: 5 USV-layers - ReLU activation - 5 USV-layers - ReLu activation - 1 unconstrained final linear layer, on dataset generated by a non-linear teacher network: ReLu-linear. Top row, first plot: training and testing MSE loss along learning. Best train loss is $0.015$, best test loss is $0.018$. Top row, second plot: mutual informations curves of the 10 hidden layers showing the slight accentuation of compression in deeper layers. Remaining: mutual information from each layer displayed separately.
       \label{fig:regression}}     
    \end{center}
\end{figure}

\begin{figure}[h!]
    \begin{center}
      \includegraphics[width=\linewidth]{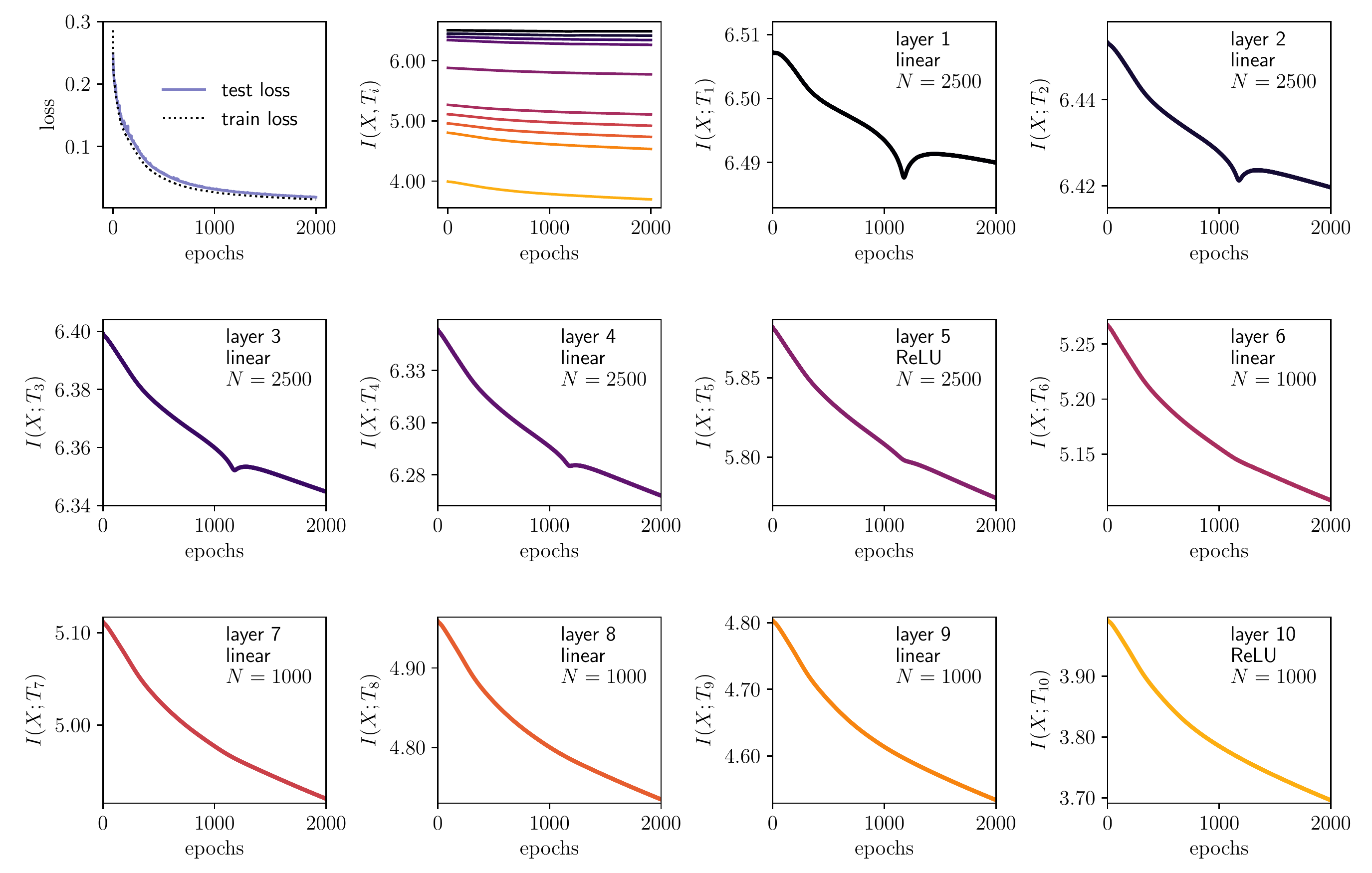}
     \caption{
        Independent run outcome for Figure \ref{fig:regression} of the Supplementary Material. Example of regression with a 10 hidden-layer student network: 5 USV-layers - ReLU activation - 5 USV-layers - ReLu activation - 1 unconstrained final linear layer, on dataset generated by a non-linear teacher network: ReLu-linear. Top row, first plot: training and testing MSE loss along learning. Best train loss is $0.015$, best test loss is $0.019$. Top row, second plot: mutual informations curves of the 10 hidden layers showing the slight accentuation of compression in deeper layers. Remaining: mutual information from each layer displayed separately.
       \label{fig:regressionbis}}     
    \end{center}
\end{figure}

In a last experiment, we even show that merely changing the weight
initialization can drastically change the behavior of mutual informations
during training while resulting in identical training and testing final
performances. We consider here a setting closely related to the classification
on correlated data presented in the main text. 
 The generative model is a
a simple single layer generative model $\X = \tilde{W}_{\rm gen} \Y +
\epsilon$ with normally distributed code $\Y\sim \mathcal{N}(0, I_{n_Y})$ of
size $n_{Y} = 100$, from which data of size $n_{X} = 500$ are generated with matrix $\tilde{W}_{\rm gen}$ i.i.d. normally distributed as $\mathcal{N}(0, 1/\sqrt{n_Y})$  and noise $\epsilon
\sim \mathcal{N}(0, 0.01 I_{n_{X}})$. The
recognition model attempts to solve the binary classification problem of
recovering the label $y = \mathrm{sign}(Y_1)$, the sign of the first neuron in
$\bm{Y}$. Again the training is done with plain SGD to minimize the cross-
entropy loss and the rest of the initial code $(Y_2, .. Y_{n_{Y}})$ acts
as noise/nuisance with respect to the learning task. On Figure
\ref{fig:coefinit} we compare 3 identical 5-layers recognition models with
sizes 500-1000-500-250-100-2, and activations hardtanh-hardtanh-hardtanh-
hartanh-softmax. For the model presented at the top row, initial weights were
sampled according to $W_{\ell,ij} \sim \mathcal{N}(0, 4/n_{\ell-1})$, for the
model of the middle row $\mathcal{N}(0, 1/n_{\ell-1})$ was used instead, and
finally $\mathcal{N}(0, \nicefrac{1}{4n_{\ell-1}})$ for the bottom row. The first column
shows that training is delayed for the weight initialized at smaller values,
but eventually catches up and reaches accuracies superior to $0.97$ both in training and testing.
Meanwhile, mutual informations have different initial values for the different
weight initializations and follow very different paths. They either decrease
during the entire learning, or on the contrary are only increasing, or actually
feature an hybrid path.  We further note that it is to some extent surprising
that the mutual information would increase at all in the first row if we
expect the hardtanh saturation to instead induce compression. Figure \ref{fig:coefinitbis} 
presents a second run of the same experiment with a different random seed. Findings are identical.

\begin{figure}[t!]
    \begin{center}
      \includegraphics[width=\linewidth]{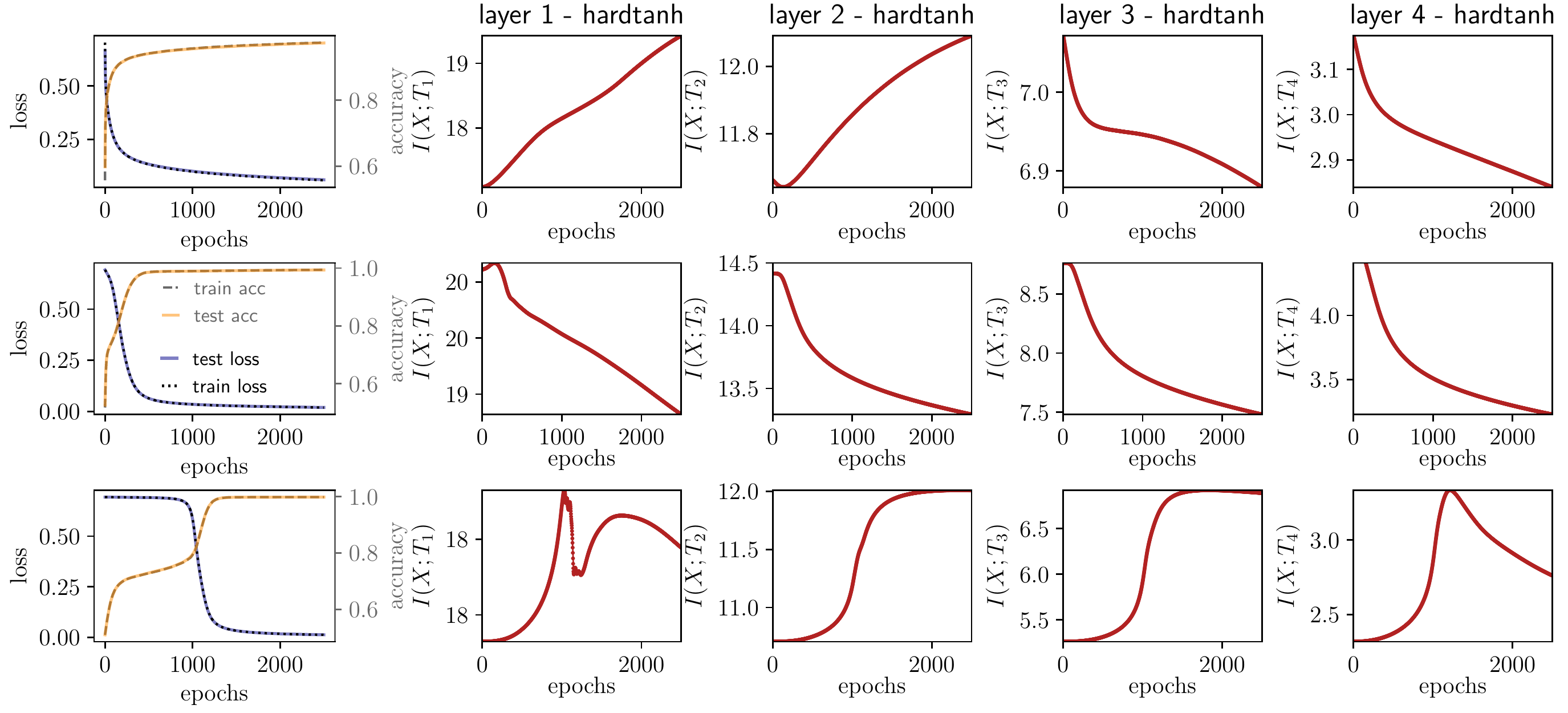}
     \caption{
        Learning and hidden-layers mutual information curves for a classification problem with correlated input data, using a 4-USV hardtanh layers and 1 unconstrained softmax layer, from 3 different initializations. Top: Initial weights at layer $\ell$ of variance $4 / n_{\ell-1}$, best training accuracy 0.999, best test accuracy 0.994. Middle:  Initial weights at layer $\ell$ of variance $1 / n_{\ell-1}$, best train accuracy 0.994, best test accuracy 0.9937. Bottom:  Initial weights at layer $\ell$ of variance $0.25 / n_{\ell-1}$, best train accuracy 0.975, best test accuracy 0.974. The overall direction of evolution of the mutual information can be flipped by a change in weight initialization without changing drastically final performance in the classification task. 
       \label{fig:coefinit}}     
    \end{center}
\end{figure}

\begin{figure}[h!]
    \begin{center}
      \includegraphics[width=\linewidth]{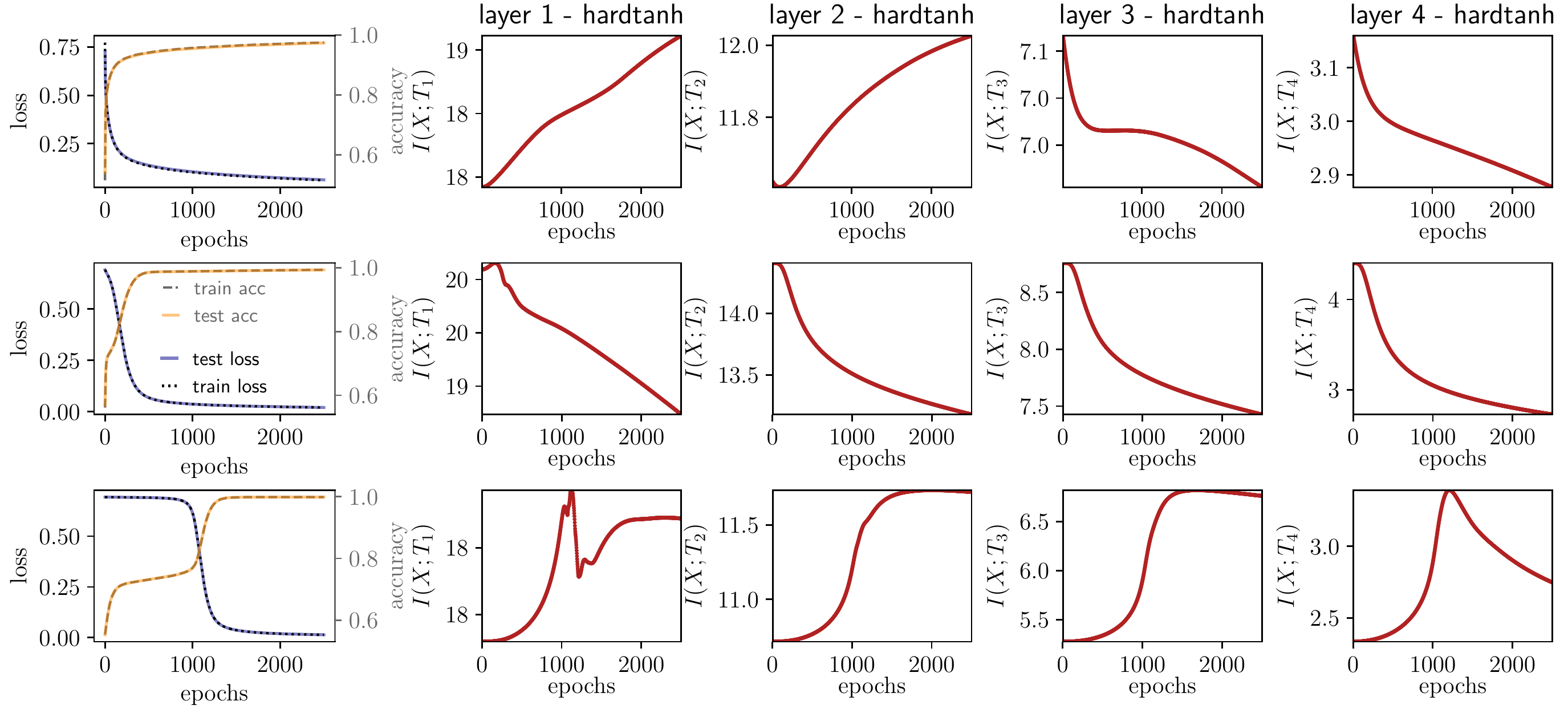}
     \caption{
        Independent run outcome for Figure \ref{fig:coefinit} of the Supplementary Material. Learning and hidden-layers mutual information curves for a classification problem with correlated input data, using a 4-USV hardtanh layers and 1 unconstrained softmax layer, from 3 different initializations. Top: Initial weights at layer $\ell$ of variance $4 / n_{\ell-1}$, best training accuracy 0.999, best test accuracy 0.998. Middle:  Initial weights at layer $\ell$ of variance $1 / n_{\ell-1}$, best train accuracy 0.9935, best test accuracy 0.9933. Bottom:  Initial weights at layer $\ell$ of variance $0.25 / n_{\ell-1}$, best train accuracy 0.974, best test accuracy 0.973. The overall direction of evolution of the mutual information can be flipped by a change in weight initialization without changing drastically final performance in the classification task. 
       \label{fig:coefinitbis}}     
    \end{center}
\end{figure}

These observed differences and non-trivial observations raise numerous
questions, and suggest that within the examined setting, a simple information theory of deep learning
remains out-of-reach.

\appendix
\section{Proofs of some technical propositions}

\subsection{The Nishimori identity}\label{app:nishimori}

\begin{proposition}[Nishimori identity] \label{prop:nishimori}
	Let $(\bX,\bY) \in \R^{n_1} \times \R^{n_2}$ be a couple of random variables. Let $k \geq 1$ and let $\bX^{(1)}, \dots, \bX^{(k)}$ be $k$ i.i.d.\ samples (given $\bY$) from the conditional distribution $P(\bX=\cdot\, | \bY)$, independently of every other random variables. Let us denote $\langle - \rangle$ the expectation operator w.r.t.\ $P(\bX= \cdot\, | \bY)$ and $\E$ the expectation w.r.t.\ $(\bX,\bY)$. Then, for all continuous bounded function $g$ we have
	\begin{align}
	\E \langle g(\bY,\bX^{(1)}, \dots, \bX^{(k)}) \rangle
	=
	\E \langle g(\bY,\bX^{(1)}, \dots, \bX^{(k-1)}, \bX) \rangle\,.	
	\end{align}
\begin{proof}
	This is a simple consequence of Bayes formula.
	It is equivalent to sample the couple $(\bX,\bY)$ according to its joint distribution or to sample first $\bY$ according to its marginal distribution and then to sample $\bX$ conditionally to $\bY$ from its conditional distribution $P(\bX=\cdot\,|\bY)$. Thus the $(k+1)$-tuple $(\bY,\bX^{(1)}, \dots,\bX^{(k)})$ is equal in law to $(\bY,\bX^{(1)},\dots,\bX^{(k-1)},\bX)$.
\end{proof}
\end{proposition}

\subsection{Limit of the sequence $(\rho_1(n_0))_{n_0 \geq 1}$}\label{app:convergenceRho}
Here we prove Proposition~\ref{prop:convergenceRho}, i.e.\ that the sequence $(\rho_1(n_0))_{n_0 \geq 1}$ converges to $\rho_1 \defeq \E[\varphi_1^2(T, \bA_1)]$, where $T \sim \cN(0,\rho_0)$ and $\bA_1 \sim P_{A_1}$ are independent,
under the hypotheses \ref{hyp:bounded} \ref{hyp:c2} \ref{hyp:phi_gauss2}.

If $\rho_0 = 0$ then $\bX^0 = 0$ almost surely (a.s.) and $\rho_1(n_0) = \E \varphi_1^2(0, \bA_1) = \rho_1$ for every $n_0 \geq 1$, making the result trivial.\newline
From now on, assume $\rho_0 > 0$. Given $\bX^0$, one has $\Big[\frac{\bW{1} \bX^0}{\sqrt{n_0}}\Big]_1 \sim \cN\Big(0,\frac{\Vert \bX^0 \Vert^2}{n_0}\Big)$. Therefore
\begin{equation*}
\rho_1(n_0) \defeq \E\bigg[ \varphi_1^2\bigg( \bigg[\frac{\bW{1} \bX^0}{\sqrt{n_0}}\bigg]_{1}, \bA_1\bigg) \!\bigg]
\!= \E \! \int \!\! dt \, dP_{A_1}(\ba)\varphi_1^2(t,\ba) \frac{\exp{-\frac{t^2}{2\frac{\Vert \bX^0 \Vert^2}{n_0}}}}{\sqrt{2\pi \frac{\Vert \bX^0 \Vert^2}{n_0}}}
\!=\E\bigg[h\left(\frac{\Vert \bX^0 \Vert^2}{n_0}\right)\!\bigg],
\end{equation*}
where $h: v \mapsto \int dt dP_{A_1}(\ba) \varphi_1^2(t, \ba) \frac{1}{\sqrt{2\pi v}}\exp(\nicefrac{-t^2}{2v})$ is a function on $]0,+\infty[$. It is easily shown to be continuous under~\ref{hyp:c2} thanks to the \textit{dominated convergence theorem}.\newline
By the Strong Law of Large Numbers, $\nicefrac{\Vert \bX^0 \Vert^2}{n_0}$ converges a.s.\ to $\rho_0$. Combined with the continuity of $h$, one has
\begin{equation*}
\lim_{n_0 \to +\infty} h\left(\frac{\Vert \bX^0 \Vert^2}{n_0}\right)\,\stackrel{\mathclap{\normalfont\mbox{a.s.}}}{=} \, h(\rho_0) = \rho_1 \:.
\end{equation*}
Noticing that $\left\vert h\left(\nicefrac{\Vert \bX^0 \Vert^2}{n_0}\right) \right\vert \leq \sup \varphi_1^2$, the \textit{dominated convergence theorem} gives
\begin{equation*}
\rho_1(n_0) = \E\bigg[h\left(\frac{\Vert \bX^0 \Vert^2}{n_0}\right)\bigg]
\xrightarrow[n_0 \to +\infty]{} \E\bigg[\lim_{n_0 \to +\infty} h\left(\frac{\Vert \bX^0 \Vert^2}{n_0}\right)\bigg] = \rho_1 \; .
\end{equation*}

\subsection{Properties of the third scalar channel}\label{app:propertiesThirdChannel}

\begin{proposition}\label{propertiesThirdChannel}
Assume $\varphi_1$ is bounded (as it is the case under~\ref{hyp:c2}). Let $V,U \iid \cN(0,1)$ and $\rho_0 \geq 0$, $q_0 \in [0,\rho_0]$.
For any $r \geq 0$, $Y_0^{\prime (r)} = \sqrt{r}\varphi_1(\sqrt{q}\, V + \sqrt{\rho - q} \,U, \bA_1) + Z'$ where $Z' \sim \cN(0,1)$, $\bA_1 \sim P_{A_1}$. The function
\begin{equation*}
	\Psi_{\varphi_1}(q_0, \, \cdot \, ;\rho_0): r \mapsto
	\E \ln \!\int \!\! {\cal D}u P_{\rm out,1}^{(r)}\big( Y_0^{\prime (r)} \big \vert \sqrt{q}\, V + \sqrt{\rho - q} \,u\big). 
\end{equation*}
is twice-differentiable, convex, non-decreasing and $\frac{\rho_1}{2}$-Lipschitz on $\R_+$. Then the function
\begin{equation*}
f: r \mapsto \sup_{q_0 \in [0,\rho_0]} \inf_{r_0 \geq 0}  \psi_{P_0}(r_0) + \alpha_1 \Psi_{\varphi_1}(q_0, r;\rho_0) - \frac{r_0 q_0}{2}  
\end{equation*}
is convex, non-decreasing and $\left(\alpha_1 \frac{\rho_1}{2}\right)$-Lipschitz on $\R_+$.
\begin{proof}
For fixed $\rho_0$ and $q_0$, let $\Psi_{\varphi_1} \equiv \Psi_{\varphi_1}(q_0, \, \cdot \, ;\rho_0)$. Note that
\begin{multline*}
	\Psi_{\varphi_1}(r) = \E\bigg[\int \!d y_0^{\prime} \frac{1}{\sqrt{2\pi}} e^{-\frac{1}{2}(y_0^{\prime} - \sqrt{r}\varphi_1(\sqrt{q_0}V + \sqrt{\rho_0-q_0}U,\bA_1))^2}\\
	\cdot \ln \int \! {\cal D}u \, dP_{A_1}(\ba)
	e^{\sqrt{r}y_0^{\prime} \varphi_1(\sqrt{q_0}V + \sqrt{\rho_0 - q_0} u,\ba)-\frac{r}{2}\varphi_1^2(\sqrt{q_0}V + \sqrt{\rho_0 - q_0}u,\ba)}\bigg] \,.
\end{multline*}
With the properties imposed on $\varphi_1$, all the domination hypotheses to prove the twice-differentiability of $\psi_{\varphi_1}$ are reunited. Denote $\langle - \rangle_r$ the expectation operator w.r.t.\ the joint posterior distribution
\begin{equation*}
dP(u, \ba \vert Y_0^{\prime}, V) = \frac{1}{\cZ(Y_0^{\prime}, V)}{\cal D}u \, dP_{A_1}(\ba) e^{\sqrt{r}y_0^{\prime} \varphi_1(\sqrt{q_0}V + \sqrt{\rho_0 - q_0} u,\ba)-\frac{r}{2}\varphi_1^2(\sqrt{q_0}V + \sqrt{\rho_0 - q_0}u,\ba)} \,,
\end{equation*}
where $\cZ(Y_0^{\prime}, V)$ is a normalization factor. Using Gaussian integration by parts and the Nishimori property (Proposition \ref{prop:nishimori}), one verifies that for all $r \geq 0$
\begin{align*}
\Psi_{\varphi_1}^{\prime}(r)&= \frac{1}{2}\E \big[ \langle \varphi_1(\sqrt{q_0}V + \sqrt{\rho_0 - q_0} u, \ba) \rangle_r^2\big]\, \geq 0\,,\\
\Psi_{\varphi_1}^{''}(r)&= \frac{1}{2}\E\big[ \big(\langle \varphi_1^2(\sqrt{q_0}V + \sqrt{\rho_0 - q_0} u,\ba) \rangle_r - \langle \varphi_1(\sqrt{q_0}V + \sqrt{\rho_0 - q_0} u,\ba) \rangle_r^2\big)^2 \big] \geq 0\,.
\end{align*}
Hence $\Psi_{\varphi_1}$ is non-decreasing and convex. The Lipschitzianity follows simply from
\begin{equation*}
\big\vert \Psi_{\varphi_1}^{\prime}(r) \big\vert
\leq \frac{1}{2} \big\vert \E \langle \varphi_1^2(\sqrt{q_0}V + \sqrt{\rho_0 - q_0} u, \ba) \rangle_r \big\vert
= \frac{1}{2} \big\vert \E[\varphi_1^2(\sqrt{q_0}V + \sqrt{\rho_0 - q_0} U, \bA_1)]\big\vert
= \frac{1}{2} \rho_1 \,.
\end{equation*}
The Nishimori identity was used once again to obtain the penultimate equality. Finally, $f$ properties are direct consequences of its definition as the ``$\sup \inf$'' of convex, non-decreasing, $\frac{\rho_1}{2}$-lipschitzian functions.
\end{proof}
\end{proposition}
\section{Derivative of the averaged interpolating free entropy}\label{appendix_interpolation}
This appendix is dedicated to the proof of Proposition~\ref{prop:der_f_t}, i.e.\ to the derivation of the interpolating free entropy $f_{\mathbf{n},\epsilon}(t)$ with respect to the \textit{time} $t$.
First we show that for all $t \in (0,1)$
\begin{multline}\label{eq:der_f_t_raw}	
		\frac{df_{\mathbf{n},\epsilon}(t)}{dt} = 
		-\frac{1}{2} \frac{n_1}{n_0}\E\Bigg[\Bigg\langle
		\Bigg(\frac{1}{n_1}\sum_{\mu=1}^{n_2} u_{Y_{t,\epsilon,\mu}}'( S_{t,\epsilon,\mu} )u_{Y_{t,\epsilon,\mu}}'(s_{t,\epsilon,\mu}) - r_{\epsilon}(t)\Bigg)
		\Big( \widehat{Q}- q_{\epsilon}(t)\Big)
		\Bigg\rangle_{\! \mathbf{n},t,\epsilon}\,\Bigg]\\
		+ \frac{n_1}{n_0} \frac{r_{\epsilon}(t)}{2} \big(q_{\epsilon}(t)-\rho_1(n_0)\big)
		-\frac{A_{\mathbf{n},\epsilon}(t)}{2}\,,
\end{multline}
where
\begin{align}\label{An}
A_{\mathbf{n},\epsilon}(t) \defeq
		\E\Bigg[\sum_{\mu=1}^{n_2}
\frac{P_{\rm out,2}''(Y_{t,\epsilon,\mu} | S_{t,\epsilon,\mu})}{P_{\rm out,2}(Y_{t,\epsilon,\mu} | S_{t,\epsilon,\mu})} 
\Bigg(\frac{\big\Vert\bX^1\big\Vert^2}{n_1} - \rho_1(n_0)
\Bigg)
\frac{\ln \cZ_{\mathbf{n},t,\epsilon}}{n_0} \Bigg]\,.
\end{align}
$P_{\rm out,2}'(y|x)$ and $P_{\rm out,2}''(y|x)$ denote the first and second $x$-derivatives.
Once this is done, we prove that $A_{\mathbf{n},\epsilon}(t)$ goes to $0$ uniformly in $t \in [0,1]$ as $n_0,n_1,n_2 \to +\infty$ (while $\nicefrac{n_1}{n_0} \to \alpha_1$, $\nicefrac{n_2}{n_1} \to \alpha_2$), thus proving Proposition~\ref{prop:der_f_t}.
\subsection{Computing the derivative: proof of\eqref{eq:der_f_t_raw}} \label{app:computationDerivative}
Recall definition \eqref{ft}. Once written as a function of the interpolating Hamiltonian \eqref{interpolating-ham}, it becomes
	\begin{multline}
	f_{\mathbf{n},\epsilon}(t) \!= \!\frac{1}{n_0} \E_{\bW{1},\bW{2},\bV}\bigg[\int \!\! d\bY d\bY' dP_0(\bX^0)dP_{A_1}(\bA_1) {\cal D}\bU  (2\pi)^{-\frac{n_1}{2}} e^{-\cH_{t,\epsilon}(\bX^0,\bA_1,\bU;\bY,\bY',\bW{1},\bW{2},\bV)}\\
	\cdot \ln \int dP_0(\bx) dP_{A_1}(\ba_1){\cal D}\bu\, e^{-\cH_{t,\epsilon}(\bx,\ba_1,\bu;\bY,\bY',\bW{1},\bW{2},\bV)}\bigg].
\end{multline}
Here, and from now on, one drops the dependence on $(t,\epsilon)$ when writing $\bY$ and $\bY'$ as they are dummy variables on which the integration is performed.
We will need the Hamiltonian $t$-derivative ${\cal H}'_t$ given by
\begin{multline}\label{tDerivativeHamiltonian}
\cH_{t,\epsilon}'(\bX^0, \bA_1, \bU;\bY,\bY',\bW{1},\bW{2},\bV)
	= - \sum_{\mu=1}^{n_2}\frac{dS_{t,\epsilon,\mu}}{dt} u'_{Y_\mu}( S_{t,\epsilon,\mu})\\
	- \frac{1}{2} \frac{r_{\epsilon}(t)}{\sqrt{R_1(t,\epsilon)}} \sum_{i=1}^{n_1}
			\varphi_1\bigg( \bigg[\frac{\bW{1} \bX^0}{\sqrt{n_0}}\bigg]_{i}, \bA_{1,i}\bigg)
			\bigg(Y'_i  - \sqrt{R_1(t,\epsilon)} \varphi_1\bigg(  \bigg[\frac{\bW{1} \bX^0}{\sqrt{n_0}}\bigg]_{i}, \bA_{1,i}\bigg)\bigg)\,.
\end{multline}
The derivative of the interpolating free entropy for $0 < t < 1$ thus reads
\begin{multline}
\frac{df_{\mathbf{n},\epsilon}(t)}{dt} 
= -\underbrace{\frac{1}{n_0}
	\E \big[\cH_{t,\epsilon}'(\bX^0,\bA_1,\bU;\bY,\bY',\bW{1},\bW{2},\bV)\ln \cZ_{\mathbf{n},t,\epsilon}\big]}_{T_1}\\
	- \underbrace{\frac{1}{n_0} \E \Big[\Big\langle \cH_{t,\epsilon}'(\bx,\ba_1,\bu;\bY,\bY',\bW{1},\bW{2},\bV) \Big\rangle_{\! \mathbf{n},t,\epsilon} \,\Big] }_{T_2}\,,
		\label{formulaWithT1T2}
\end{multline}
where $\cZ_{\mathbf{n},t,\epsilon} \equiv \cZ_{\mathbf{n},t,\epsilon}(\bY,\bY',\bW{1}, \bW{2},\bV)$ is defined in \eqref{Zt}.
In the remaining part of this subsection~\ref{app:computationDerivative}, to lighten notations,
the second argument of the function $\varphi_1$ will be omitted (except in a few occasions). Namely, we will write for $i=1 \dots n_1$
\begin{equation*}
\varphi_1\Big(  \Big[\frac{\bW{1} \bX^0}{\sqrt{n_0}}\Big]_{i}\Big) \equiv \varphi_1\Big(  \Big[\frac{\bW{1} \bX^0}{\sqrt{n_0}}\Big]_{i}, \bA_{1,i}\Big) \, , 
\; \varphi_1\Big(  \Big[\frac{\bW{1} \bx}{\sqrt{n_0}}\Big]_{i}\Big) \equiv \varphi_1\Big(  \Big[\frac{\bW{1} \bx}{\sqrt{n_0}}\Big]_{i}, \ba_{1,i}\Big) \,.
\end{equation*}
It does not hurt the understanding of the derivation of \eqref{eq:der_f_t_raw} as the latter relies on integration by parts w.r.t.\ the Gaussian random variables $\bW{1}$, $\bW{2}$, $\bV$, $\bU$, $\bZ'$.\\

Let first compute $T_1$. For $1 \leq \mu \leq n_2$ one has from \eqref{defS_t_mu}
\begin{multline}\label{firstTermT1}
-\E \bigg[\frac{dS_{t,\epsilon,\mu}}{dt} u'_{Y_\mu}(S_{t,\epsilon,\mu})  \ln \cZ_{\mathbf{n},t,\epsilon}  \bigg]
	=\frac{1}{2}\E \bigg[
			\frac{1}{\sqrt{n_1 (1-t)}}\bigg[\bW{2} \varphi_1\bigg(\frac{\bW{1} \bX^0}{\sqrt{n_0}}\bigg)\bigg]_{\mu} u'_{Y_\mu}(S_{t,\epsilon,\mu}) \ln \cZ_{\mathbf{n},t,\epsilon}
	\bigg]\\
	- \frac{1}{2}\E\Bigg[\Bigg(\frac{q_{\epsilon}(t)}{ \sqrt{R_2(t,\epsilon)}} V_{\mu}
	+ \frac{\rho_1(n_0) - q_{\epsilon}(t)}{\sqrt{\rho_1(n_0)t + 2s_n - R_2(t,\epsilon)}} U_{\mu}\Bigg) u'_{Y_\mu}(S_{t,\epsilon,\mu}) \ln \cZ_{\mathbf{n},t,\epsilon}  \Bigg]\,.
\end{multline}
By Gaussian integration by parts w.r.t $(\bW{2})_{\mu i}$, $1 \leq i \leq n_1$, the first expectation becomes
\begin{align}
&\frac{1}{\sqrt{n_1(1-t)}}\E\bigg[
			\bigg[\bW{2} \varphi_1\bigg(\frac{\bW{1} \bX^0}{\sqrt{n_0}}\bigg)\bigg]_{\mu}
			u_{Y_{\mu}}' ( S_{t,\epsilon,\mu} )
			\ln \cZ_{\mathbf{n},t,\epsilon}
		\bigg]&\nn
&\quad =\frac{1}{\sqrt{n_1 (1-t)}}\sum_{i=1}^{n_1}\E \bigg[\int d\bY d\bY' (2\pi)^{-\frac{n_1}{2}}
	e^{-\cH_{t,\epsilon}(\bX^0,\bA_1,\bU;\bY,\bY',\bW{1},\bW{2},\bV)}\nn
&\qquad\qquad\qquad\qquad\qquad\qquad \cdot (\bW{2})_{\mu i} \varphi_1\bigg(\bigg[\frac{\bW{1} \bX^0}{\sqrt{n_0}}\bigg]_i\bigg)
			u_{Y_{\mu}}' (S_{t,\epsilon,\mu}) \ln \cZ_{\mathbf{n},t,\epsilon} \bigg]\nn
&\quad = \frac{1}{n_1} \sum_{i=1}^{n_1}
			\E\bigg[
				\varphi_1^2\bigg(\bigg[\frac{\bW{1} \bX^0}{\sqrt{n_0}}\bigg]_i\bigg)
				\Big( u_{Y_{\mu}}''(S_{t,\epsilon,\mu}) + u_{Y_{\mu}}' (S_{t,\epsilon,\mu})^2 \Big)
				\ln \cZ_{\mathbf{n},t,\epsilon}
			\bigg]\nn
&\qquad\qquad\qquad\qquad + \frac{1}{n_1} \sum_{i=1}^{n_1} \E\Bigg[\bigg\langle
				\varphi_1\bigg(\bigg[\frac{\bW{1} \bX^0}{\sqrt{n_0}}\bigg]_i\bigg)
				\varphi_1\bigg(\bigg[\frac{\bW{1} \bx}{\sqrt{n_0}}\bigg]_i\bigg)
				u_{Y_{\mu}}' (S_{t,\epsilon,\mu}) u_{Y_{\mu}}' (s_{t,\epsilon,\mu}) \bigg\rangle_{\!\! \mathbf{n},t,\epsilon}\,\Bigg]\nonumber\\
&\quad=\E\Bigg[\frac{\big\Vert\bX^1\big\Vert^2}{n_1}
			\frac{P_{\rm out,2}''(Y_{\mu} | S_{t,\epsilon,\mu})}{P_{\rm out,2}(Y_{\mu} | S_{t,\epsilon,\mu})}
			\ln \cZ_{\mathbf{n},t,\epsilon} \Bigg]
			+ \E\Big[\big\langle \widehat{Q} \:
			u_{Y_{\mu}}' ( S_{t,\epsilon,\mu} )
			u_{Y_{\mu}}' ( s_{t,\epsilon,\mu} )
			\big\rangle_{\! \mathbf{n},t,\epsilon}\Big].\label{eq:compA1}
\end{align}
In the last equality we used the identity $u_{Y_\mu}'' ( x ) + u_{Y_\mu}' ( x )^2 = \frac{P_{\rm out,2}''(Y_{\mu} | x)}{P_{\rm out,2}(Y_{\mu} | x)}$ and the definition of the overlap $\widehat{Q}$.
Now we turn our attention to the second expectation in the right hand side of \eqref{firstTermT1}. Using again Gaussian integration by parts, but this time w.r.t $V_\mu, U_\mu \iid {\cal N}(0,1)$, one similarly obtains
\begin{align}
&\E \Bigg[
			\Bigg( \frac{q_{\epsilon}(t)}{ \sqrt{R_2(t,\epsilon)}} V_{\mu}
			+ \frac{\rho_1(n_0) - q_{\epsilon}(t)}{\sqrt{\rho_1(n_0)t + 2s_n - R_2(t,\epsilon)}} U_{\mu}\Bigg)
			u_{Y_{\mu}}' ( S_{t,\epsilon,\mu} )
			\ln \cZ_{\mathbf{n},t,\epsilon} \Bigg]\nn
&\qquad\qquad= \E \Bigg[\int d\bY d\bY' (2\pi)^{-\frac{n_1}{2}} e^{-\cH_{t,\epsilon}(\bX^0,\bU;\bY,\bY',\bW{1},\bW{2},\bV)}\nn
&\qquad\qquad\qquad\qquad\cdot \Bigg( \frac{q_{\epsilon}(t)}{ \sqrt{R_2(t,\epsilon)}} V_{\mu}
+ \frac{\rho_1(n_0) - q_{\epsilon}(t)}{\sqrt{\rho_1(n_0)t + 2s_n - R_2(t,\epsilon)}} U_{\mu}\Bigg)
			u_{Y_{\mu}}' ( S_{t,\epsilon,\mu} )
		\ln \cZ_{\mathbf{n},t,\epsilon}  \Bigg]
		 \nn
&\qquad\qquad = \E\Bigg[
			\rho_1(n_0) \frac{P_{\rm out,2}''(Y_\mu | S_{t,\epsilon,\mu})}{P_{\rm out,2}(Y_\mu | S_{t,\epsilon,\mu})} 
			\ln \cZ_{\mathbf{n},t,\epsilon} \Bigg]
		+\E \Big\langle q_{\epsilon}(t) u'_{Y_\mu}(S_{t,\epsilon,\mu}) u'_{Y_\mu}(s_{t,\epsilon,\mu})\Big\rangle_{\! \mathbf{n},t,\epsilon}\,.\label{eq:compA2}
\end{align}
Combining equations \eqref{firstTermT1}, \eqref{eq:compA1} and \eqref{eq:compA2} together gives us
\begin{multline}\label{finalFormulaFirstTermT1}
-\E \bigg[\frac{dS_{t,\epsilon,\mu}}{dt} u'_{Y_\mu}(S_{t,\epsilon,\mu}) \ln \cZ_{\mathbf{n},t,\epsilon} \bigg]
	 = \frac{1}{2}\E\Bigg[
			\frac{P_{\rm out,2}''(Y_\mu | S_{t,\epsilon,\mu})}{P_{\rm out,2}(Y_\mu | S_{t,\epsilon,\mu})} 
			\Bigg(
				\frac{\big\Vert\bX^1\big\Vert^2}{n_1} - \rho_1(n_0)
			\Bigg)
			\ln \cZ_{\mathbf{n},t,\epsilon} \Bigg]\\
+\frac{1}{2} \E\Big[\Big\langle
			\!\big( \widehat{Q}- q_{\epsilon}(t)\big)
			u_{Y_{\mu}}'( S_{t,\epsilon,\mu} )
			u_{Y_{\mu}}'( s_{t,\epsilon,\mu} )\Big\rangle_{\! \mathbf{n},t,\epsilon}\,\Big] \,.
\end{multline}
It remains to put the term $\E\big[\varphi_1\big(\big[\nicefrac{\bW{1} \bX^0}{\sqrt{n_0}}\big]_{i}\Big) \big(Y'_i  - \sqrt{R_1(t,\epsilon)} \varphi_1\big(\big[\nicefrac{\bW{1} \bX^0}{\sqrt{n_0}}\big]_{i}\big)\big)\ln \cZ_{\mathbf{n},t,\epsilon}\big]$ in a nicer form, as seen from \eqref{tDerivativeHamiltonian} and \eqref{formulaWithT1T2}.
Once again it is achieved by an integration by parts, this time w.r.t.\ the standard Gaussian random variable $Z_i' = Y_i'-\sqrt{R_1(t,\epsilon)}\varphi_1\big(\big[\nicefrac{\bW{1} \bX^0}{\sqrt{n_0}}\big]_i\big)$. It comes
\begin{flalign}
&\E\bigg[\varphi_1\bigg( \bigg[\frac{\bW{1} \bX^0}{\sqrt{n_0}}\bigg]_{i}\bigg)
	\bigg(Y'_i  - \sqrt{R_1(t,\epsilon)} \varphi_1\bigg( \bigg[\frac{\bW{1} \bX^0}{\sqrt{n_0}}\bigg]_{i}\bigg)\bigg)
	\ln \cZ_{\mathbf{n},t,\epsilon} \bigg]&\nn
&\qquad\qquad=\E\bigg[\varphi_1\bigg( \bigg[\frac{\bW{1} \bX^0}{\sqrt{n_0}}\bigg]_{i}\bigg) Z'_i \ln \cZ_{\mathbf{n},t,\epsilon} \bigg]\nn
&\qquad\qquad=\E \left[\varphi_1\bigg( \bigg[\frac{\bW{1} \bX^0}{\sqrt{n_0}}\bigg]_{i}\bigg) 
		Z_i' \ln{\int dP_0(\bx) dP_{A_1}(\ba_1){\cal D}\bu\, e^{-\cH_{t,\epsilon}(\bx,\ba_1,\bu;\bY,\bY',\bW{1},\bW{2},\bV)}}\right]\nn
&\qquad\qquad = -\E \Bigg[\varphi_1\bigg( \bigg[\frac{\bW{1} \bX^0}{\sqrt{n_0}}\bigg]_{i}\bigg)
	\bigg\langle \sqrt{R_1(t,\epsilon)}\bigg(
	\varphi_1\bigg( \bigg[\frac{\bW{1} \bX^0}{\sqrt{n_0}}\bigg]_{i}\bigg)
	- \varphi_1\bigg( \bigg[\frac{\bW{1} \bx}{\sqrt{n_0}}\bigg]_{i}\bigg)\bigg) +Z_i' \bigg\rangle_{\! \mathbf{n},t,\epsilon} \,\Bigg]\nn
&\qquad\qquad = -\sqrt{R_1(t,\epsilon)} \bigg(\rho_1(n_0) - \E \bigg\langle
	\varphi_1\bigg( \bigg[\frac{\bW{1} \bX^0}{\sqrt{n_0}}\bigg]_{i}\bigg)
	\varphi_1\bigg( \bigg[\frac{\bW{1} \bx}{\sqrt{n_0}}\bigg]_{i}\bigg) \bigg\rangle_{\! \mathbf{n},t,\epsilon} \,\bigg) \,.
\end{flalign}
After taking the sum over $i \in \{1,\dots,n_1\}$, we get
\begin{multline}
-\frac{1}{2}\frac{r_{\epsilon}(t)}{\sqrt{R_1(t,\epsilon)}}
\E\Bigg[ \frac{1}{n_0} \sum_{i=1}^{n_1}
		\varphi_1\bigg( \bigg[\frac{\bW{1} \bX^0}{\sqrt{n_0}}\bigg]_{i}\bigg)
		\bigg(Y'_i  - \sqrt{R_1(t,\epsilon)} \varphi_1\bigg( \bigg[\frac{\bW{1} \bX^0}{\sqrt{n_0}}\bigg]_{i}\bigg)\bigg)
		\ln \cZ_{\! \mathbf{n},t,\epsilon} \Bigg]\\
	= \frac{n_1}{n_0} \frac{r_{\epsilon}(t)}{2} \Big(\rho_1(n_0)-
		\E\Big[\big\langle \widehat{Q}\big\rangle_{\! \mathbf{n},t,\epsilon}\,\Big]\Big)
		\,.
\end{multline}
Therefore, for all $t \in (0,1)$,
\begin{multline}
T_1 = 
\frac{1}{2}\E\Bigg[\sum_{\mu=1}^{n_2}
\frac{P_{\rm out,2}''(Y_\mu | S_{t,\epsilon,\mu})}{P_{\rm out,2}(Y_\mu | S_{t,\epsilon,\mu})} 
\Bigg(\frac{\big\Vert\bX^1\big\Vert^2}{n_1} - \rho_1(n_0)
\Bigg)
\frac{\ln \cZ_{\mathbf{n},t,\epsilon}}{n_0} \Bigg]\\
+\frac{1}{2} \frac{n_1}{n_0}\E\Bigg[\Bigg\langle
\Bigg(\frac{1}{n_1}\sum_{\mu=1}^{n_2} u_{Y_{\mu}}'( S_{t,\epsilon,\mu} )u_{Y_{\mu}}'(s_{t,\epsilon,\mu}) - r_{\epsilon}(t)\Bigg)
\Big( \widehat{Q}- q_{\epsilon}(t)\Big)
\Bigg\rangle_{\! \mathbf{n},t,\epsilon}\,\Bigg]\\
+ \frac{n_1}{n_0} \frac{r_{\epsilon}(t)}{2} \big(\rho_1(n_0)- q_{\epsilon}(t)\big).
\end{multline}	

To obtain \eqref{eq:der_f_t_raw}, we have to show that $T_2$ is zero. The Nishimori identity (see Proposition~\ref{prop:nishimori}) says
\begin{multline} \label{NishimoriT2}
T_2 = \frac{1}{n_0} \E\Big[\big\langle \cH_{t,\epsilon}'(\bx,\ba_1,\bu;\bY,\bY',\bW{1},\bW{2},\bV) \big\rangle_{\! \mathbf{n},t,\epsilon}\Big]\\
		= \frac{1}{n_0} \E\big[\cH_{t,\epsilon}'(\bX^0,\bA_1,\bU;\bY,\bY',\bW{1},\bW{2},\bV)\big]. 
\end{multline}
From \eqref{tDerivativeHamiltonian} it directly comes
\begin{flalign}
&\E\big[\cH_{t,\epsilon}'(\bX^0, \bA_1, \bU;\bY,\bY',\bW{1},\bW{2},\bV)\big]& \nn
&\qquad\qquad\qquad=- \sum_{\mu=1}^{n_2}
	\E\bigg[\frac{dS_{t,\epsilon,\mu}}{dt} u'_{Y_\mu}(S_{t,\epsilon,\mu})\bigg]
	- \frac{1}{2} \frac{r_{\epsilon}(t)}{\sqrt{R_1(t,\epsilon)}} \sum_{i=1}^{n_1}
	\E\bigg[\varphi_1\bigg( \bigg[\frac{\bW{1} \bX^0}{\sqrt{n_0}}\bigg]_{i}, \bA_{1,i}\bigg) Z_i^{\prime}\bigg]\nn
&\qquad\qquad\qquad= - \sum_{\mu=1}^{n_2} \E\bigg[\frac{dS_{t,\epsilon,\mu}}{dt} u'_{Y_\mu}(S_{t,\epsilon,\mu})\bigg] \,.\label{rewriting_T2}
\end{flalign}
Performing the same integration by parts than the ones leading to \eqref{finalFormulaFirstTermT1}, we obtain
\begin{equation}\label{T2isZero}
- \sum_{\mu=1}^{n_2} \E\bigg[\frac{dS_{t,\epsilon,\mu}}{dt} u'_{Y_\mu}( S_{t,\epsilon,\mu})\bigg]
	=\frac{1}{2}\E\Bigg[
		\sum_{\mu=1}^{n_2} \frac{P_{\rm out,2}''(Y_\mu | S_{t,\epsilon,\mu})}{P_{\rm out,2}(Y_\mu | S_{t,\epsilon,\mu})} 
		\Bigg(
		\frac{\big\Vert \bX^1 \big\Vert^2}{n_1} - \rho_1(n_0)
		\Bigg)\Bigg]
	= 0\,.
\end{equation}
The last equality follows from a computation in the next section, see \eqref{towerPropertyGivesZero}.
The combination of ~\eqref{NishimoriT2},~\eqref{rewriting_T2} and~\eqref{T2isZero} shows that $T_2 = 0$.
	
\subsection{Proof that $A_{\mathbf{n},\epsilon}(t)$ vanishes uniformly as $n_0 \to +\infty$}\label{app:uniformVanishingA}
The last step to prove Proposition~\ref{prop:der_f_t} is to show that $A_{\mathbf{n},\epsilon}(t)$ -- see definition \eqref{An} -- vanishes uniformly in ${t \in [0,1]}$ and $\epsilon$ as $n_0 \to +\infty$, under conditions~\ref{hyp:bounded}-\ref{hyp:c2}-\ref{hyp:phi_gauss2}. First we show that
\begin{align}		
		f_{\mathbf{n},\epsilon}(t) \cdot \E\Bigg[
			\sum_{\mu=1}^{n_2} \frac{P_{\rm out,2}''(Y_{t,\epsilon,\mu} | S_{t,\epsilon,\mu})}{P_{\rm out,2}(Y_{t,\epsilon,\mu} | S_{t,\epsilon,\mu})} 
			\Bigg(\frac{\big\Vert\bX^1\big\Vert^2}{n_1} - \rho_1(n_0)\Bigg)			
		\Bigg] = 0\,. \label{115}
\end{align}		
Once this is done, we use the fact that $\nicefrac{\ln \cZ_{\mathbf{n},t,\epsilon}}{n_0}$ concentrates around $f_{\mathbf{n},\epsilon}(t)$ to prove that $A_{\mathbf{n},\epsilon}(t)$ vanishes uniformly.

Start by noticing the simple fact that $\int P_{\rm out,2}''(y|s)dy = 0$ for all $s \in \R$. Consequently, for ${\mu \in \{1 ,\dots, n_2 \}}$,
\begin{align}
	\E \bigg[
		\frac{P_{\rm out,2}''(Y_{t,\epsilon,\mu} | S_{t,\epsilon,\mu})}{P_{\rm out,2}(Y_{t,\epsilon,\mu} | S_{t,\epsilon,\mu})}		
		\, \bigg| \, \bX^1, \bS_{t,\epsilon} \bigg]
		= \int dY_\mu P_{\rm out,2}''(Y_\mu|S_{t,\epsilon,\mu}) = 0\,. \label{117}
\end{align}
The ``tower property'' of the conditional expectation then gives
\begin{multline}
\E\Bigg[
	\sum_{\mu=1}^{n_2} \frac{P_{\rm out,2}''(Y_{t,\epsilon,\mu} | S_{t,\epsilon,\mu})}{P_{\rm out,2}(Y_{t,\epsilon,\mu} | S_{t,\epsilon,\mu})} 
	\Bigg(\frac{\big\Vert\bX^1\big\Vert^2}{n_1} - \rho_1(n_0)\Bigg)\Bigg]\\
	= \E\Bigg[
	\E \Bigg[ \sum_{\mu=1}^{n_2}\frac{P_{\rm out,2}''(Y_{t,\epsilon,\mu} | S_{t,\epsilon,\mu})}{P_{\rm out,2}(Y_{t,\epsilon,\mu} | S_{t,\epsilon,\mu})}		
	\, \Bigg| \, \bX^1, \bS_{t,\epsilon} \Bigg]
	\Bigg(\frac{\big\Vert\bX^1\big\Vert^2}{n_1} - \rho_1(n_0)\Bigg)\Bigg]
 	= 0 \,. \label{towerPropertyGivesZero}
\end{multline}		
This implies \eqref{115}. Using successively \eqref{115} and the Cauchy-Schwarz inequality, we have
\begin{align}
\big\vert A_{\mathbf{n},\epsilon}(t) \big\vert 
&= \Bigg\vert \E\Bigg[
	\sum_{\mu=1}^{n_2} \frac{P_{\rm out,2}''(Y_{t,\epsilon,\mu} | S_{t,\epsilon,\mu})}{P_{\rm out,2}(Y_{t,\epsilon,\mu} | S_{t,\epsilon,\mu})} 
	\Bigg(\frac{\big\Vert\bX^1\big\Vert^2}{n_1} - \rho_1(n_0)\Bigg)
	\bigg(\frac{\ln \cZ_{\mathbf{n},t,\epsilon}}{n_0}
	-f_{\mathbf{n},\epsilon}(t)\bigg)\Bigg]\Bigg\vert\nn
&\leq \E\Bigg[
			\Bigg(\sum_{\mu=1}^{n_2} \frac{P_{\rm out,2}''(Y_{t,\epsilon,\mu} | S_{t,\epsilon,\mu})}{P_{\rm out,2}(Y_{t,\epsilon,\mu} | S_{t,\epsilon,\mu})}\Bigg)^{\!\! 2}
			\Bigg(\frac{\big\Vert\bX^1\big\Vert^2}{n_1} - \rho_1(n_0)\Bigg)^{\!\! 2}\,
		\Bigg]^{\frac{1}{2}}\!\!
		\cdot\,\E\Bigg[\bigg(\frac{\ln \cZ_{\mathbf{n},t,\epsilon}}{n_0}
-f_{\mathbf{n},\epsilon}(t)\bigg)^{\!\! 2} \,\Bigg]^{\frac{1}{2}}.\label{boundWithCS}
\end{align}
Making use of the ``tower property'' of conditional expectation once more, one obtains
\begin{multline}
\E\Bigg[
\Bigg(\sum_{\mu=1}^{n_2} \frac{P_{\rm out,2}''(Y_{t,\epsilon,\mu} | S_{t,\epsilon,\mu})}{P_{\rm out,2}(Y_{t,\epsilon,\mu} | S_{t,\epsilon,\mu})}\Bigg)^{\!\! 2}
\Bigg(\frac{\big\Vert\bX^1\big\Vert^2}{n_1} - \rho_1(n_0)\Bigg)^{\!\! 2}\,
\Bigg]\\
= \E \Bigg[
\E \Bigg[
\Bigg(\sum_{\mu=1}^{n_2} \frac{P_{\rm out,2}''(Y_{t,\epsilon,\mu} | S_{t,\epsilon,\mu})}{P_{\rm out,2}(Y_{t,\epsilon,\mu} | S_{t,\epsilon,\mu})}\Bigg)^{\!\! 2}
\, \Bigg| \, \bX^1, \bS_{t,\epsilon}
\Bigg] \cdot \Bigg(\frac{\big\Vert\bX^1\big\Vert^2}{n_1} - \rho_1(n_0)\Bigg)^{\!\! 2} \,
	\Bigg]. \label{eq:tower2}
	\end{multline}
Conditionally on $\bS_{t,\epsilon}$, the random variables $\Big(\frac{P_{\rm out,2}''(Y_{t,\epsilon,\mu} | S_{t,\epsilon,\mu})}{P_{\rm out,2}(Y_{t,\epsilon,\mu} | S_{t,\epsilon,\mu})}\Big)_{1 \leq \mu \leq n_2}$ are i.i.d.\ and centered. Therefore
	\begin{align}
		\E \Bigg[
			\Bigg(\sum_{\mu=1}^{n_2} \frac{P_{\rm out,2}''(Y_{t,\epsilon,\mu} | S_{t,\epsilon,\mu})}{P_{\rm out,2}(Y_{t,\epsilon,\mu} | S_{t,\epsilon,\mu})}\Bigg)^{\!\! 2}
			\, \Bigg| \, \bX^1, \bS_{t,\epsilon}
		\Bigg]
		&=
		\E \Bigg[
			\Bigg(\sum_{\mu=1}^{n_2} \frac{P_{\rm out,2}''(Y_{t,\epsilon,\mu} | S_{t,\epsilon,\mu})}{P_{\rm out,2}(Y_{t,\epsilon,\mu} | S_{t,\epsilon,\mu})}\Bigg)^{\!\! 2}
			\, \Bigg| \, \bS_{t,\epsilon}
		\Bigg]\nn
		&= n_2 \, \E \Bigg[
			\Bigg(\frac{P_{\rm out,2}''(Y_{t,\epsilon,1} | S_{t,\epsilon,1})}{P_{\rm out,2}(Y_{t,\epsilon,1} | S_{t,\epsilon,1})}\Bigg)^{\!\! 2}
			\, \Bigg| \, \bS_{t,\epsilon}
		\Bigg]. \label{eq:var_simp}
	\end{align}
Under condition~\ref{hyp:c2}, it is not difficult to show that there exists a constant $C > 0$ such that
	\begin{align}
		\E \bigg[ \bigg(\frac{P_{\rm out,2}''(Y_{t,\epsilon,1} | S_{t,\epsilon,1})}{P_{\rm out,2}(Y_{t,\epsilon,1} | S_{t,\epsilon,1})}\bigg)^2 \, \bigg| \, \bS_{t,\epsilon} \bigg] \leq C\,.
		\label{eq:borne_ddp}
	\end{align}
Combining now \eqref{eq:borne_ddp}, \eqref{eq:var_simp} and \eqref{eq:tower2} we obtain that
	\begin{equation}
	\E \Bigg[
		\Bigg(\sum_{\mu=1}^{n_2} \frac{P_{\rm out,2}''(Y_{t,\epsilon,\mu} | S_{t,\epsilon,\mu})}{P_{\rm out,2}(Y_{t,\epsilon,\mu} | S_{t,\epsilon,\mu})}\Bigg)^{\!\! 2}
		\Bigg(\frac{\big\Vert\bX^1\big\Vert^2}{n_1} - \rho_1(n_0)\Bigg)^{\!\! 2}\,
		\Bigg]
	\leq C \cdot \frac{n_2}{n_1} \cdot
		\frac{1}{n_1} {\mathbb{V}\mathrm{ar}}\big(\big\Vert \bX^1 \big\Vert^2\big)\,. \label{boundOnCauchySchwartz}
	\end{equation}
It remains to prove that $\nicefrac{{\mathbb{V}\mathrm{ar}}\big(\big\Vert \bX^1 \big\Vert^2\big)}{n_1}$
is bounded, where we recall that $\bX^1 \defeq \varphi_1^2\Big( \frac{\bW{1} \bX^0}{\sqrt{n_0}}, \bA_{1} \Big)$. To do so, we use the identity
\begin{equation}
\frac{1}{n_1} {\mathbb{V}\mathrm{ar}}\big(\big\Vert \bX^1 \big\Vert^2\big)
= \frac{1}{n_1} \E\Big[{\mathbb{V}\mathrm{ar}}\big(\big\Vert \bX^1 \big\Vert^2 \big\vert \bX^0\big)\Big]
+ \frac{1}{n_1} {\mathbb{V}\mathrm{ar}}\Big(\E\big[\big\Vert \bX^1 \big\Vert^2 \big\vert \bX^0 \big]\Big)
\end{equation}
and show that both terms in the right hand side are bounded.\newline
First, the term $\E\big[{\mathbb{V}\mathrm{ar}}(\Vert \bX^1 \Vert^2| \bX^0)\big]$. Conditionally on $\bX^0$, the random variables $(X_i^1)_{1 \leq i \leq n_1}$ are i.i.d.\ and
\begin{equation}
{\mathbb{V}\mathrm{ar}}\big(\big\Vert \bX^1 \big\Vert^2 \big\vert \bX^0 \big)
= \sum_{i=1}^{n_1} {\mathbb{V}\mathrm{ar}}\big(\big(X_i^1\big)^2 \big\vert \bX^0 \big)
= n_1 \, {\mathbb{V}\mathrm{ar}}\big(\big(X_1^1\big)^2 \big\vert \bX^0 \big) \, .
\end{equation}
It follows that
\begin{equation}
\frac{1}{n_1} \E\big[{\mathbb{V}\mathrm{ar}}\big(\big\Vert \bX^1 \big\Vert^2 \big\vert \bX^0\big)\big]
= \E\big[{\mathbb{V}\mathrm{ar}}\big(\big(X_1^1\big)^2 \big\vert \bX^0 \big)\big]
\leq {\mathbb{V}\mathrm{ar}}\big(\big(X_1^1\big)^2\big)
\leq \E\bigg[ \varphi_1^4\bigg( \bigg[\frac{\bW{1} \bX^0}{\sqrt{n_0}}\bigg]_{1}, \bA_{1,1} \bigg)\bigg].
\end{equation}
Under~\ref{hyp:c2}, the expectation $\E\big[\varphi_1^4\big( \big[\nicefrac{\bW{1} \bX^0}{\sqrt{n_0}}\big]_{1}, \bA_{1,1} \big)\big]$ is bounded because $\varphi_1$ is bounded.\newline
Second, the term ${\mathbb{V}\mathrm{ar}}\big(\E[\Vert \bX^1 \Vert^2 | \bX^0]\big)$. We have
\begin{equation}
\E\big[ \big\Vert \bX^1 \big\Vert^2 \big\vert \bX^0\big]
= n_1 \cdot \E\bigg[\varphi_1^2\bigg( \bigg[\frac{\bW{1} \bX^0}{\sqrt{n_0}}\bigg]_{1} , \bA_{1,1}\bigg) \bigg\vert \bX^0\bigg] = n_1 \cdot g(X_1^0, \dots, X_{n_0}^0)
\end{equation}
where $g(\bc) \defeq \E\big[\varphi_1^2\big( \big[\nicefrac{\bW{1} \bc}{\sqrt{n_0}}\big]_{1} , \bA_{1,1}\big)\big]$ for any $n_0$-dimensional real vector $\bc = (c_1,\dots,c_{n_0})$.
The partial derivatives of $g$ satisfy for $1 \leq j \leq n_0$
\begin{align}
\frac{\partial g}{\partial c_j}(\bc)
&= \E\bigg[2\varphi_1\bigg( \bigg[\frac{\bW{1} \bc}{\sqrt{n_0}}\bigg]_{1}, \bA_{1,1} \bigg)\varphi'_1\bigg( \bigg[\frac{\bW{1} \bc}{\sqrt{n_0}}\bigg]_{1}, \bA_{1,1} \bigg) \frac{(\bW{1})_{1j}}{\sqrt{n_0}}\bigg]\nn
 &= \frac{2 c_j}{n_0}\E\bigg[\varphi_1^{'2}\bigg( \bigg[\frac{\bW{1} \bc}{\sqrt{n_0}}\bigg]_{1},\bA_{1,1} \bigg)
 + \varphi_1\bigg( \bigg[\frac{\bW{1} \bc}{\sqrt{n_0}}\bigg]_{1} , \bA_{1,1}\bigg)\varphi_1^{''}\bigg( \bigg[\frac{\bW{1} \bc}{\sqrt{n_0}}\bigg]_{1},\bA_{1,1} \bigg)\bigg] \, ,\label{partialDerivative_intByParts}
\end{align}
where \eqref{partialDerivative_intByParts} was obtained by integrating by parts w.r.t.\ $(\bW{1})_{1j}$.
The prior $P_0$ has bounded support $\mathcal{X} \subseteq [-S,S]$ under the hypothesis~\ref{hyp:bounded}. Then, for every $\bc \in \mathcal{X}^{n_0}$, we have
\begin{equation}
\bigg\vert \frac{\partial g}{\partial c_j}(\bc) \bigg\vert
\leq \frac{2 S}{n_0}\cdot \big(\sup \vert \varphi_1^{\prime} \vert^2 + \sup \vert \varphi_1  \vert \cdot \sup \vert \varphi_1^{''} \vert\big) \leq \frac{C}{n_0} \,,
\end{equation}
for some constant $C>0$. Here the hypothesis~\ref{hyp:c2} was used to bound the expectation in \eqref{partialDerivative_intByParts}. Thus, the function $g$ satisfies the bounded difference property, i.e.\ $\forall j \in \{1,\dots,n_0\}$
\begin{equation}
\sup_{\bc \in \mathcal{X}^{n_0}, c'_j\in \mathcal{X}} \big\vert g(\bc) - g(c_1, \dots,c'_j,\dots, c_{n_0}) \big\vert
\leq \frac{C}{n_0} \sup_{c_j, c'_j\in \mathcal{X}} \big\vert c_j - c'_j \big\vert \leq \frac{2S \cdot C}{n_0}\, .
\end{equation}
Applying Proposition~\ref{bounded_diff} (see Appendix~\ref{appendix_concentration}) it comes
\begin{equation}
{\mathbb{V}\mathrm{ar}}\big(g(\bX^0)\big) \leq \frac{1}{4} \sum_{j=1}^{n_0} \bigg(\frac{2S \cdot C}{n_0}\bigg)^2 = \frac{C'}{n_0} \,,
\end{equation}
and
\begin{equation}
\frac{1}{n_1}{\mathbb{V}\mathrm{ar}}\Big(\E\big[\big\Vert \bX^1 \big\Vert^2 \big\vert \bX^0\big]\Big)
= n_1 \cdot {\mathbb{V}\mathrm{ar}}\big(g(\bX^0)\big) \leq \frac{n_1}{n_0} C' \, .
\end{equation}
It ends the proof of
$\nicefrac{{\mathbb{V}\mathrm{ar}}(\Vert \bX^1 \Vert^2)}{n_1}$ boundedness in the limit ${n_0 \to +\infty}$. Combining \eqref{boundWithCS}, \eqref{boundOnCauchySchwartz} and the latter, it comes for $n_0$ large enough
\begin{align}\label{boundAwithVarFreeEntropy}
		\big\vert A_{\mathbf{n},\epsilon}(t) \big\vert
		\leq\,
		K \, \E\bigg[\bigg(\frac{\ln \cZ_{\mathbf{n},t,\epsilon}}{n_0} - f_{\mathbf{n},\epsilon}(t)\bigg)^{\! 2} \,\bigg]^{1/2}
\end{align}
with $K > 0$ a constant independent of both $t$ and $\epsilon$. The uniform convergence of $A_{\mathbf{n},\epsilon}(t)$ to 0 then follows from~\eqref{boundAwithVarFreeEntropy} and Theorem~\ref{concentrationtheorem} in Appendix~\ref{appendix_concentration}, that shows $\E[( n_0^{-1}\ln\cZ_{\mathbf{n},t,\epsilon} - f_{\mathbf{n},\epsilon}(t))^2]$ vanishes uniformly in $t$, $\epsilon$ when $n_0 \to +\infty$.
\section{Concentration of free entropy and overlaps}

\subsection{Concentration of the free entropy}\label{appendix_concentration}
In this section, we prove that the free entropy of the interpolation model studied in Sec.~\ref{interp-est-problem} concentrates around its expectation (uniformly in $t$ and $\epsilon$), i.e.\ we prove Theorem~\ref{concentrationtheorem} stated below.
$C(\varphi_1, \varphi_2, \alpha_1, \alpha_2, S)$ will denote any generic positive constant depending {\it only} on $\varphi_1$, $\varphi_2$, $\alpha_1$, $\alpha_2$, $S$.
Remember that $S$ is a bound on the signal absolute values. It is also understood that the dimensions $n_0$, $n_1$, $n_2$ are large enough so that $\nicefrac{n_1}{n_0} \approx \alpha_1$,  $\nicefrac{n_2}{n_1} \approx \alpha_2$.

\begin{thm}\label{concentrationtheorem}
Under~\ref{hyp:bounded}~\ref{hyp:c2}~\ref{hyp:phi_gauss2}, there exists a positive constant ${C(\varphi_1, \varphi_2, \alpha_1, \alpha_2, S)}$ such that $\forall t \in [0,1]$
\begin{equation}\label{fluctuation}
	\mathbb{E}\left[\left( \frac{\ln \mathcal{Z}_{\mathbf{n},t,\epsilon}}{n_0} - \mathbb{E}\left[\frac{\ln \mathcal{Z}_{\mathbf{n},t,\epsilon}}{n_0}\right] \right)^2\right] 
\leq \frac{C(\varphi_1, \varphi_2, \alpha_1, \alpha_2, S)}{n_0}.
\end{equation}
\end{thm}
One recalls some setups and notations for the reader's convenience. The interpolating Hamiltonian \eqref{stmu}-\eqref{interpolating-ham} is
\begin{equation}\label{interp-ham}
	- \sum_{\mu=1}^{n_2}
	\ln P_{\rm out,2} ( Y_{\mu} |s_{t,\epsilon}(\bx, \ba_1, u_\mu)) + \frac{1}{2} \sum_{i=1}^{n_1}\bigg(Y'_{i}  - \sqrt{R_1(t,\epsilon)}\, \varphi_1\bigg(\bigg[\frac{\bW{1} \bx}{\sqrt{n_0}}\bigg]_i, \ba_{1,i}\bigg)\bigg)^2 \; ,
\end{equation}
where $s_{t, \epsilon, \mu}(\bx, \ba_1, u_\mu) = \sqrt{\frac{1-t}{n_1}}\, \Big[\bW{2} \varphi_1\Big(\frac{\bW{1} \bx}{\sqrt{n_0}},\ba_1\Big)\Big]_\mu  + k_1(t) \,V_{\mu} + k_2(t) \,u_{\mu}$ with
\begin{equation*}
k_1(t,\epsilon) \defeq \sqrt{R_2(t,\epsilon)}, \quad k_2(t) \defeq \sqrt{\rho_1(n_0)t +2s_{n_0} - R_2(t,\epsilon)} \; .
\end{equation*}
This Hamiltonian follows from the interpolating model
\begin{align*}
\begin{cases}
Y_{t,\epsilon,\mu} = \varphi_2\Big(
\sqrt{\frac{1-t}{n_1}}\, \Big[\bW{2} \varphi_1\Big(\frac{\bW{1} \bX^0}{\sqrt{n_0}}, \bA_1\Big)\Big]_\mu  + k_1(t) \,V_{\mu} + k_2(t) \,U_{\mu}, \bA_{2,\mu}\Big) + \sqrt{\Delta} Z_\mu \,, &1 \leq \mu \leq n_2, \\
Y'_{t,\epsilon,i} \: = \sqrt{R_1(t,\epsilon)}\, \varphi_1\Big(\Big[\frac{\bW{1} \bX^0}{\sqrt{n_0}}\Big]_i, \bA_{1,i}\Big) + Z'_i\,, &1 \leq i \leq n_1 \,,
\end{cases}
\end{align*}
where $(\bA_{1,i})_{i = 1}^{n_1} \iid P_{A_1}$, $(\bA_{2,\mu})_{\mu = 1}^{n_2} \iid P_{A_2}$, $(Z_\mu)_{\mu=1}^{n_2} \,,(Z_i^{'})_{i=1}^{n_1} \iid {\cal N}(0,1)$. Recall the definition $\bX^1 \defeq \varphi_1\Big(\frac{\bW{1} \bX^0}{\sqrt{n_0}}, \bA_{1}\Big)$. To lighten notations, one also defines $\bx^1 \equiv \bx^1(\bx,\ba_1) \defeq \varphi_1\Big(\frac{\bW{1} \bx}{\sqrt{n_0}}, \ba_{1}\Big)$. The channel $P_{\rm out,2}$ defined in \eqref{defPout2} can be written as
\begin{align}
P_{\rm out,2}(Y_{t,\epsilon,\mu} | s_{t, \epsilon, \mu}(\bx, \ba_1, \bu))
&= \! \int \! dP_{A_2}(\ba_{2,\mu}) \frac{1}{\sqrt{2\pi\Delta}}
e^{-\frac{1}{2\Delta}\big(Y_{t,\epsilon,\mu} - \varphi_2(s_{t,\epsilon,\mu}(\bx,\ba_1,u_\mu),\ba_{2,\mu}\big)\big)^2}\nn
&= \! \int \! dP_{A_2}(\ba_{2,\mu}) \frac{1}{\sqrt{2\pi\Delta}}e^{-\frac{1}{2\Delta}(\Gamma_{t,\epsilon,\mu}(\bx, \ba_1, \ba_{2,\mu}, u_\mu)+ \sqrt{\Delta} Z_{\mu})^2 }
\label{RFrep}
\end{align}
with
\begin{multline}\label{defGamma}
\Gamma_{t,\epsilon,\mu}(\bx, \ba_1, \ba_{2,\mu}, u_\mu) 
	\defeq  \varphi_2\bigg(
	\sqrt{\frac{1-t}{n_1}}\, \big[\bW{2} \bX^1 \big]_\mu  + k_1(t) \,V_{\mu} + k_2(t) \,U_{\mu}, \bA_{2,\mu}\bigg)\\
	- \varphi_2\bigg(
	\sqrt{\frac{1-t}{n_1}}\, \big[\bW{2} \bx^1\big]_\mu  + k_1(t) \,V_{\mu} + k_2(t) \,u_{\mu}, \ba_{2,\mu}\bigg)\,.
\end{multline}
From \eqref{interp-ham}, \eqref{RFrep}, \eqref{defGamma} the free entropy of the interpolating model reads
\begin{align}
\frac{\ln \mathcal{Z}_{\mathbf{n},t,\epsilon}}{n_0} = \frac{\ln\widehat{\mathcal{Z}}_{\mathbf{n},t,\epsilon}}{n_0} - \frac{1}{2n_0}\sum_{\mu=1}^{n_2} Z_\mu^2 
-\frac{1}{2n_0}\sum_{i=1}^{n_1} Z_i^{\prime 2} - \frac{n_2}{2n_0} \ln(2\pi\Delta)\label{expressionZwithZhat}
\end{align}
where
\begin{equation}
\frac{\ln\widehat{\mathcal{Z}}_{\mathbf{n},t,\epsilon}}{n_0} = \frac{1}{n_0} \ln \bigg(\int  dP_0(\bx) dP_{A_1}(\ba_1) dP_{A_2}(\ba_2) \mathcal{D}\bu \, e^{- \widehat{\cH}_{t,\epsilon}(\bx,\ba_1, \ba_2, \bu)}\bigg)\,, \label{defZhat}
\end{equation}
and
\begin{multline}\label{explicit-ham}
\widehat{\cH}_{t,\epsilon}(\bx, \ba_1, \ba_2, \bu) =
	\frac{1}{2\Delta}\sum_{\mu=1}^{n_2} \Big(\Gamma_{t,\epsilon,\mu}(\bx,\ba_1, \ba_{2,\mu}, u_\mu)^2 
	+ 2 \sqrt{\Delta} Z_\mu \Gamma_{t,\epsilon,\mu}(\bx, \ba_1, \ba_{2,\mu},u_\mu)\Big)\\
	+\frac{1}{2} \sum_{i=1}^{n_1}\Big(\sqrt{R_1(t,\epsilon)} X_i^1 - \sqrt{R_1(t,\epsilon)} x_i^1 \Big)^2
	+ 2 Z_i^\prime \Big(\sqrt{R_1(t,\epsilon)} X_i^1- \sqrt{R_1(t,\epsilon)} x_i^1 \Big)\,.
\end{multline}
From \eqref{defGamma}, \eqref{defZhat}, \eqref{explicit-ham}, note that $\nicefrac{\ln \widehat{\mathcal{Z}}_{\mathbf{n},t,\epsilon}}{n_0}$ has been written as a function of $\bZ$, $\bZ^\prime$, $\bV$, $\bU$, $\bW{2}$, $\bA_2$, $\bW{1}$, $\bX^1$.
Our goal is to show that the free energy \eqref{expressionZwithZhat} concentrates around its expectation.
We will prove that there exists a positive constant $C(\varphi_1, \varphi_2, \alpha_1, \alpha_2, S)$ such that the variance of $\nicefrac{\ln \hat{\cZ}_{\mathbf{n},t,\epsilon}}{n_0}$ is bounded by $\nicefrac{C(\varphi_1, \varphi_2, \alpha_1, \alpha_2, S)}{n_0}$. This concentration property together with \eqref{expressionZwithZhat} implies \eqref{fluctuation}, i.e.\ Theorem~\ref{concentrationtheorem}.

The whole proof relies on two important variance bounds that are recalled below. The reader can refer to \cite{boucheron2004concentration} (Chapter 3) for detailed proofs of these statements.
\begin{proposition}[Gaussian Poincaré inequality]\label{poincare}
Let $\bU = (U_1, \dots, U_N)$ be a vector of $N$ independent standard normal random variables. Let $g: \mathbb{R}^N \to \mathbb{R}$ be a continuously differentiable function. Then
\begin{align}
	 {\mathbb{V}\mathrm{ar}}(g(\bU)) \leq \E \left[\| \nabla g (\bU) \|^2 \right] \,.
\end{align}
\end{proposition}

\begin{proposition}\label{bounded_diff}
Let $\mathcal{U} \subset \R$.
Let $g: \mathcal{U}^N \to \mathbb{R}$ a function that satisfies the bounded difference property, i.e., there exists some constants $c_1, \dots, c_N \geq 0$ such that
$$
\sup_{\substack{u_1, \dots u_N \in \mathcal{U}^N \\ u_i' \in \mathcal{U}}}
\vert g(u_1, \dots, u_i, \ldots, u_N) - g(u_1, \dots, u_i', \ldots, u_N)\vert \leq c_i
\,, \ \text{for all} \ 1 \leq i \leq N \,.
$$ 
Let $\boldsymbol{U}=(U_1, \dots, U_N)$ be a vector of $N$ independent random variables that takes values in $\mathcal{U}$. Then
\begin{align}
	 {\mathbb{V}\mathrm{ar}}(g(\bU)) \leq \frac{1}{4} \sum_{i=1}^N c_i^2 \,.
\end{align}
\end{proposition}
The order in which the concentrations are proved matters. First, we show the concentration w.r.t.\ the Gaussian variables $\bZ, \bZ^\prime, \bV, \bU, \bW{2}$ and w.r.t.\ $\bA_2$ thanks to the classical Gaussian Poincaré inequality and a bounded difference argument, respectively. This first part is performed working conditionally on $\bW{1}$, $\bX^1$.
Then the concentration w.r.t.\ $\bW{1}$ and $\bX^1 \defeq \varphi_1\big(\nicefrac{\bW{1}\bX^0}{\sqrt{n_0}}, \bA_1\big)$ is obtained by proving the concentation w.r.t.\ $\bX_1$, $\bW{1}$ and $\bX^0$, in this order.
 
Before starting the proof of Theorem~\ref{concentrationtheorem} we point out that, under~\ref{hyp:c2}, all the suprema $\sup \vert \varphi_k \vert$, $\sup \vert \varphi_k^{'} \vert$, $\sup \vert \varphi_k^{''} \vert$ for $k \in \{1,2\}$ are well-defined, and $\vert \bX_i^1 \vert \leq \sup \vert \varphi_1 \vert$ for all $i \in \{1,\dots,n_1\}$.

\subsubsection{Concentration conditionally on $\bW{1}$, $\bX^1$}
In all this subsection, we work conditionally on $\bW{1}$, $\bX^1$ and we prove that $\nicefrac{\ln \widehat{\cZ}_{\mathbf{n},t,\epsilon}}{n_0}$ is close to its expectation w.r.t.\ all the other random variables, namely $\bZ$, $\bZ'$, $\bV$, $\bU$, $\bW{2}$ and $\bA_2$:
\begin{lemma} \label{lem:concentration_Z_Z'_V_U_W2}
Under~\ref{hyp:bounded}~\ref{hyp:c2}~\ref{hyp:phi_gauss2}, there exists a positive constant ${C(\varphi_1, \varphi_2, \alpha_1, \alpha_2, S)}$ such that $\forall t \in [0,1]$
\begin{equation}
	\E \Bigg[\Bigg(
		\frac{\ln \widehat{\cZ}_{\mathbf{n},t,\epsilon}}{n_0}  - \E\bigg[\frac{\ln \widehat{\cZ}_{\mathbf{n},t,\epsilon}}{n_0}  \bigg\vert \bX^1, \bW{1}\bigg]\Bigg)^2\Bigg]
		\leq \frac{C(\varphi_1, \varphi_2, \alpha_1, \alpha_2, S)}{n_0} \,.
\end{equation}
\end{lemma}
\noindent Lemma~\ref{lem:concentration_Z_Z'_V_U_W2} follows simply from Lemmas~\ref{lem:concentrationGauss_Z_Z'}, \ref{lem:concentrationGauss_V_U_W2}, \ref{lem:concentration_A2} proven below.
\begin{lemma} \label{lem:concentrationGauss_Z_Z'}
Under~\ref{hyp:bounded}~\ref{hyp:c2}~\ref{hyp:phi_gauss2}, there exists a positive constant ${C(\varphi_1, \varphi_2, \alpha_1, \alpha_2, S)}$ such that $\forall t \in [0,1]$
	\begin{equation}
	\E \Bigg[\Bigg(
	\frac{\ln \widehat{\cZ}_{\mathbf{n},t,\epsilon}}{n_0}  - \E\bigg[\frac{\ln \widehat{\cZ}_{\mathbf{n},t,\epsilon}}{n_0} \bigg\vert \bV, \bU, \bW{2}, \bA_2, \bX^1, \bW{1}\bigg]
	\Bigg)^2\Bigg]
	\leq \frac{C(\varphi_1, \varphi_2, \alpha_1, \alpha_2, S)}{n_0} \,.
	\end{equation}
\end{lemma}
\begin{proof}
Here $g = \nicefrac{\ln\widehat{\mathcal{Z}}_{\mathbf{n},t,\epsilon}}{n_0}$ is seen as a function of $\bZ$, $\bZ^\prime$ \textit{only} and we work conditionally to all other random variables, i.e.\ $\bV$, $\bU$, $\bW{2}$, $\bA_2$, $\bX^1$, $\bW{1}$.
The squared norm of the gradient of $g$ reads
\begin{align}\label{gra}
\Vert \nabla g\Vert^2 = 
\sum_{\mu=1}^{n_2} \left|\frac{\partial g}{\partial Z_\mu}\right|^2  
+ 
\sum_{i=1}^{n_1} \left| \frac{\partial g}{\partial Z_i^\prime}\right|^2.
\end{align} 
Each of these partial derivatives are of the form 
$\partial g = -n_0^{-1} \langle \partial \widehat{\mathcal{H}}_{t,\epsilon}\rangle_{\widehat{\mathcal{H}}_{t,\epsilon}}$
where the Gibbs bracket $\langle - \rangle_{\widehat{\mathcal{H}}_{t,\epsilon}}$ pertains to the effective Hamiltonian \eqref{explicit-ham}. One finds
\begin{align*}
\left\vert\frac{\partial g}{\partial Z_\mu}\right\vert  &= \frac{1}{n_0\sqrt{\Delta}} \big\vert\langle \Gamma_{t,\mu} \rangle_{\widehat{\mathcal{H}}_{t,\epsilon}}\big\vert
\leq \frac{2}{n_0\sqrt{\Delta}} \sup\vert \varphi_2 \vert \;,\\
\left\vert\frac{\partial g}{\partial Z_i^\prime}\right\vert &= \frac{1}{n_0} \Big\vert\Big\langle \sqrt{R_1(t,\epsilon)} X_i^1 - \sqrt{R_1(t,\epsilon)}x_i^1 \Big\rangle_{\widehat{\mathcal{H}}_{t,\epsilon}}\Big\vert
\leq \frac{2 \sqrt{\rmax}}{n_0} \sup \vert \varphi_1 \vert \;,
\end{align*}
and, replacing in \eqref{gra}, one gets $\Vert \nabla g\Vert^2 \leq 4 n_0^{-1}\big(\frac{n_2}{n_0} \Delta^{-1} \sup \vert \varphi_2\vert^2 + \rmax \frac{n_1}{n_0}\sup \vert \varphi_1 \vert^2\big) $.
Applying Proposition~\ref{poincare}, one obtains
\begin{multline}\label{eq:var_part_Z}
	\E \Bigg[
	\Bigg(\frac{\ln \widehat{\cZ}_{\mathbf{n},t,\epsilon}}{n_0}  - \E\bigg[\frac{\ln \widehat{\cZ}_{\mathbf{n},t,\epsilon}}{n_0}
	\bigg\vert \bV, \bU, \bW{2}, \bA_2, \bX^1, \bW{1} \bigg]\Bigg)^2
	\Bigg\vert \bV, \bU, \bW{2}, \bA_2, \bX^1, \bW{1}\Bigg]\\
	\leq \frac{C(\varphi_1, \varphi_2, \alpha_1, \alpha_2, S)}{n_0}
\end{multline}
almost surely. Taking the expectation in \eqref{eq:var_part_Z} gives the lemma.
\end{proof}
\begin{lemma}\label{lem:concentrationGauss_V_U_W2}
Under~\ref{hyp:bounded}~\ref{hyp:c2}~\ref{hyp:phi_gauss2}, there exists a positive constant ${C(\varphi_1, \varphi_2, \alpha_1, \alpha_2, S)}$ such that $\forall t \in [0,1]$
	\begin{equation}
	\E \Bigg[\Bigg(
	\E\bigg[\frac{\ln \widehat{\cZ}_{\mathbf{n},t,\epsilon}}{n_0} \bigg\vert \bV, \bU, \bW{2}, \bA_2, \bX^1, \bW{1}\bigg]
	- \E\bigg[\frac{\ln \widehat{\cZ}_{\mathbf{n},t,\epsilon}}{n_0} \bigg\vert \bA_2, \bX^1, \bW{1} \bigg]
	\Bigg)^2\Bigg]
	\leq \frac{C(\varphi_1, \varphi_2, \alpha_1, \alpha_2, S)}{n_0} \,.
	\end{equation}
\end{lemma}
\begin{proof}
Here $g = \nicefrac{\mathbb{E}[\ln\widehat{\mathcal{Z}}_{\mathbf{n},t,\epsilon} \vert \bV, \bU, \bW{2}, \bA_2, \bX^1, \bW{1}]}{n_0}$ is seen as a function of $\bV$, $\bU$, $\bW{2}$ \textit{only} and we work conditionally to all other random variables, i.e.\ $\bA_2$, $\bX^1$, $\bW{1}$. From now on, and until the end of the proof, we will denote by $\tilde{\E}[\,\cdot\,]$ the conditional expectation $\E[\,\cdot\, \vert \bV, \bU, \bW{2}, \bA_2, \bX^1, \bW{1}]$
\begin{align*}
\left\vert\frac{\partial g}{\partial V_\mu}\right\vert
	&= n_0^{-1}\bigg\vert \tilde{\E}\bigg[\bigg\langle \big(\Gamma_{t, \epsilon, \mu} + \sqrt{\Delta} Z_\mu) \Delta^{-1} \frac{\partial \Gamma_{t, \epsilon, \mu}}{\partial V_\mu} \bigg\rangle_{\!\! \widehat{\mathcal{H}}_{t,\epsilon}}\bigg]\bigg\vert\\
	&\leq n_0^{-1}\tilde{\E}\bigg[(2\sup\vert \varphi_2\vert  + \sqrt{\Delta} \vert Z_\mu\vert) \frac{\sqrt{\rho_1(n_0)}}{\Delta} 2 \sup\vert \varphi_2^{'}\vert \bigg]\\
	&= n_0^{-1}\bigg(2\sup\vert \varphi_2\vert  + \sqrt{\frac{2\Delta}{\pi}}\bigg) \frac{\sqrt{\rho_1(n_0)}}{\Delta} 2 \sup\vert \varphi_2^{'}\vert
\end{align*}
The same inequality holds for $\big\vert\frac{\partial g}{\partial U_\mu}\big\vert$. To compute the derivative w.r.t.\ $(\bW{2})_{\mu i}$, first remark that
\begin{multline*}
\frac{\partial \Gamma_{t, \epsilon, \mu}}{\partial (\bW{2})_{\mu i}}
	= \sqrt{\frac{1-t}{n_1}} \bigg(
	X_i^1\varphi_2^{'}\bigg(
	\sqrt{\frac{1-t}{n_1}}\, \big[\bW{2} \bX^1\big]_\mu  + k_1(t) \,V_{\mu} + k_2(t) \,U_{\mu}\bigg)\\
	- x_i^1 \varphi_2^{'}\bigg(\sqrt{\frac{1-t}{n_1}}\, \big[\bW{2}\bx^1\big]_\mu  + k_1(t) \,V_{\mu} + k_2(t) \,u_{\mu}\bigg)
\bigg)
\end{multline*}
Therefore
\begin{align*}
\left\vert \frac{\partial g}{\partial (\bW{2})_{\mu i}} \right\vert
	& = n_0^{-1}\bigg\vert \tilde{\E}\bigg[\bigg\langle (\Gamma_{t, \epsilon,\mu} + \sqrt{\Delta}Z_\mu) \Delta^{-1}\frac{\partial \Gamma_{t, \epsilon,\mu}}{\partial (\bW{2})_{\mu i}} \bigg\rangle_{\!\! \widehat{\mathcal{H}}_{t,\epsilon}}\bigg]\bigg\vert\\
	&\leq \frac{1}{n_0 \sqrt{n_1}}\tilde{\E}\Big[(2\sup\vert \varphi_2\vert + \sqrt{\Delta} \vert Z_\mu\vert) \Delta^{-1}
(2 \sup\vert \varphi_1 \vert \sup\vert \varphi_2^{'} \vert)\Big] \\
	&= \frac{1}{n_0 \sqrt{n_1}}\bigg(2\sup\vert \varphi_2\vert + \sqrt{\frac{2\Delta}{\pi}}\bigg) \Delta^{-1}
(2 \sup\vert \varphi_1 \vert \sup\vert \varphi_2^{'} \vert)
\end{align*}
Putting these inequalities together one ends up with
\begin{align*}
\Vert \nabla g\Vert^2 &= \sum_{\mu=1}^{n_2} \left\vert\frac{\partial g}{\partial V_\mu}\right\vert^2
+
\sum_{\mu=1}^{n_2} \left\vert\frac{\partial g}{\partial U_\mu}\right\vert^2
+
\sum_{\mu=1}^{n_2} \sum_{i=1}^{n_1} \left\vert \frac{\partial g}{\partial (\bW{2})_{\mu i}} \right\vert^2\\
&\leq \frac{n_2}{n_0^2} \Big(2 \underbrace{\rho_1(n_0)}_{\to \rho_1}  + \sup\vert \varphi_1 \vert^2 \Big)
\left(\frac{2\sup\vert \varphi_2^{'} \vert}{\Delta}\right)^2 \Bigg(2\sup\vert \varphi_2\vert + \sqrt{\frac{2\Delta}{\pi}}\Bigg)^2 \; .
\end{align*}
Then the lemma follows once again of Proposition~\ref{poincare}.
\end{proof}
Next the variance bound of Lemma~\ref{bounded_diff} is applied to show the concentration of $\nicefrac{\mathbb{E}[\ln\hat{\mathcal{Z}}_{\mathbf{n},t,\epsilon} \vert \bA_2, \bX^1, \bW{1}]}{n_0}$ w.r.t.\ $\bA_2$.
\begin{lemma}\label{lem:concentration_A2}
Under~\ref{hyp:bounded}~\ref{hyp:c2}~\ref{hyp:phi_gauss2}, there exists a positive constant ${C(\varphi_1, \varphi_2, \alpha_1, \alpha_2, S)}$ such that $\forall t \in [0,1]$
	\begin{equation}
		\E \Bigg[\Bigg(
		\E\bigg[\frac{\ln \widehat{\cZ}_{\mathbf{n},t,\epsilon}}{n_0} \bigg\vert \bA_2, \bX^1, \bW{1} \bigg]
		- \E\bigg[\frac{\ln \widehat{\cZ}_{\mathbf{n},t,\epsilon}}{n_0} \bigg\vert \bX^1, \bW{1} \bigg]
		\Bigg)^2\Bigg]
		\leq \frac{C(\varphi_1, \varphi_2, \alpha_1, \alpha_2, S)}{n_0} \,.
	\end{equation}
\end{lemma}
\begin{proof}
Here $g = \nicefrac{\E[\ln\hat{\mathcal{Z}}_{\mathbf{n},t,\epsilon} \vert  \bA_2, \bX^1, \bW{1}]}{n_0}$ is seen as a function of $\bA_2$ only, and we work conditionally to $\bX^1, \bW{1}$. Define $\E_{G}$ as the expectation w.r.t.\ the Gaussian random variables $\bZ$, $\bZ'$, $\bV$, $\bU$, $\bW{2}$, so that $g = \nicefrac{E_G[\ln\widehat{\mathcal{Z}}_{\mathbf{n},t,\epsilon}]}{n_0}$.

Let $\nu \in \{1, \dots, n_2 \}$. One wants to estimate the variation $g(\bA_2) - g(\bA_2^{(\nu)})$ for two configurations $\bA_2$ and $\bA_2^{(\nu)}$ with $A_{2,\mu}^{(\nu)} = A_{2,\mu}$ for $\mu \neq \nu$.
$\widehat{\mathcal{H}}_{t,\epsilon}^{(\nu)}$ and $\Gamma_{t,\epsilon,\mu}^{(\nu)}$ will denote the quantities $\widehat{\mathcal{H}}_{t.\epsilon}$ and $\Gamma_{t,\epsilon,\mu}$ where $\bA_2$ is replaced by $\bA_2^{(\nu)}$, respectively.
By an application of Jensen's inequality one finds
\begin{align}\label{jensen1}
\frac{1}{n_0} \mathbb{E}_{G}\Big[ \big\langle \widehat{\mathcal{H}}_{t,\epsilon}^{(\nu)} - \widehat{\mathcal{H}}_{t,\epsilon} \big\rangle_{\widehat{\mathcal{H}}_{t,\epsilon}^{(\nu)}} \Big]
\leq 
g(\bA_2) - g(\bA_2^{(\nu)}) 
\leq 
\frac{1}{n_0} \mathbb{E}_{G}\Big[ \big\langle \widehat{\mathcal{H}}_{t,\epsilon}^{(\nu)} - \widehat{\mathcal{H}}_{t,\epsilon} \big\rangle_{\widehat{\mathcal{H}}_{t,\epsilon}} \Big]
\end{align}
where the Gibbs brackets pertain to the effective Hamiltonians \eqref{explicit-ham}. From \eqref{explicit-ham} we obtain 
\begin{multline*}
\widehat{\mathcal{H}}_{t,\epsilon}^{(\nu)} - \widehat{\mathcal{H}}_{t,\epsilon}
= \frac{1}{2\Delta}\sum_{\mu=1}^{n_2} \Gamma_{t,\epsilon,\mu}^{{(\nu)}2} - \Gamma_{t,\epsilon,\mu}^2 + 2\sqrt{\Delta} Z_\mu(\Gamma_{t,\epsilon,\mu}^{(\nu)} - \Gamma_{t,\epsilon,\mu})\\
= \frac{1}{2\Delta} \left( \Gamma_{t,\epsilon,\nu}^{{(\nu)}2} - \Gamma_{t,\epsilon,\nu}^2 + 2 \sqrt{\Delta} Z_\nu(\Gamma_{t,\epsilon,\nu}^{(\nu)} - \Gamma_{t,\epsilon,\nu}) \right) \,.
\end{multline*}
Notice that $\big\vert \Gamma_{t,\epsilon,\nu}^{(\nu)2} -  \Gamma_{t,\epsilon,\nu}^2 + 2Z_\nu (\Gamma_{t,\epsilon,\nu}^{(\nu)} - \Gamma_{t,\epsilon,\nu})\big\vert
\leq 8 \sup\vert\varphi_2\vert^2 + 4 \vert Z_\nu\vert \sup\vert\varphi_2\vert$. From \eqref{jensen1} we conclude that $g$ satisfies the bounded difference property:
\begin{equation*}
\Big\vert g\big(\bA_2\big) - g\big(\bA_2^{(\nu)}\big) \Big\vert
	\leq \frac{2\sup\vert\varphi_2\vert}{\Delta n_0} \Bigg( 2 \sup\vert\varphi_2\vert + \sqrt{\frac{2}{\pi}\Delta}\Bigg).
\end{equation*}
An application of Proposition~\ref{bounded_diff} (remember $\bA_{2,1}, \dots, \bA_{2,n_2}$ are i.i.d.) thus gives
\begin{multline*}
\E \Bigg[\Bigg(
\E\bigg[\frac{\ln \widehat{\cZ}_{\mathbf{n},t,\epsilon}}{n_0} \bigg\vert \bA_2, \bX^1, \bW{1} \bigg]
- \E\bigg[\frac{\ln \widehat{\cZ}_{\mathbf{n},t,\epsilon}}{n_0} \bigg\vert \bX^1, \bW{1} \bigg]
\Bigg)^2
\Bigg\vert \bX^1, \bW{1}\Bigg]\\
\leq \frac{n_2 \sup\vert\varphi_2\vert^2}{\Delta^2 n_0^2} \Bigg( 2 \sup\vert\varphi_2\vert + \sqrt{\frac{2}{\pi}\Delta}\Bigg)^2
\end{multline*}
almost surely. Taking the expectation on the latter bound gives Lemma~\ref{lem:concentration_A2}.
\end{proof}

\subsubsection{Concentration with respect to $\bX^1$ and $\bW{1}$}
In this subsection, we prove that $\nicefrac{\E[\ln \widehat{\cZ}_{\mathbf{n},t,\epsilon} \vert X^1 , \bW{1}]}{n_0}$ is close to its expectation, i.e.\ 
\begin{lemma}\label{lem:concentration_X1_W1}
Under~\ref{hyp:bounded}~\ref{hyp:c2}~\ref{hyp:phi_gauss2}, there exists a positive constant ${C(\varphi_1, \varphi_2, \alpha_1, \alpha_2, S)}$ such that $\forall t \in [0,1]$
	\begin{equation}
	\E \Bigg[\Bigg(
		\E\bigg[\frac{\ln \widehat{\cZ}_{\mathbf{n},t,\epsilon}}{n_0} \bigg\vert \bX^1,\bW{1} \bigg] - \E\bigg[\frac{\ln \widehat{\cZ}_{\mathbf{n},t,\epsilon}}{n_0} \bigg]
	\Bigg)^2\Bigg]
		\leq \frac{C(\varphi_1, \varphi_2, \alpha_1, \alpha_2, S)}{n_0} \,.
	\end{equation}
\end{lemma}
\begin{proof}
	Remember that $\bX^1 \defeq \varphi_1\big(\nicefrac{\bW{1} \bX^0}{\sqrt{n_0}}, \bA_{1}\big)$. Also $\nicefrac{\ln \widehat{\cZ}_{\mathbf{n},t,\epsilon}}{n_0}$ depends on $\bX^0$ \textit{only through} $\bX^1$. Therefore $\E\big[\nicefrac{\ln \widehat{\cZ}_{\mathbf{n},t,\epsilon}}{n_0} \big\vert \bX^1,\bW{1} \big] = \E\big[\nicefrac{\ln \hat{\cZ}_{\mathbf{n},t,\epsilon}}{n_0} \big\vert \bX^1,\bW{1},\bX^0 \big]$ and the lemma follows from Lemmas~\ref{lem:concentration_X1},~\ref{lem:concentration_W1},~\ref{lem:concentration_X0} stated and proven below.
\end{proof}
\begin{lemma}\label{lem:concentration_X1}
Under~\ref{hyp:bounded}~\ref{hyp:c2}~\ref{hyp:phi_gauss2}, there exists a positive constant ${C(\varphi_1, \varphi_2, \alpha_1, \alpha_2, S)}$ such that $\forall t \in [0,1]$
	\begin{equation}
	\E \Bigg[\Bigg(
	\E\bigg[\frac{\ln \widehat{\cZ}_{\mathbf{n},t,\epsilon}}{n_0} \bigg\vert \bX^1,\bW{1}, \bX^0 \bigg] - \E\bigg[\frac{\ln \widehat{\cZ}_{\mathbf{n},t,\epsilon}}{n_0}  \bigg\vert \bW{1}, \bX^0\bigg]
	\Bigg)^2\Bigg]
	\leq \frac{C(\varphi_1, \varphi_2, \alpha_1, \alpha_2, S)}{n_0} \,.
	\end{equation}
\end{lemma}
\begin{proof}
By definition $\bX^1 \defeq \varphi_1\big(\nicefrac{\bW{1} \bX^0}{\sqrt{n_0}}, \bA_{1}\big)$.
As $\bA_{1,1},\dots,\bA_{1,n_1}$ are i.i.d., the random variables $X_1^1,\dots,X_{n_1}^1$ are i.i.d.\ conditionally on $\bW{1}$ and $\bX^0$.
Define ${g(\bc) = \E[\nicefrac{\ln\widehat{\mathcal{Z}}_{\mathbf{n},t,\epsilon}}{n_0} \vert \bX^1 = \bc, \bW{1},\bX^0]}$. We will show that $g$ satisfies the bounded difference property, then an application of Proposition~\ref{bounded_diff} will end the proof.
	
Let $i \in \{1, \dots, n_1\}$. Consider two vectors ${\bc,\bc^{(i)} \in [-\sup \vert \varphi_1 \vert, \sup \vert \varphi_1 \vert]^{n_1}}$ such that $c_j^{(i)}= c_j$ for $j \neq i$.
For $s \in [0,1]$ we define $\psi(s) = g(s \bc + (1-s) \bc^{(i)})$. Hence $\psi(1) = g(\bc)$ and $\psi(0)=g(\bc^{(i)})$. If we can prove that
\begin{equation}\label{psiPrimeInequality}
	\big\vert \psi'(s) \big\vert \leq \frac{C(\varphi_1, \varphi_2, \alpha_1, \alpha_2, S)}{n_0} \quad \forall s \in [0,1] \, ,
\end{equation}
then the bounded difference property follows, namely
\begin{equation*}
	\sup\limits_{\bc, \bc^{(i)}} \big\vert g(\bc) - g(\bc^{(i)}) \big\vert \leq \frac{C(\varphi_1, \varphi_2, \alpha_1, \alpha_2, S)}{n_0} \:.
\end{equation*}
Let $\tilde{\E}[\,\cdot\,] \defeq \E[ \, \cdot \, \vert \bX^1 = s \bc + (1-s)\bc^{(i)}, \bW{1}, \bX^0]$. The derivative of $\psi$ satisfies
\begin{align}
	\big\vert \psi'(s) \big\vert
	&= \frac{\big\vert c_i^1 - c_i^{(i)} \big\vert}{n_0}
	\tilde{\E}\bigg[\bigg\langle \frac{\partial \widehat{\cH}_{t,\epsilon}}{\partial X_i^1} \bigg\rangle_{\!\! \hat{\cH}_{t,\epsilon}}\,\bigg]\nn
	&\leq \,\frac{2 \sup \vert \varphi_1 \vert}{n_0 \Delta}
	\sum_{\mu=1}^{n_2} \bigg\vert \tilde{\E}\bigg[\bigg\langle (\Gamma_{t,\epsilon,\mu} + \sqrt{\Delta}Z_\mu) \frac{\partial \Gamma_{t,\epsilon,\mu}}{\partial X_i^1} \bigg\rangle_{\!\widehat{\mathcal{H}}_{t,\epsilon}} \bigg]\bigg\vert\nn
	&\qquad\quad+ \frac{2 \sup \vert \varphi_1 \vert}{n_0} \Big\vert\tilde{\E}\Big[
	\sqrt{R_1(t,\epsilon)}\Big\langle \sqrt{R_1(t,\epsilon)} X_i^1 - \sqrt{R_1(t,\epsilon)} x_i^1 + Z_i^{'}\Big\rangle_{\!\widehat{\mathcal{H}}_{t,\epsilon}}\,\Big]\Big\vert\,.
\end{align}
This last expectation satisfies
\begin{multline*}
	\Big\vert \tilde{\E}\Big[\sqrt{R_1(t,\epsilon)}
	\Big\langle \sqrt{R_1(t,\epsilon)} X_i^1 - \sqrt{R_1(t,\epsilon)} x_i^1 + Z_i^{'}\Big\rangle_{\!\! \widehat{\mathcal{H}}_{t,\epsilon}} \,
	\Big]\bigg\vert\\
	\leq \tilde{\E}\Big[\sqrt{R_1(t,\epsilon)}
	\big(\sqrt{R_1(t,\epsilon)} \, 2 \sup \vert \varphi_1 \vert + \vert Z_i^{'} \vert \big)
	\Big]
	\leq \rmax \, 2 \sup \vert \varphi_1 \vert + \sqrt{\frac{2\rmax}{\pi}} \:,
\end{multline*}
For $\mu \in \{1,\dots,n_2\}$, set $K \defeq k_1(t) \,V_{\mu} + k_2(t) \,U_{\mu}$ to lighten the notations and perform an integration by parts w.r.t.\ $(\bW{2})_{\mu i}$ in the following expectation:
\begin{align*}
	&\bigg\vert \tilde{\E}\bigg[\bigg\langle \!(\Gamma_{t,\epsilon,\mu} + \sqrt{\Delta}Z_\mu) \frac{\partial \Gamma_{t,\epsilon,\mu}}{\partial X_i^1} \bigg\rangle_{\!\widehat{\mathcal{H}}_{t,\epsilon}}\,\bigg]\bigg\vert\\
	&\;= \bigg\vert \tilde{\E}\bigg[\bigg\langle \!(\Gamma_{t,\epsilon,\mu} + \sqrt{\Delta}Z_\mu)\sqrt{\frac{1-t}{n_1}}(\bW{2})_{\mu i}
	\varphi_2^{\prime}\bigg(\sqrt{\frac{1-t}{n_1}}\, \Big[\bW{2} \bX^1\Big]_\mu + K, \bA_{2,\mu}\bigg) \!\bigg\rangle_{\!\!\widehat{\mathcal{H}}_{t,\epsilon}} \,\bigg]\bigg\vert\\
	&\;\leq \frac{1}{n_1}\bigg\vert \tilde{\E}\bigg[\bigg\langle \!(\Gamma_{t,\epsilon,\mu} + \sqrt{\Delta}Z_\mu)X_i^1
	\varphi_2^{''}\bigg(\sqrt{\frac{1-t}{n_1}}\, \Big[\bW{2} \bX^1\Big]_\mu + K, \bA_{2,\mu}\bigg) \!\bigg\rangle_{\!\!\widehat{\mathcal{H}}_{t,\epsilon}} \,\bigg]\bigg\vert\\
	&\;\qquad+ \frac{1}{\sqrt{n_1}}\bigg\vert \tilde{\E}\bigg[\bigg\langle
	\frac{\partial \Gamma_{t,\epsilon, \mu}}{\partial (\bW{2})_{\mu,i}}
	\varphi_2^{\prime}\bigg(\sqrt{\frac{1-t}{n_1}}\, \Big[\bW{2} \bX^1\Big]_\mu + K, \bA_{2,\mu}\bigg) \!\bigg\rangle_{\!\!\widehat{\mathcal{H}}_{t,\epsilon}} \,\bigg]\bigg\vert\\
	&\;\qquad+ \frac{1}{\sqrt{n_1}} \bigg\vert \tilde{\E}\bigg[\bigg\langle \!(\Gamma_{t,\epsilon,\mu} + \sqrt{\Delta}Z_\mu)
	\varphi_2^{\prime}\bigg(\sqrt{\frac{1-t}{n_1}}\, \Big[\bW{2} \bX^1\Big]_\mu + K, \bA_{2,\mu}\bigg)
	\frac{\partial \widehat{\mathcal{H}}_{t,\epsilon}}{\partial (\bW{2})_{\mu i}}
	 \!\bigg\rangle_{\!\!\widehat{\mathcal{H}}_{t,\epsilon}} \,\bigg]\bigg\vert\\
	&\;\qquad+ \frac{1}{\sqrt{n_1}} \bigg\vert \tilde{\E}\bigg[\bigg\langle \!(\Gamma_{t,\epsilon,\mu} + \sqrt{\Delta}Z_\mu)
	\varphi_2^{\prime}\bigg(\sqrt{\frac{1-t}{n_1}}\, \Big[\bW{2} \bX^1\Big]_\mu + K, \bA_{2,\mu}\bigg) \!\bigg\rangle_{\!\!\widehat{\mathcal{H}}_{t,\epsilon}}
	\bigg\langle \frac{\partial \widehat{\mathcal{H}}_{t,\epsilon}}{\partial (\bW{2})_{\mu i}} \bigg\rangle_{\!\!\widehat{\mathcal{H}}_{t,\epsilon}}
	\bigg]\bigg\vert.
\end{align*}
It is easily shown that the first expectation is bounded by a constant $C(\varphi_1, \varphi_2, \alpha_1, \alpha_2, S)$, while the last three expectations are bounded by a term $\nicefrac{C(\varphi_1, \varphi_2, \alpha_1, \alpha_2, S)}{\sqrt{n_1}}$. The condition \eqref{psiPrimeInequality} is thus satisfied and Proposition~\ref{bounded_diff} implies that
	 \begin{equation*}
	 \E \Bigg[\Bigg(
	 \E\bigg[\frac{\ln \hat{\cZ}_{\mathbf{n},t,\epsilon}}{n_0} \bigg\vert \bX^1,\bW{1}, \bX^0 \bigg] - \E\bigg[\frac{\ln \hat{\cZ}_{\mathbf{n},t,\epsilon}}{n_0}  \bigg\vert \bW{1}, \bX^0\bigg]
	 \Bigg)^2  \Bigg\vert \bW{1}, \bX^0 \Bigg]
	 \leq \frac{C(\varphi_1, \varphi_2, \alpha_1, \alpha_2, S)}{n_0}
	 \end{equation*}
	 almost surely. Taking the expectation in this inequality ends the proof.
\end{proof}
\begin{lemma}\label{lem:concentration_W1}
Under~\ref{hyp:bounded}~\ref{hyp:c2}~\ref{hyp:phi_gauss2}, there exists a positive constant ${C(\varphi_1, \varphi_2, \alpha_1, \alpha_2, S)}$ such that $\forall t \in [0,1]$
	\begin{equation}
	\E \Bigg[\Bigg(
	\E\bigg[\frac{\ln \widehat{\cZ}_{\mathbf{n},t,\epsilon}}{n_0} \bigg\vert \bW{1}, \bX^0 \bigg] - \E\bigg[\frac{\ln \widehat{\cZ}_{\mathbf{n},t,\epsilon}}{n_0}  \bigg\vert \bX^0\bigg]
	\Bigg)^2\Bigg]
	\leq \frac{C(\varphi_1, \varphi_2, \alpha_1, \alpha_2, S)}{n_0} \,.
	\end{equation}
\end{lemma}
\begin{proof}
Here $g = \nicefrac{\mathbb{E}[\ln\widehat{\mathcal{Z}}_{\mathbf{n},t,\epsilon} \vert \bW{1}, \bX^0]}{n_0}$ is seen as a function of $\bW{1}$ \textit{only} and we work conditionally to $\bX^0$. To lighten the equations, we write $\tilde{\E}[\,\cdot\,]$ in place of $\E[\,\cdot\,\vert \bW{1}, \bX^0]$ and the second argument of $\varphi_1$ and $\varphi_2$ is not written explicitly. The partial derivatives of $g$ w.r.t.\ $(\bW{1})_{i j}$ reads
\begin{multline*}
\frac{\partial g}{\partial (\bW{1})_{i j}}
	= -\frac{1}{n_0 \Delta} \sum_{\mu=1}^{n_2} 
	\tilde{\E}\bigg[\bigg\langle (\Gamma_{t, \epsilon,\mu} + \sqrt{\Delta}Z_\mu) \frac{\partial \Gamma_{t, \epsilon,\mu}}{\partial (\bW{1})_{i j}} \bigg\rangle_{\!\! \widehat{\mathcal{H}}_{t,\epsilon}} \, \bigg]\\
-\frac{\sqrt{R_1(t,\epsilon)}}{n_0^{3/2}} \mathbb{E}\bigg[\bigg\langle 
\bigg(\sqrt{R_1(t,\epsilon)} X_i^1 - \sqrt{R_1(t,\epsilon)} \varphi_1\bigg(\bigg[\frac{\bW{1} \bx}{\sqrt{n_0}}\bigg]_i\bigg) + Z_i^{'}\bigg)\\
\cdot\bigg(X_j^0\varphi_1^{'}\bigg(\bigg[\frac{\bW{1} \bX^0}{\sqrt{n_0}}\bigg]_i\bigg)-
x_j\varphi_1^{'}\bigg(\bigg[\frac{\bW{1} \bx}{\sqrt{n_0}}\bigg]_i\bigg)\bigg)
\bigg\rangle_{\!\! \widehat{\mathcal{H}}_{t,\epsilon}} \,\bigg]
\end{multline*}
In a similar fashion to what is done in previous proofs, the absolute value of the second term in this partial derivative can be upperbounded by
\begin{equation*}
\frac{\sqrt{\rmax}}{n_0^{3/2}} \left(2\sqrt{\rmax} \sup \vert \varphi_1 \vert + \sqrt{\frac{2}{\pi}}\right) \cdot 2 S \, \sup \vert \varphi_1^{'} \vert \:.
\end{equation*}
The first term requires more work. First notice that
\begin{multline*}
\frac{\partial \Gamma_{t, \epsilon,\mu}}{\partial (\bW{1})_{ij}}\\= \sqrt{\frac{1-t}{n_0 n_1}}(\bW{2})_{\mu i}
\bigg(
X_j^0 \varphi_1^{'}\bigg(\bigg[\frac{\bW{1} \bX^0}{\sqrt{n_0}}\bigg]_i\bigg)
	\varphi_2^{'}\bigg(\sqrt{\frac{1-t}{n_1}}\, \bigg[\bW{2}
		\varphi_1\bigg(\frac{\bW{1} \bX^0}{\sqrt{n_0}}\bigg)
	\bigg]_\mu  + k_1(t) \,V_{\mu} + k_2(t) \,U_{\mu}\bigg)\\
- x_j \varphi_1^{'}\bigg(\bigg[\frac{\bW{1} \bx}{\sqrt{n_0}}\bigg]_i\bigg)
	\varphi_2^{'}\bigg(\sqrt{\frac{1-t}{n_1}}\, \bigg[\bW{2}
		\varphi_1\bigg(\frac{\bW{1} \bx}{\sqrt{n_0}}\bigg)
	\bigg]_\mu  + k_1(t) \,V_{\mu} + k_2(t) \,u_{\mu}\bigg)
\bigg).
\end{multline*}
It follows that
\begin{equation*}
\tilde{\E}\bigg[\bigg\langle (\Gamma_{t, \epsilon,\mu} + \sqrt{\Delta}Z_\mu) \frac{\partial \Gamma_{t, \epsilon,\mu}}{\partial (\bW{1})_{i j}} \bigg\rangle_{\!\! \widehat{\mathcal{H}}_{t,\epsilon}} \bigg]
= \sqrt{\frac{1-t}{n_0 n_1}} \tilde{\E}\bigg[ (\bW{2})_{\mu i} \Big\langle \! (\Gamma_{t, \epsilon,\mu} + \sqrt{\Delta}Z_\mu) \tilde{\Gamma}_{t, \epsilon,\mu}^{(ij)} \Big\rangle_{\! \widehat{\mathcal{H}}_{t,\epsilon}} \bigg],
\end{equation*}
where
\begin{multline*}
\tilde{\Gamma}_{t, \epsilon,\mu}^{(ij)} \defeq
X_j^0 \varphi_1^{'}\bigg(\bigg[\frac{\bW{1} \bX^0}{\sqrt{n_0}}\bigg]_i\bigg)
	\varphi_2^{'}\bigg(\sqrt{\frac{1-t}{n_1}}\, \bigg[\bW{2}
		\varphi_1\bigg(\frac{\bW{1} \bx}{\sqrt{n_0}}\bigg)
		\bigg]_\mu  + k_1(t) \,V_{\mu} + k_2(t) \,U_{\mu}\bigg)\\
- x_j\varphi_1^{'}\bigg(\bigg[\frac{\bW{1} \bx}{\sqrt{n_0}}\bigg]_i\bigg)
	\varphi_2^{'}\bigg(\sqrt{\frac{1-t}{n_1}}\, \bigg[\bW{2}
		\varphi_1\bigg(\frac{\bW{1} \bx}{\sqrt{n_0}}\bigg)\bigg]_\mu  + k_1(t) \,V_{\mu} + k_2(t) \,u_{\mu}\bigg).
\end{multline*}
An integration by parts w.r.t.\ $(\bW{2})_{\mu i}$ gives
\begin{align*}
&\tilde{\E}\bigg[ (\bW{2})_{\mu i} \bigg\langle \!\!(\Gamma_{t,\epsilon,\mu} + \sqrt{\Delta}Z_\mu) \tilde{\Gamma}_{t,\epsilon,\mu}^{(ij)} \bigg\rangle_{\!\!\widehat{\mathcal{H}}_{t,\epsilon}}\bigg]\\
&\qquad\qquad\qquad= \tilde{\E}\bigg[ \bigg\langle \frac{\partial \Gamma_{t, \epsilon,\mu}}{\partial (\bW{2})_{\mu i}} \tilde{\Gamma}_{t, \epsilon, \mu}^{(ij)} \bigg\rangle_{\!\widehat{\mathcal{H}}_{t,\epsilon}} \bigg]
+ \tilde{\E}\bigg[\bigg\langle \!(\Gamma_{t, \epsilon,\mu} + \sqrt{\Delta}Z_\mu) \frac{\partial \tilde{\Gamma}_{t, \epsilon, \mu}^{(ij)}}{\partial (\bW{2})_{\mu i}} \bigg\rangle_{\!\widehat{\mathcal{H}}_{t,\epsilon}}\bigg]\\
&\qquad\qquad\qquad\qquad- \tilde{\E}\bigg[\bigg\langle (\Gamma_{t, \epsilon,\mu} + \sqrt{\Delta}Z_\mu)^2 \, \tilde{\Gamma}_{t, \epsilon, \mu}^{(ij)} \Delta^{-1}\frac{\partial \Gamma_{t, \epsilon,\mu}}{\partial (\bW{2})_{i j}} \bigg\rangle_{\!\widehat{\mathcal{H}}_{t,\epsilon}}\bigg]\\
&\qquad\qquad\qquad\qquad+ \tilde{\E}\bigg[\Big\langle (\Gamma_{t, \epsilon,\mu} + \sqrt{\Delta}Z_\mu) \tilde{\Gamma}_{t, \epsilon, \mu}^{(\mu i)} \Big\rangle_{\!\widehat{\mathcal{H}}_{t,\epsilon}} \bigg\langle (\Gamma_{t, \epsilon,\mu} + \sqrt{\Delta}Z_\mu) \Delta^{-1}\frac{\partial \Gamma_{t, \epsilon,\mu}}{\partial (\bW{2})_{\mu i}} \bigg\rangle_{\!\! \widehat{\mathcal{H}}_{t,\epsilon}} \,\bigg].
\end{align*}
The first two conditional expectations satisfy
\begin{align*}
\Bigg\vert \tilde{\E}\bigg[ \bigg\langle \frac{\partial \Gamma_{t, \epsilon,\mu}}{\partial (\bW{2})_{\mu i}} \tilde{\Gamma}_{t, \epsilon, \mu}^{(ij)} \bigg\rangle_{\!\! \widehat{\mathcal{H}}_{t,\epsilon}} \,\bigg] \Bigg\vert
&\leq \frac{2 \sup \vert \varphi_1 \vert  \sup \vert \varphi_2^{'} \vert}{\sqrt{n_1}} \cdot 2S \sup \vert \varphi_1^{'} \vert \sup \vert \varphi_2^{'} \vert \,,\\
\Bigg\vert \tilde{\E}\bigg[\bigg\langle \big(\Gamma_{t, \epsilon,\mu} + \sqrt{\Delta}Z_\mu \big) \frac{\partial \tilde{\Gamma}_{t, \epsilon, \mu}^{(ij)}}{\partial (\bW{2})_{\mu i}} \bigg\rangle_{\!\! \widehat{\mathcal{H}}_{t,\epsilon}} \,\bigg] \Bigg\vert
&\leq \bigg( 2\sup\vert \varphi_2\vert + \sqrt{\frac{2\Delta}{\pi}} \bigg) \frac{2S \sup \vert \varphi_1^{'} \vert\sup \vert \varphi_1 \vert \sup \vert \varphi_2^{''} \vert}{\sqrt{n_1}} \,,
\end{align*}
while for the last two we have
\begin{align*}
&\bigg\vert \tilde{\E}\bigg[\bigg\langle \big(\Gamma_{t, \epsilon,\mu} + \sqrt{\Delta}Z_\mu \big)^2 \, \tilde{\Gamma}_{t, \epsilon, \mu}^{(ij)} \frac{\partial \Gamma_{t, \epsilon,\mu}}{\partial (\bW{2})_{i j}} \bigg\rangle_{\widehat{\mathcal{H}}_{t,\epsilon}}  \bigg] \bigg\vert\\
&\qquad \leq  \E \Big[ \big(2 \sup\vert \varphi_2\vert + \sqrt{\Delta} \vert Z_\mu \vert\big)^2 \Big] \cdot 2 S \sup \vert \varphi_1^{'} \vert \sup \vert \varphi_2^{'} \vert \cdot \frac{2 \sup \vert \varphi_1 \vert  \sup \vert \varphi_2^{'} \vert}{\sqrt{n_1}}\\
&\qquad =  \Bigg( 4 \sup\vert \varphi_2\vert^2 + \Delta + 2\sup\vert \varphi_2\vert \sqrt{\frac{2\Delta}{\pi}} \Bigg) \cdot 2 S \sup \vert \varphi_1^{'} \vert \sup \vert \varphi_2^{'} \vert \cdot \frac{2 \sup \vert \varphi_1 \vert  \sup \vert \varphi_2^{'} \vert}{\sqrt{n_1}} \:,\\
&\bigg\vert \tilde{\E}\bigg[\Big\langle \big(\Gamma_{t, \epsilon,\mu} + \sqrt{\Delta}Z_\mu\big) \tilde{\Gamma}_{t, \epsilon, \mu}^{(ij)} \Big\rangle_{\! \widehat{\mathcal{H}}_{t,\epsilon}} \bigg\langle \big(\Gamma_{t, \epsilon,\mu} + \sqrt{\Delta}Z_\mu\big) \frac{\partial \Gamma_{t, \epsilon,\mu}}{\partial (\bW{2})_{i j}} \bigg\rangle_{\!\! \widehat{\mathcal{H}}_{t,\epsilon}}
\,\bigg] \bigg\vert\\
&\qquad\leq \Bigg( 4 \sup\vert \varphi_2\vert^2 + \Delta + 2\sup\vert \varphi_2\vert \sqrt{\frac{2\Delta}{\pi}} \Bigg) \cdot 2 S \sup \vert \varphi_1^{'} \vert \sup \vert \varphi_2^{'} \vert \cdot \frac{2 \sup \vert \varphi_1 \vert  \sup \vert \varphi_2^{'} \vert}{\sqrt{n_1}} \:.
\end{align*}
Putting all these inequalities together, there exists a positive constant $C_{1}(\varphi_1, \varphi_2, \alpha_1, \alpha_2, S)$ such that
\begin{align*}
\Bigg\vert \frac{1}{n_0 \Delta} \sum_{\mu=1}^{n_2} \tilde{\E}\bigg[\Big\langle (\Gamma_{t, \epsilon,\mu} + \sqrt{\Delta}Z_\mu) \frac{\partial \Gamma_{t, \epsilon,\mu}}{\partial (\bW{1})_{i j}} \Big\rangle_{\widehat{\mathcal{H}}_{t,\epsilon}}\bigg] \Bigg\vert
\leq \frac{1}{n_0^{3/2}} \cdot \frac{n_2}{n_1} \cdot C_{1}(\varphi_1, \varphi_2, \alpha_1, \alpha_2, S)
\end{align*}
almost surely. Hence we have shown the existence of a constant $C(\varphi_1, \varphi_2, \alpha_1, \alpha_2, S)$ such that almost surely
$\forall (i,j) \in \{1,\dots,n_1\} \times \{1,\dots,n_0\}: \left\vert \nicefrac{\partial g}{\partial (\bW{1})_{i j}} \right\vert \leq \nicefrac{C_2(\varphi_1, \varphi_2, \alpha_1, \alpha_2, S)}{n_0^{3/2}}$. Finally
\begin{equation*}
\big\Vert \nabla g \big\Vert^2
= \sum_{i=1}^{n_1} \sum_{j=1}^{n_0} \bigg\vert \frac{\partial g}{\partial (\bW{1})_{i j}} \bigg\vert^2
\leq \frac{1}{n_0} \cdot \frac{n_1}{n_0} \Big( C_2(\varphi_1, \varphi_2, \alpha_1, \alpha_2, S \Big)^2
\end{equation*}
a.s.\ and an application of Proposition~\ref{poincare} ends the proof.
\end{proof}
\begin{lemma}\label{lem:concentration_X0}
Under~\ref{hyp:bounded}~\ref{hyp:c2}~\ref{hyp:phi_gauss2}, there exists a positive constant ${C(\varphi_1, \varphi_2, \alpha_1, \alpha_2, S)}$ such that $\forall t \in [0,1]$
	\begin{equation}
	\E \Bigg[\Bigg(
	\E\bigg[\frac{\ln \widehat{\cZ}_{\mathbf{n},t,\epsilon}}{n_0} \bigg\vert \bX^0 \bigg] - \E\bigg[\frac{\ln \widehat{\cZ}_{\mathbf{n},t,\epsilon}}{n_0}\bigg]
	\Bigg)^{\!\! 2} \,\Bigg]
	\leq \frac{C(\varphi_1, \varphi_2, \alpha_1, \alpha_2, S)}{n_0} \,.
	\end{equation}
\end{lemma}
\begin{proof}
Here $g = \nicefrac{\mathbb{E}[\ln\widehat{\mathcal{Z}}_{\mathbf{n},t,\epsilon} \vert \bX^0]}{n_0}$ is a function of $\bX^0$ \textit{only}. To lighten the equations, the second argument of $\varphi_1$ and $\varphi_2$ is never written explicitly. Also we omit the additional terms $k_1(t) \,\bV + k_2(t) \,\bU / \bu$ that appear in the first argument of $\varphi_2$, To be clear, we will abusively write:
\begin{gather*}
\varphi_1\bigg( \frac{\bW{1} \bX^0}{\sqrt{n_0}}\bigg) \equiv \varphi_1\bigg( \frac{\bW{1} \bX^0}{\sqrt{n_0}}, \bA_1\bigg)\,,\;
\varphi_1\bigg( \frac{\bW{1} \bx}{\sqrt{n_0}}\bigg) \equiv \varphi_1\bigg( \frac{\bW{1} \bx}{\sqrt{n_0}}, \ba_1\bigg)\,,\\
\;\,\varphi_2\bigg( \frac{\bW{2} \bX^1}{\sqrt{n_1}}\bigg) \equiv \varphi_2\bigg( \frac{\bW{2} \bX^1}{\sqrt{n_1}} + k_1(t)\bV + k_2(t)\bU , \bA_2\bigg)\,,\\
\varphi_2\bigg( \frac{\bW{2} \bx^1}{\sqrt{n_1}}\bigg) \equiv \varphi_2\bigg( \frac{\bW{2} \bx^1}{\sqrt{n_1}} + k_1(t)\bV + k_2(t)\bu , 
\ba_2\bigg)\,.
\end{gather*}
This notation also applies to the derivatives of $\varphi_1$ and $\varphi_2$ w.r.t.\ their first argument. We will show that the partial derivatives of $g$ are almost surely bounded by $\nicefrac{C(\varphi_1, \varphi_2, \alpha_1, \alpha_2, S)}{n_0}$. Then a bounded difference argument similar to the one used in Lemma~\ref{lem:concentration_X1} will end the proof. For $j \in \{1,\dots,n_0\}$:
\begin{multline*}
\frac{\partial g}{\partial X_j^0}
=\\ \frac{-\Delta^{-1}}{n_0^{\nicefrac{3}{2}} \sqrt{n_1}} \sum_{\mu=1}^{n_2} \sum_{i=1}^{n_1}
\tilde{\E}\bigg[(\bW{1})_{ij}(\bW{2})_{\mu i}\varphi_1^{'}\bigg(\bigg[\frac{\bW{1} \bX^0}{\sqrt{n_0}}\bigg]_i\bigg)
\varphi_2^{'}\bigg(\sqrt{\frac{1-t}{n_1}}\big[\bW{2} \bX^1\big]_{\mu}\bigg)
\big\langle \Gamma_{t, \epsilon,\mu} + \sqrt{\Delta}Z_\mu\big\rangle_{\widehat{\mathcal{H}}_{t,\epsilon}} \bigg]\\
-\frac{\sqrt{R_1(t,\epsilon)}}{n_0^{\nicefrac{3}{2}}} \sum_{i=1}^{n_1}
\tilde{\E}\bigg[(\bW{1})_{ij}\varphi_1^{'}\bigg(\bigg[\frac{\bW{1} \bX^0}{\sqrt{n_0}}\bigg]_i\bigg)\bigg\langle 
\sqrt{R_1(t,\epsilon)} X_i^1 - \sqrt{R_1(t,\epsilon)} \varphi_1\bigg(\bigg[\frac{\bW{1} \bx}{\sqrt{n_0}}\bigg]_i\bigg) + Z_i^{'}\bigg\rangle_{\!\! \widehat{\mathcal{H}}_{t,\epsilon}} \, \bigg]\,.
\end{multline*}
This expression can be further simplified using that $\bZ,\bZ'$ are centred and independent of everything:
\begin{multline*}
\frac{\partial g}{\partial X_j^0}=\\
\frac{-\Delta^{-1}}{n_0 \sqrt{n_0 n_1}} \sum_{\mu=1}^{n_2} \sum_{i=1}^{n_1}
\tilde{\E}\bigg[(\bW{1})_{ij}(\bW{2})_{\mu i}\varphi_1^{'}\bigg(\bigg[\frac{\bW{1} \bX^0}{\sqrt{n_0}}\bigg]_i\bigg)
\varphi_2^{'}\bigg(\sqrt{\frac{1-t}{n_1}}\big[\bW{2} \bX^1\big]_{\mu}\bigg)
\big\langle \Gamma_{t, \epsilon,\mu}\big\rangle_{\widehat{\mathcal{H}}_{t,\epsilon}} \bigg]\\
-\frac{R_1(t,\epsilon)}{n_0^{\nicefrac{3}{2}}} \sum_{i=1}^{n_1}
\tilde{\E}\bigg[(\bW{1})_{ij}\varphi_1^{'}\bigg(\bigg[\frac{\bW{1} \bX^0}{\sqrt{n_0}}\bigg]_i\bigg)
\bigg\langle X_i^1 -\varphi_1\bigg(\bigg[\frac{\bW{1} \bx}{\sqrt{n_0}}\bigg]_i\bigg)\bigg\rangle_{\!\! \widehat{\mathcal{H}}_{t,\epsilon}} \bigg]\,.
\end{multline*}
Integrating by parts w.r.t.\ $(\bW{1})_{ij}$, we get
\begin{equation*}
\tilde{\E}\bigg[(\bW{1})_{ij}
\varphi_1^{'}\bigg(\bigg[\frac{\bW{1} \bX^0}{\sqrt{n_0}}\bigg]_i\bigg) X_i^1  \bigg]
	= \frac{X_j^0}{\sqrt{n_0}} \tilde{\E}\bigg[\varphi_1^{''}\bigg(\bigg[\frac{\bW{1} \bX^0}{\sqrt{n_0}}\bigg]_i\bigg) X_i^1
	+ \varphi_1^{'}\bigg(\bigg[\frac{\bW{1} \bX^0}{\sqrt{n_0}}\bigg]_i\bigg)^{\!\! 2} \,  \bigg] \,,
\end{equation*}
whose absolute value is upperbounded by $\nicefrac{C(\varphi_1, \varphi_2, \alpha_1, \alpha_2, S)}{\sqrt{n_0}}$, and
\begin{align*}
&\tilde{\E}\bigg[(\bW{1})_{ij}
\varphi_1^{'}\bigg(\bigg[\frac{\bW{1} \bX^0}{\sqrt{n_0}}\bigg]_i\bigg)
\bigg\langle \varphi_1\bigg(\bigg[\frac{\bW{1} \bx}{\sqrt{n_0}}\bigg]_i\bigg)\bigg\rangle_{\!\! \widehat{\mathcal{H}}_{t,\epsilon}}  \bigg]\\
	&\qquad= \frac{X_j^0}{\sqrt{n_0}} \tilde{\E}\bigg[
	\varphi_1^{''}\bigg(\bigg[\frac{\bW{1} \bX^0}{\sqrt{n_0}}\bigg]_i\bigg)
	\bigg\langle \varphi_1\bigg(\bigg[\frac{\bW{1} \bx}{\sqrt{n_0}}\bigg]_i\bigg)\bigg\rangle_{\!\! \widehat{\mathcal{H}}_{t,\epsilon}} \bigg]\\
	&\qquad\qquad\qquad\qquad\qquad\qquad\qquad+ \tilde{\E}\bigg[
	\varphi_1^{'}\bigg(\bigg[\frac{\bW{1} \bX^0}{\sqrt{n_0}}\bigg]_i\bigg)
	\bigg\langle \varphi_1^{'}\bigg(\bigg[\frac{\bW{1} \bx}{\sqrt{n_0}}\bigg]_i\bigg)  \frac{x_j}{\sqrt{n_0}}\bigg\rangle_{\!\! \widehat{\mathcal{H}}_{t,\epsilon}} \bigg]\\
	&\qquad\qquad\qquad\qquad\qquad\qquad\qquad + \tilde{\E}\bigg[
	\varphi_1^{'}\bigg(\bigg[\frac{\bW{1} \bX^0}{\sqrt{n_0}}\bigg]_i\bigg)
	\bigg\langle \varphi_1\bigg(\bigg[\frac{\bW{1} \bx}{\sqrt{n_0}}\bigg]_i\bigg)
	\frac{\partial \widehat{\mathcal{H}}_{t,\epsilon}}{\partial (\bW{1})_{ij}}\bigg\rangle_{\!\! \widehat{\mathcal{H}}_{t,\epsilon}} \bigg]\\
	&\qquad\qquad\qquad\qquad\qquad\qquad\qquad + \tilde{\E}\bigg[
	\varphi_1^{'}\bigg(\bigg[\frac{\bW{1} \bX^0}{\sqrt{n_0}}\bigg]_i\bigg)
	\bigg\langle \varphi_1\bigg(\bigg[\frac{\bW{1} \bx}{\sqrt{n_0}}\bigg]_i\bigg) \bigg\rangle_{\!\! \widehat{\mathcal{H}}_{t,\epsilon}}
	\bigg\langle \frac{\partial \widehat{\mathcal{H}}_{t,\epsilon}}{\partial (\bW{1})_{ij}}\bigg\rangle_{\!\! \widehat{\mathcal{H}}_{t,\epsilon}} \bigg] \,.
\end{align*}
The first two conditional expectations on the right hand side of this last equality are upperbounded by $\nicefrac{C(\varphi_1, \varphi_2, \alpha_1, \alpha_2, S)}{\sqrt{n_0}}$. It is also the case for the last two conditional expectations: it suffices to proceed as in the proof of Lemma~\ref{lem:concentration_W1} (note that $\nicefrac{\partial \widehat{\mathcal{H}}_{t,\epsilon}}{\partial (\bW{1})_{ij}}$ was already computed to prove the aforementioned lemma). Hence there exists a positive constant $C(\varphi_1, \varphi_2, \alpha_1, \alpha_2, S)$ such that almost surely for every $i \in \{1,\dots,n_1\}$
\begin{multline}\label{upperboundExp_i}
\bigg\vert \tilde{\E}\bigg[(\bW{1})_{ij}
	\varphi_1^{'}\bigg(\bigg[\frac{\bW{1} \bX^0}{\sqrt{n_0}}\bigg]_i\bigg)\bigg
	\langle X_i^1 -\varphi_1\bigg(\bigg[\frac{\bW{1} \bx}{\sqrt{n_0}}\bigg]_i\bigg)\bigg\rangle_{\!\! \widehat{\mathcal{H}}_{t,\epsilon}} \bigg]		\bigg\vert
\leq \frac{C(\varphi_1, \varphi_2, \alpha_1, \alpha_2, S)}{\sqrt{n_0}} \,.
\end{multline}
It remains to upperbound, for every pair $(\mu,i) \in \{1,\dots,n_2\} \times \{1,\dots,n_1\}$, the absolute value of the conditional expectation
\begin{flalign*}
&\tilde{\E}\bigg[(\bW{1})_{ij}(\bW{2})_{\mu i}\varphi_1^{'}\bigg(\bigg[\frac{\bW{1} \bX^0}{\sqrt{n_0}}\bigg]_i\bigg)
\varphi_2^{'}\bigg(\sqrt{\frac{1-t}{n_1}}\big[\bW{2} \bX^1\big]_{\mu}\bigg)
\big\langle \Gamma_{t, \epsilon,\mu}\big\rangle_{\widehat{\mathcal{H}}_{t,\epsilon}} \bigg]&\\
&= \tilde{\E}\bigg[(\bW{1})_{ij}(\bW{2})_{\mu i}\varphi_1^{'}\bigg(\bigg[\frac{\bW{1} \bX^0}{\sqrt{n_0}}\bigg]_i\bigg)
\varphi_2^{'}\bigg(\sqrt{\frac{1-t}{n_1}}\big[\bW{2} \bX^1\big]_{\mu}\bigg)
\varphi_2\bigg(\sqrt{\frac{1-t}{n_1}}\big[\bW{2} \bX^1\big]_{\mu}\bigg)\bigg]\\
&\quad- \tilde{\E}\bigg[(\bW{1})_{ij}(\bW{2})_{\mu i}\varphi_1^{'}\bigg(\bigg[\frac{\bW{1} \bX^0}{\sqrt{n_0}}\bigg]_i\bigg)
\varphi_2^{'}\bigg(\sqrt{\frac{1-t}{n_1}}\big[\bW{2} \bX^1\big]_{\mu}\bigg)
\bigg\langle\varphi_2\bigg(\sqrt{\frac{1-t}{n_1}}\big[\bW{2} \bx^1\big]_{\mu}\bigg) \bigg\rangle_{\! \widehat{\mathcal{H}}_{t,\epsilon}} \,\bigg].
\end{flalign*}
The first conditional expectation on the right hand side is easily upperbounded by $\nicefrac{C(\varphi_1, \varphi_2, \alpha_1, \alpha_2, S)}{n_0}$: integrate by parts, first w.r.t.\ $(\bW{1})_{ij}$ and then w.r.t.\ $(\bW{2})_{\mu i}$. For the second conditional expectation, an integration by parts w.r.t.\ $(\bW{1})_{ij}$ returns:
\begin{flalign*}
&\tilde{\E}\bigg[(\bW{1})_{ij}(\bW{2})_{\mu i}\varphi_1^{'}\bigg(\bigg[\frac{\bW{1} \bX^0}{\sqrt{n_0}}\bigg]_i\bigg)
\varphi_2^{'}\bigg(\sqrt{\frac{1-t}{n_1}}\big[\bW{2} \bX^1\big]_{\mu}\bigg)
\bigg\langle\varphi_2\bigg(\sqrt{\frac{1-t}{n_1}}\big[\bW{2} \bx^1\big]_{\mu}\bigg) \bigg\rangle_{\!\! \widehat{\mathcal{H}}_{t,\epsilon}} \,\bigg]&\\
&=\tilde{\E}\bigg[\frac{X_j^0}{\sqrt{n_0}} (\bW{2})_{\mu i}\varphi_1^{''}\bigg(\bigg[\frac{\bW{1} \bX^0}{\sqrt{n_0}}\bigg]_i\bigg)
\varphi_2^{'}\bigg(\sqrt{\frac{1-t}{n_1}}\big[\bW{2} \bX^1\big]_{\mu}\bigg)
\bigg\langle\varphi_2\bigg(\sqrt{\frac{1-t}{n_1}}\big[\bW{2} \bx^1\big]_{\mu}\bigg)
\bigg\rangle_{\!\! \widehat{\mathcal{H}}_{t,\epsilon}}\,\bigg]\\
&\;+\tilde{\E}\bigg[\sqrt{\frac{1-t}{n_0 n_1}}X_j^0 (\bW{2})_{\mu i}^2\varphi_1^{\prime}\bigg(\bigg[\frac{\bW{1} \bX^0}{\sqrt{n_0}}\bigg]_i\bigg)^{\!\! 2}
\varphi_2^{''}\bigg(\sqrt{\frac{1-t}{n_1}}\big[\bW{2} \bX^1\big]_{\mu}\bigg)\\
&\;\qquad\qquad\qquad\qquad\qquad\qquad\qquad\qquad\qquad\cdot \bigg\langle \! \varphi_2\bigg(\sqrt{\frac{1-t}{n_1}}\big[\bW{2} \bx^1\big]_{\mu}\bigg)\!
\bigg\rangle_{\!\! \widehat{\mathcal{H}}_{t,\epsilon}}\,\bigg]\\
&\;+\tilde{\E}\bigg[\sqrt{\frac{1-t}{n_0 n_1}}(\bW{2})_{\mu i}^2
\varphi_1^{\prime}\bigg(\bigg[\frac{\bW{1} \bX^0}{\sqrt{n_0}}\bigg]_i\bigg)
\varphi_2^{'}\bigg(\sqrt{\frac{1-t}{n_1}}\big[\bW{2} \bX^1\big]_{\mu}\bigg)\\
&\;\qquad\qquad\qquad\qquad\qquad\qquad\qquad\qquad\qquad\cdot \bigg\langle x_j \varphi_1^{\prime}\bigg(\bigg[\frac{\bW{1} \bx}{\sqrt{n_0}}\bigg]_i\bigg)
\varphi_2^{\prime}\bigg(\sqrt{\frac{1-t}{n_1}}\big[\bW{2} \bx^1\big]_{\mu}\bigg)
\bigg\rangle_{\!\! \widehat{\mathcal{H}}_{t,\epsilon}} \,\bigg]\\
&\;-\tilde{\E}\bigg[(\bW{2})_{\mu i}\varphi_1^{'}\bigg(\bigg[\frac{\bW{1} \bX^0}{\sqrt{n_0}}\bigg]_i\bigg)
\varphi_2^{'}\bigg(\sqrt{\frac{1-t}{n_1}}\big[\bW{2} \bX^1\big]_{\mu}\bigg)
\bigg\langle\varphi_2\bigg(\sqrt{\frac{1-t}{n_1}}\big[\bW{2} \bx^1\big]_{\mu}\bigg)
\frac{\partial \widehat{\mathcal{H}}_{t,\epsilon}}{\partial (\bW{1})_{ij}} \bigg\rangle_{\!\! \widehat{\mathcal{H}}_{t,\epsilon}} \,\bigg]\\
&\;+\tilde{\E}\bigg[(\bW{2})_{\mu i}\varphi_1^{'}\bigg(\bigg[\frac{\bW{1} \bX^0}{\sqrt{n_0}}\bigg]_i\bigg)
\varphi_2^{'}\bigg(\sqrt{\frac{1-t}{n_1}}\big[\bW{2} \bX^1\big]_{\mu}\bigg)\\
&\;\qquad\qquad\qquad\qquad\qquad\qquad\qquad\qquad\qquad\cdot 
\bigg\langle\varphi_2\bigg(\sqrt{\frac{1-t}{n_1}}\big[\bW{2} \bx^1\big]_{\mu}\bigg)\bigg\rangle_{\!\! \widehat{\mathcal{H}}_{t,\epsilon}}
\bigg\langle\frac{\partial \widehat{\mathcal{H}}_{t,\epsilon}}{\partial (\bW{1})_{ij}} \bigg\rangle_{\!\! \widehat{\mathcal{H}}_{t,\epsilon}} \,\bigg].
\end{flalign*}
The first conditional expectation on the right hand side can then be upperbounded, after integrating by parts w.r.t.\ $(\bW{2})_{\mu i}$, by $\nicefrac{C(\varphi_1, \varphi_2, \alpha_1, \alpha_2, S)}{n_0}$. The second and third conditional expectations are easily upperbounded by
\begin{equation*}
\frac{1}{\sqrt{n_0 n_1}} \sup \vert\varphi_1^{'}\vert^2 \cdot \sup \vert\varphi_2\vert \cdot \sup \vert\varphi_2^{''}\vert
\cdot \underbrace{\E\big[(\bW{2})_{\mu i}^2\big]}_{= 1} \leq \frac{C(\varphi_1, \varphi_2, \alpha_1, \alpha_2, S)}{n_0} \,.
\end{equation*}
For the fourth and fifth conditional expectations, we recall the following expression for the partial derivative $\nicefrac{\partial \widehat{\mathcal{H}}_{t,\epsilon}}{\partial (\bW{1})_{ij}}$:
\begin{multline}\label{partialDerivativeH_W1}
\frac{\partial \widehat{\mathcal{H}}_{t,\epsilon}}{\partial (\bW{1})_{ij}}=
-\frac{1}{\Delta} \sqrt{\frac{1-t}{n_0 n_1}} \sum_{\nu=1}^{n_2} (\bW{2})_{\nu i}
\big(\Gamma_{t, \epsilon, \nu} + \sqrt{\Delta}Z_\nu \big) \tilde{\Gamma}_{t, \epsilon, \nu}^{(ij)}\\
-\sqrt{\frac{R_1(t,\epsilon)}{n_0}}
	\bigg(\sqrt{R_1(t,\epsilon)} X_i^1 - \sqrt{R_1(t,\epsilon)} \varphi_1\bigg(\bigg[\frac{\bW{1} \bx}{\sqrt{n_0}}\bigg]_i\bigg) + Z_i^{'}\bigg)\\
	\cdot\bigg(X_j^0\varphi_1^{'}\bigg(\bigg[\frac{\bW{1} \bX^0}{\sqrt{n_0}}\bigg]_i\bigg)-
x_j\varphi_1^{'}\bigg(\bigg[\frac{\bW{1} \bx}{\sqrt{n_0}}\bigg]_i\bigg)\bigg)\,.
\end{multline}
We can use this expression to split both conditional expectations into $n_2 + 1$ terms such that:
\begin{itemize}
\item the term due to $\nu = \mu$ is directly upperbounded by $\nicefrac{C(\varphi_1, \varphi_2, \alpha_1, \alpha_2, S) \E[ (\bW{2})_{\mu i}^2]}{\sqrt{n_1 n_0}}\:$;
\item each term due to a $\nu \neq \mu$ is upperbounded by $\nicefrac{C(\varphi_1, \varphi_2, \alpha_1, \alpha_2, S)}{n_0^2}$ after integrating by parts w.r.t.\ $(\bW{2})_{\mu i}$ and $(\bW{2})_{\nu i}\,$;
\item the term due to the second line of \eqref{partialDerivativeH_W1} is upperbounded by $\nicefrac{C(\varphi_1, \varphi_2, \alpha_1, \alpha_2, S)}{n_0}$ after integrating by parts w.r.t.\ $(\bW{2})_{\mu i}$.
\end{itemize}
All in all, there exists a positive constant $C(\varphi_1, \varphi_2, \alpha_1, \alpha_2, S)$ such that almost surely
\begin{multline}\label{upperboundExp_mu_i}
\forall (\mu,i) \in \{1,\dots,n_2\} \times \{1,\dots,n_1\}\,:\\
\bigg\vert \tilde{\E}\bigg[(\bW{1})_{ij}(\bW{2})_{\mu i}\varphi_1^{'}\bigg(\bigg[\frac{\bW{1} \bX^0}{\sqrt{n_0}}\bigg]_i\bigg)
\varphi_2^{'}\bigg(\sqrt{\frac{1-t}{n_1}}\big[\bW{2} \bX^1\big]_{\mu}\bigg)
\big\langle \Gamma_{t, \epsilon,\mu}\big\rangle_{\widehat{\mathcal{H}}_{t,\epsilon}} \bigg]	\bigg\vert\\
\leq \frac{C(\varphi_1, \varphi_2, \alpha_1, \alpha_2, S)}{n_0} \,.
\end{multline}
Putting \eqref{upperboundExp_i} and \eqref{upperboundExp_mu_i} together gives the existence of a positive constant $C(\varphi_1, \varphi_2, \alpha_1, \alpha_2, S)$ such that for all $j \in \{1,\dots,n_0\}$ we have $\vert \nicefrac{\partial g}{\partial X_j^0} \vert \leq \nicefrac{C(\varphi_1, \varphi_2, \alpha_1, \alpha_2, S)}{n_0}$ almost surely, thus ending the proof.
\end{proof}
\subsubsection{Proof of Theorem~\ref{concentrationtheorem}}
From Lemmas~\ref{lem:concentration_Z_Z'_V_U_W2} and~\ref{lem:concentration_X1_W1}, we directly obtain the bound
\begin{multline*}
{\mathbb{V}\mathrm{ar}}\bigg(\frac{\ln \widehat{\cZ}_{\mathbf{n},t,\epsilon}}{n_0} \bigg)
	= \E\Bigg[\Bigg(\frac{\ln \widehat{\cZ}_{\mathbf{n},t,\epsilon}}{n_0} - \E\bigg[\frac{\ln \widehat{\cZ}_{\mathbf{n},t,\epsilon}}{n_0} \bigg\vert \bX^1, \bW{1} \bigg]\Bigg)^{\!\! 2} \,\Bigg]\\
	+ \E\Bigg[\Bigg(\E\bigg[\frac{\ln \widehat{\cZ}_{\mathbf{n},t,\epsilon}}{n_0} \bigg\vert \bX^1, \bW{1} \bigg] - \E\bigg[\frac{\ln \widehat{\cZ}_{\mathbf{n},t,\epsilon}}{n_0}\bigg]\Bigg)^{\!\! 2} \,\Bigg]
	\leq \frac{C(\varphi_1, \varphi_2, \alpha_1, \alpha_2, S)}{n_0} \,.
\end{multline*}
As mentioned before, this implies Theorem~\ref{concentrationtheorem} thanks to \eqref{expressionZwithZhat}.
\subsection{Concentration of the overlap}\label{appendix-overlap}
This section presents the proof of Proposition~\ref{prop:concentrationOverlap}. This proof is essentially the same as the one provided for the one-layer GLM \cite{BarbierOneLayerGLM}. All along this section $t \in [0,1]$ is fixed, and the averaged free entropy is treated as a mapping $(R_1,R_2) \mapsto f_{\mathbf{n},\epsilon}(t)$ of $R_1 = R_1(t,\epsilon)$ and $R_2 = R_2(t,\epsilon)$, given by \eqref{R1R2}. The same is true for the free entropy corresponding to a realization of the quenched variables:
\begin{equation*}
(R_1,R_2) \mapsto F_{\mathbf{n},\epsilon}(t) \defeq \frac{1}{n_0} \ln \cZ_{\mathbf{n},t,\epsilon}(\bY_{t,\epsilon},\bY_{t,\epsilon}',\bW{1},\bW{2},\bV) \,.
\end{equation*}
Define $\bx^1  \defeq \varphi_{1}\big(\nicefrac{\bW{1} \bx}{\sqrt{n_0}}, \ba_{1}\big)$ where the triplet $(\bx,\bu,\ba_1)$ is sampled from the posterior distribution associated to the Gibbs bracket $\langle - \rangle_{\mathbf{n},t, \epsilon}$. Hence $\bx^1$ is nothing but a sample obtained from the conditionnal distribution of $\bX^1$ given $(\bY_{t,\epsilon}, \bY_{t,\epsilon}^{\prime}, \bW{1},\bW{2},\bV)$.
Let 
\begin{equation*}
\mathcal{L} \defeq \frac{1}{n_1} \sum_{i=1}^{n_1}\Bigg(
\frac{\big(x_i^1\big)^2}{2} - x_i^1 X_i^1 - \frac{x_i^1  Z_i^{\prime}}{2\sqrt{R_1}} \Bigg)\,.
\end{equation*}
First we prove a formula that we uses extensively, in particular in Lemma~\ref{lem:perturbation_f}.
\begin{lemma}[Formula for $\E \langle \mathcal{L} \rangle_{\mathbf{n},t,\epsilon}$]\label{lem:expectationL}
For any $\epsilon \in \mathcal{B}_{n_0}$,
\begin{equation}
	\E \langle \mathcal{L} \rangle_{\mathbf{n},t,\epsilon} = -\frac{1}{2} \E \langle \widehat{Q} \rangle_{\mathbf{n},t,\epsilon} \; . \label{formula ExpectationL}
\end{equation}
\begin{proof}
For $i \in \{1,\dots,n_1\}$
\begin{align*}
\E\Big[X_i^1 \big\langle x_i^1\big\rangle_{\mathbf{n},t,\epsilon} \Big]
&= \E\Big[\big\langle x_i^1\big\rangle_{\mathbf{n},t,\epsilon}^2 \Big]\,;\\
\E\Big[\big\langle x_i^1\big\rangle_{\mathbf{n},t,\epsilon} Z_i^{\prime}\Big]
&= \E\left[\frac{\partial \langle x_i^1 \rangle_{t, \epsilon}}{\partial \widehat{Z}_i}\right]
= \E\Big[ \sqrt{R_1(t,\epsilon)} \Big(\big\langle (x_i^1)^2 \big\rangle_{\mathbf{n},t,\epsilon}-\big\langle x_i^1 \big\rangle_{\mathbf{n},t,\epsilon}^2\Big)\Big]\,.
\end{align*}
The first equality follows from the Nishimory identity and the second one from an integration by parts w.r.t. $Z_i^{\prime}$. From $\mathcal{L}$ definition we directly get
\begin{align*}
	\E \langle \mathcal{L}\rangle_{\mathbf{n},t,\epsilon}
	&= \frac{1}{n_1} \sum_{i=1}^{n_1} \frac{1}{2}\E\big[\big\langle (x_i^1)^2 \big\rangle_{t, \epsilon}\big]
	- \E\big[X_i^1 \big\langle x_i^1 \big\rangle_{t, \epsilon}\big] 
	-\frac{1}{2\sqrt{R_1(t,\epsilon)}}\E\big[\big\langle x_i^1 \big\rangle_{t, \epsilon} Z_i^{\prime} \big]\nn
	&= \frac{1}{n_1} \sum_{i=1}^{n_1} -\frac{1}{2}\E\big[\langle x_i^1 \rangle_{\mathbf{n},t,\epsilon}^2\big]
	= -\frac{1}{2} \frac{1}{n_1} \sum_{i=1}^{n_1} \E\big[\langle x_i^1 \rangle_{\mathbf{n},t,\epsilon} X_i^1\big] \
	= -\frac{1}{2} \E \langle \widehat{Q} \rangle_{\mathbf{n},t,\epsilon} \; .
\end{align*}
\end{proof}
\end{lemma}
The fluctuations of the overlap $\widehat{Q}= \frac{1}{n_1}\sum\limits_{i=1}^{n_1} x_i^1 X_i^1$ and those of $\mathcal{L}$ are related through the identity
\begin{multline}
\mathbb{E}\Big[\Big\langle (\mathcal{L} - \mathbb{E}\langle \mathcal{L}\rangle_{\mathbf{n},t, \epsilon})^2\Big\rangle_{\! \mathbf{n},t, \epsilon}\Big]\\
= \frac{1}{4}\mathbb{E}\Big[\Big\langle \big(\widehat{Q} - \mathbb{E}\langle \widehat{Q} \rangle_{n, t, \epsilon}\big)^2\Big\rangle_{\! \mathbf{n},t,\epsilon}\Big]
+ \frac{1}{2}\mathbb{E}\big[\langle \widehat{Q}^2\rangle_{\mathbf{n},t,\epsilon} -   \langle \widehat{Q} \rangle_{\mathbf{n},t,\epsilon}^2\big]
+ \frac{\rho_1(n_0)}{4n_{1} R_1} \,,\label{linkConcentrationOverlapL}
\end{multline}
from which it directly follows
\begin{equation}
\mathbb{E}\Big[\Big\langle (\mathcal{L} - \mathbb{E}\langle \mathcal{L}\rangle_{\mathbf{n},t, \epsilon})^2\Big\rangle_{\! \mathbf{n},t, \epsilon}\Big]
\geq \frac{1}{4}\mathbb{E}\Big[\Big\langle \big(\widehat{Q} - \mathbb{E}\langle \widehat{Q} \rangle_{n, t, \epsilon}\big)^2\Big\rangle_{\! \mathbf{n},t,\epsilon}\Big]\,.\label{boundConcentrationOverlapL}
\end{equation}
The full derivation in Section 6 of \cite{barbier_stoInt} can be reproduced exactly -- doing the identifications $X_i^1 \leftrightarrow S_i$, $x_i^1 \leftrightarrow X_i$, $n_1 \leftrightarrow n$, $R_1 \leftrightarrow \tilde\epsilon$ -- to obtain \eqref{linkConcentrationOverlapL}.
Indeed, the proof in Section 6 of \cite{barbier_stoInt} involves only lengthy algebra using the Nishimori identity and integration by parts w.r.t.\ the Gaussian random variable $Z_i^{\prime}$.
Therefore Proposition~\ref{prop:concentrationOverlap} is a consequence of $\mathcal{L}$'s own concentration:
\begin{proposition}[Concentration of $\mathcal{L}$ on ${\E[\langle\mathcal{L}\rangle_{\mathbf{n},t,\epsilon}]}$]\label{L-concentration}
Assume $(q_{\epsilon})_{\epsilon \in {\cal B}_{n_0}}$ and $(r_{\epsilon})_{\epsilon \in {\cal B}_{n_0}}$ are regular (see Definition~\ref{def:reg} in section~\ref{interp-est-problem}).
Under assumptions~\ref{hyp:bounded},~\ref{hyp:c2},~\ref{hyp:phi_gauss2} there exists a constant $C(\varphi_1,\varphi_2,\alpha_1,\alpha_2, S)$ independent of $t$ such that
\begin{equation}
	\int_{{\cal B}_{n_0}} \!\! d\epsilon\,\E\big[\big\langle \big(\mathcal{L} - \E\big[\langle \mathcal{L} \rangle_{\mathbf{n},t,\epsilon}\big]\big)^2\big\rangle_{\mathbf{n},t,\epsilon}\big]
	\leq \frac{C(\varphi_1,\varphi_2,\alpha_1,\alpha_2, S)}{n_0^{\nicefrac{1}{4}}}\;.
\end{equation}
\end{proposition}
The proof of this proposition is broken in two parts, in accordance with the equality
\begin{equation}
\mathbb{E}\Big[\big\langle (\mathcal{L} - \mathbb{E}\langle \mathcal{L}\rangle_{\mathbf{n},t,\epsilon})^2\big\rangle_{\mathbf{n},t,\epsilon}\Big]
= \mathbb{E}\Big[\big\langle (\mathcal{L} - \langle \mathcal{L}\rangle_{\mathbf{n},t,\epsilon})^2\big\rangle_{\mathbf{n},t,\epsilon}\Big]
+ \mathbb{E}\Big[(\langle \mathcal{L}\rangle_{\mathbf{n},t,\epsilon} - \mathbb{E}\langle \mathcal{L}\rangle_{\mathbf{n},t,\epsilon})^2\Big].
\end{equation}
A first lemma expresses concentration w.r.t.\ the posterior distribution (or ``thermal fluctuations''):
\begin{lemma}[Concentration of $\mathcal{L}$ on ${\langle \mathcal{L}\rangle_{\mathbf{n},t,\epsilon}}$]\label{thermal-fluctuations}
Assume $(q_{\epsilon})_{\epsilon \in {\cal B}_{n_0}}$ and $(r_{\epsilon})_{\epsilon \in {\cal B}_{n_0}}$ are regular.
Under assumptions~\ref{hyp:bounded},~\ref{hyp:c2},~\ref{hyp:phi_gauss2}, we have for $n_0$ large enough
\begin{equation}
\int_{{\cal B}_{n_0}} \!\! d\epsilon\,\E\big[\big\langle \big(\mathcal{L} - \langle \mathcal{L} \rangle_{\mathbf{n},t,\epsilon}\big)^2\big\rangle_{\mathbf{n},t,\epsilon}\big]
\leq \frac{\rho_1(1 + \rho_1)}{\alpha_{1} n_0} \,.
\end{equation}
\begin{proof}
For any realization of the quenched variables, the first two derivatives of the	free entropy read
\begin{align}
\frac{dF_{\mathbf{n}, \epsilon}(t)}{d R_1}
&= -\frac{n_1}{n_0} \langle \mathcal{L} \rangle_{\mathbf{n},t,\epsilon} 
-\frac{n_1}{2n_0}\frac{\Vert \bX^1\Vert^2}{n_1}
-\frac{1}{2n_0\sqrt{R_1}}\sum_{i=1}^{n_1}Z_i^{\prime} X_i^1\:;\label{derivative_Fn}\\
\frac{1}{n_0}\frac{d^2F_{\mathbf{n},\epsilon}(t)}{d R_1^2}
&= \left(\frac{n_1}{n_0}\right)^{\! 2} \big(\langle \mathcal{L}^2 \rangle_{\mathbf{n},t,\epsilon} - \langle \mathcal{L} \rangle_{\mathbf{n},t,\epsilon}^2\big) - \frac{1}{4n_0^2 R_1^{3/2}}\sum_{i=1}^{n_1}\big\langle x_i^1 \big\rangle_{\mathbf{n},t,\epsilon} Z_i^{\prime}\,.\label{secondDerivative_Fn}
\end{align}
Averaging both identities with respect to the quenched disorder gives
\begin{align}
\frac{df_{\mathbf{n}, \epsilon}(t)}{d R_1}
&= -\frac{n_1}{n_0} \E\big[\langle \mathcal{L} \rangle_{\mathbf{n},t,\epsilon}\big]
-\frac{n_1}{2n_0}\rho_1(n_0)
= -\frac{n_1}{2 n_0}\big(\rho_1(n_0) - \E\big[\langle \widehat{Q} \rangle_{\mathbf{n},t,\epsilon}\big]\big)\:; \label{derivative_fn}\\
\frac{1}{n_0}\frac{d^2f_{\mathbf{n},\epsilon}(t)}{d R_1^2}
&= \left(\frac{n_1}{n_0}\right)^{\! 2} \E\big[\big(\langle \mathcal{L}^2 \rangle_{\mathbf{n},t,\epsilon} - \langle \mathcal{L} \rangle_{\mathbf{n},t,\epsilon}^2\big)\big]
- \frac{1}{4n_0^2 R_1}\sum_{i=1}^{n_1}\E\big[\big\langle (x_i^1)^2 \big\rangle_{\mathbf{n},t,\epsilon} - \big\langle x_i^1 \big\rangle_{\mathbf{n},t,\epsilon}^2 \big]\,.
\end{align}
From this last equality it immediately follows
\begin{align}
\E\big[\big\langle \big(\mathcal{L} - \langle \mathcal{L} \rangle_{\mathbf{n},t,\epsilon}\big)^2\big\rangle_{\mathbf{n},t,\epsilon}\big]
&= \frac{n_0}{n_1^2}\frac{d^2f_{\mathbf{n},\epsilon}(t)}{d R_1^2} + \frac{1}{4n_1^2 R_1}\sum_{i=1}^{n_1}\E\big[\big\langle (x_i^1)^2 \big\rangle_{\mathbf{n},t,\epsilon} - \big\langle x_i^1 \big\rangle_{\mathbf{n},t,\epsilon}^2 \big]\nn
&\leq \frac{n_0}{n_1^2}\frac{d^2f_{\mathbf{n},\epsilon}(t)}{d R_1^2} + \frac{\rho_1(n_0)}{4n_1 \epsilon_1}\:,
\end{align}
where we made use of $\sum_{i=1}^{n_1} \E\big[\big\langle (x_i^1)^2 \big\rangle_{\mathbf{n},t,\epsilon}\big] = \E[\Vert \bX^1 \Vert^2]$ and $R_1 \geq \epsilon_1$.

By assumption $(q_{\epsilon})_{\epsilon \in {\cal B}_{n_0}}$ and $(r_{\epsilon})_{\epsilon \in {\cal B}_{n_0}}$ are regular.
Therefore $R:(\epsilon_1,\epsilon_2)\mapsto (R_1(t,\epsilon),R_2(t,\epsilon))$ is a $\mathcal{C}^1$-diffeomorphism whose Jacobian $J_R$ verifies $J_R(\epsilon)\ge 1$ for all $\epsilon\in{\cal B}_{n_0}$. Integrating over $\epsilon\in{\cal B}_{n_0}$ we obtain
\begin{multline}
\int_{{\cal B}_{n_0}} \!\! d\epsilon\,\E\big[\big\langle \big(\mathcal{L} - \langle \mathcal{L} \rangle_{\mathbf{n},t,\epsilon}\big)^2\big\rangle_{\mathbf{n},t,\epsilon}\big]\\
\leq \frac{n_0}{n_1^2}\int_{R({\cal B}_{n_0})} \frac{dR_1dR_2}{J_R(R^{-1}(R_1,R_2))} \, \frac{d^2f_{\mathbf{n},\epsilon}(t)}{dR_1^2}
+ \frac{\rho_1(n_0) s_{n_0}}{4n_1}\int_{s_{n_0}}^{2s_{n_0}} \frac{d\epsilon_1}{\epsilon_1}\,.
\end{multline}
Clearly $R({\cal B}_{n_0}) \subseteq [s_{n_0}, 2s_{n_0} + \rmax] \times [s_{n_0}, 2s_{n_0} + \rho_1(n_0)]$ and
\begin{align}
&\int_{{\cal B}_{n_0}} \!\! d\epsilon\,\E\big[\big\langle \big(\mathcal{L} - \langle \mathcal{L} \rangle_{\mathbf{n},t,\epsilon}\big)^2\big\rangle_{\mathbf{n},t,\epsilon}\big]\nn
&\quad\leq \frac{n_0}{n_1^2}\int_{s_{n_0}}^{2s_{n_0} + \rho_1(n_0)} dR_2 \Bigg(\frac{df_{\mathbf{n},\epsilon}(t)}{dR_1}\bigg\vert_{R_1=2s_{n_0}+\rmax,R_2} - \frac{df_{\mathbf{n},\epsilon}(t)}{dR_1}\bigg\vert_{R_1=s_{n_0}, R_2}\Bigg)
+ \frac{\rho_1(n_0) s_{n_0}}{4n_1}\ln 2\nn
&\quad\leq -\frac{n_0}{n_1^2}\int_{s_{n_0}}^{2s_{n_0} + \rho_1(n_0)} dR_2 \frac{df_{\mathbf{n},\epsilon}(t)}{dR_1}\bigg\vert_{R_1=s_{n_0}, R_2}
+ \frac{\rho_1(n_0) s_{n_0}}{4n_1}\ln 2\nn
&\quad\leq \frac{n_0}{n_1^2} (s_{n_0} + \rho_1(n_0)) \frac{n_1}{2n_0}\rho_1(n_0) 
+ \frac{\rho_1(n_0) s_{n_0}}{4n_1}\ln 2\nn
&\quad= \frac{\rho_1(n_0)}{2 n_1} \bigg(s_{n_0} + s_{n_0}\frac{\ln 2}{2} + \rho_1(n_0)\bigg)
\end{align}
The second equality is a simple consequence of $\nicefrac{df_{\mathbf{n},\epsilon}(t)}{dR_1}$'s negativity, as it is clear from \eqref{derivative_fn}. For the third inequality we made use of $0 \leq -\nicefrac{df_{\mathbf{n},\epsilon}(t)}{dR_1} \leq \frac{n_1}{2n_0}\rho_1(n_0)$. Finally, because $s_{n_0} + s_{n_0}\nicefrac{\ln 2}{2}$ is less than $1$ and $\rho_1(n_0) \to \rho_1$, we have
\begin{equation*}
\int_{{\cal B}_{n_0}} \!\! d\epsilon\,\E\big[\big\langle \big(\mathcal{L} - \langle \mathcal{L} \rangle_{\mathbf{n},t,\epsilon}\big)^2\big\rangle_{\mathbf{n},t,\epsilon}\big]
\leq \frac{\rho_1(1 + \rho_1)}{\alpha_{1} n_0} \quad \text{for } n_0 \text{ large enough.}
\end{equation*}
\end{proof}
\end{lemma}

The second lemma expresses the concentration of the average overlap w.r.t.\ realizations of the quenched variables.
\begin{lemma}[Concentration of $\langle\mathcal{L}\rangle_{\mathbf{n},t,\epsilon}$ on ${\mathbb{E}\big[\langle \mathcal{L} \rangle_{\mathbf{n},t,\epsilon}\big]}$ ]\label{disorder-fluctuations}
Assume $(q_{\epsilon})_{\epsilon \in {\cal B}_{n_0}}$ and $(r_{\epsilon})_{\epsilon \in {\cal B}_{n_0}}$ are regular.
Under assumptions~\ref{hyp:bounded},~\ref{hyp:c2},~\ref{hyp:phi_gauss2}, there exists a constant $C(\varphi_1,\varphi_2,\alpha_1,\alpha_2, S)$ independent of $t$ such that
\begin{align}\label{integral-form2}
\int_{{\cal B}_{n_0}} \!\!\!\!\! d\epsilon\, \E\big[\big(\langle \mathcal{L}\rangle_{\mathbf{n},t,\epsilon} - \mathbb{E}\langle \mathcal{L}\rangle_{\mathbf{n},t,\epsilon}\big)^2] \leq \frac{C(\varphi_1,\varphi_2,\alpha_1,\alpha_2, S)}{n_0^{\nicefrac{1}{4}}} \; .
\end{align}
\begin{proof}
We introduce the two functions
\begin{equation}\label{tildeFandf}
	\tilde{F}(R_1) \defeq F_{\mathbf{n}, \epsilon}(t) - \frac{\sqrt{R_1}}{n_0}\sup \vert \varphi_1 \vert
	\sum_{i=1}^{n_1} \vert Z_i^{\prime} \vert \;,\;
	\tilde{f}(R_1) \defeq f_{\mathbf{n}, \epsilon}(t) - \frac{\sqrt{R_1}}{n_0}\sup \vert \varphi_1 \vert\sum_{i=1}^{n_1} \E \vert Z_i^{\prime} \vert\; .
\end{equation}
The addition of the second term makes $\tilde{F}(R_1)$ convex as its second derivative is positive (remember the formula \eqref{secondDerivative_Fn} for the second derivative $\nicefrac{d^2F_{\mathbf{n},\epsilon}(t)}{d R_1^2}$ ). $f_{\mathbf{n}, \epsilon}(t)$ is convex and it remains true for $\tilde{f}(R_1)$.
Define $A \defeq \frac{1}{n_1}\sum_{i=1}^{n_1} \big(\vert Z_i^{\prime} \vert - \E[\vert Z_i^{\prime} \vert]\big)$. From \eqref{tildeFandf} we get
\begin{equation}\label{tildeFminusf}
\tilde{F}(R_1) - \tilde{f}(R_1)
= F_{\mathbf{n}, \epsilon}(t)
- f_{\mathbf{n}, \epsilon}(t) - \sqrt{R_1} \sup \vert \varphi_1 \vert \, \frac{n_1}{n_0} A \,,
\end{equation}
whose derivative reads -- remember \eqref{derivative_Fn},~\eqref{derivative_fn} --
\begin{multline}\label{derivative_tildeFminusf}
\tilde{F}^{\prime}(R_1) - \tilde{f}^{\prime}(R_1)
= \frac{n_1}{n_0}\big(\E\big[\langle \mathcal{L} \rangle_{\mathbf{n},t,\epsilon}\big]
- \langle \mathcal{L} \rangle_{\mathbf{n},t,\epsilon}\big)
-\frac{n_1}{2n_0}\Bigg(\frac{\Vert \bX^1\Vert^2}{n_1}-\rho_1(n_0)\Bigg)\\
-\frac{1}{2n_0\sqrt{R_1}}\sum_{i=1}^{n_1}Z_i^{\prime} X_i^1
- \frac{\sup \vert \varphi_1 \vert}{2 \sqrt{R_1}} \, \frac{n_1}{n_0} A \:.
\end{multline}
An application of Lemma 31 in \cite{BarbierOneLayerGLM}, combined with \eqref{tildeFminusf} and \eqref{derivative_tildeFminusf}, gives
\begin{multline}\label{inequality-lemma-convexity}
\vert \langle \mathcal{L}\rangle_{\mathbf{n},t,\epsilon} - \mathbb{E}\langle \mathcal{L}\rangle_{\mathbf{n},t,\epsilon}\vert\quad
\leq 
\!\!\!\!\!\!\!\!\!\!\! \sum_{u\in \{R_1 -\delta, R_1, R_1+\delta\}} \!\!\!\!\!\!\!\!\!\!\!\!\!\!\!
\delta^{-1} \big(\big\vert \big(F_{\mathbf{n},\epsilon}(t) - f_{\mathbf{n},\epsilon}(t)\big)\big\vert_{R_1 = u}\big\vert + \sup \vert \varphi_1 \vert \, \frac{n_1}{n_0} \vert A \vert \sqrt{u} \big)\\
+ C_\delta^+(R_1)
+ C_\delta^-(R_1)
+\bigg\vert\frac{\Vert \bX^1\Vert^2}{n_1}-\rho_1(n_0)\bigg\vert
+ \bigg\vert \frac{1}{2 n_1\sqrt{R_1}}\sum_{i=1}^{n_1}Z_i^{\prime} X_i^1 \bigg\vert
+ \frac{\sup \vert \varphi_1 \vert}{2 \sqrt{R_1}} \, A \,.
\end{multline}
where $C_\delta^+(R_1)\defeq \widetilde f'(R_1+\delta)-\widetilde f'(R_1)\ge 0$ and $C_\delta^-(R_1)\defeq \widetilde f'(R_1)-\widetilde f'(R_1-\delta)\ge 0$.

By Theorem \ref{concentrationtheorem} (see Appendix \ref{appendix_concentration}) there exists a constant $C(\varphi_1,\varphi_2,\alpha_1,\alpha_2, S)$, independent of $(t,\epsilon)$,such that $\E\big[\big(F_{\mathbf{n},\epsilon}(t) - f_{\mathbf{n},\epsilon}(t)\big)^{2}\big] \leq \nicefrac{C(\varphi_1,\varphi_2,\alpha_1,\alpha_2, S)}{n_0}$.
In the Appendix \ref{app:uniformVanishingA} we also proved the existence of a constant $C(\varphi_1,\alpha_1, S)$ such that $\E\big[\big(\nicefrac{\Vert \bX^1\Vert^2}{n_1}-\rho_1(n_0)\big)^2\big] \leq \nicefrac{C(\varphi_1,\alpha_1,S)}{n_0}$. Finally
\begin{equation*}
\E\bigg[\bigg(\frac{1}{n_1}\sum_{i=1}^{n_1}Z_i^{\prime} X_i^1 \bigg)^{\! 2}\bigg]
= \frac{1}{n_1^2}{\mathbb{V}\mathrm{ar}}\Bigg[\sum_{i=1}^{n_1}Z_i^{\prime} X_i^1\Bigg]
= \frac{{\mathbb{V}\mathrm{ar}}[Z_1^{\prime} X_1^1]}{n_1}
\leq \frac{\E[(X_1^1)^2]}{n_1} = \frac{\rho_1(n_0)}{n_1} \,,
\end{equation*}
where we used that the random variables $\{Z_i^{\prime} X_i^1\}_{1\leq i \leq n_1}$ are uncorrelated to get the second equality, and
\begin{equation*}
\E[A^2]
= \frac{{\mathbb{V}\mathrm{ar}}[\vert Z_1^{\prime} \vert]}{n_1}
\leq \frac{1}{n_1} \bigg(1 - \frac{2}{\pi}\bigg)\,.
\end{equation*}
These upperbounds are combined with \eqref{inequality-lemma-convexity} and Jensen's inequality $\big(\sum_{i=1}^pv_i\big)^2 \le p\sum_{i=1}^pv_i^2$ into
\begin{align}\label{Jensen_concentration}
&\frac{1}{11} \E\big[\big(\langle \mathcal{L}\rangle_{\mathbf{n},t,\epsilon} - \mathbb{E}\langle \mathcal{L}\rangle_{\mathbf{n},t,\epsilon}\big)^2]\nn
&\quad\leq 
\frac{3}{\delta^{2} n_0} \bigg(C(\varphi_1,\varphi_2,\alpha_1,\alpha_2, S) + \frac{n_1}{n_0} R_1 \sup \vert \varphi_1 \vert^2 \,  \bigg(1 - \frac{2}{\pi}\bigg) \bigg)
+ C_\delta^+(R_1)^2
+ C_\delta^-(R_1)^2\nn
&\quad\qquad\qquad+ \frac{C(\varphi_1,\alpha_1,S)}{n_0}
+ \frac{\rho_1(n_0)}{4 R_1 n_1}
+ \frac{\sup \vert \varphi_1 \vert^2}{4 R_1 n_1} \bigg(1 - \frac{2}{\pi}\bigg)\nn
&\quad \leq 
\frac{3}{\delta^{2} n_0} \bigg(C(\varphi_1,\varphi_2,\alpha_1,\alpha_2, S) + \alpha_1 K \sup \vert \varphi_1 \vert^2 \bigg)
+ C_\delta^+(R_1)^2
+ C_\delta^-(R_1)^2\nn
&\quad\qquad\qquad + \frac{C(\varphi_1,\alpha_1,S)}{n_0}
+ \frac{\sup \vert \varphi_1 \vert^2}{2 \alpha_1 n_0}\frac{1}{\epsilon_1} \qquad (n_0 \, \text{ large enough})\,.
\end{align}
To obtain the second inequality, we used $\nicefrac{n_1}{n_0} \to \alpha_1$, $\rho_1(n_0) \leq \sup \vert\varphi_1\vert^2$ and $\epsilon_1 \leq R_1 \leq K$ where $K \defeq 1 + \max\{\sup \vert\varphi_1\vert, \rmax\}$.

$\tilde{f}^{\prime}$ is increasing and non-positive. This, combined with $R_1 \geq \epsilon_1 \geq s_{n_0}$ gives
\begin{equation*}
\vert C_\delta^{\pm}(R_1)\vert
\leq -\tilde{f}^{\prime}(R_1 - \delta)
\leq \frac{n_1}{2 n_0} \sup \vert\varphi_1 \vert \bigg(1 + \frac{1}{\sqrt{R_1 - \delta}}\bigg)\!\!
\underbrace{\leq}_{n_0 \text{ large enough}} \!\!\!\!\!\alpha_1 \sup \vert\varphi_1 \vert \bigg(1 + \frac{1}{\sqrt{s_{n_0} - \delta}}\bigg).
\end{equation*}
Therefore, for $n_0$ large enough:
\begin{align*}
&\int_{{\cal B}_{n_0}} \!\!\!\!\!\! d\epsilon\, \big(C^+(R_1(t,\epsilon))^2 + C^-(R_1(t,\epsilon))^2\big)\nn
&\quad\leq 
\alpha_1 \sup \vert\varphi_1 \vert \bigg(1 + \frac{1}{\sqrt{s_{n_0} - \delta}}\bigg)
\int_{{\cal B}_n} d\epsilon\, \big(C^+(R_1(t,\epsilon)) + C^-(R_1(t,\epsilon))\big)\nn
&\quad= \alpha_1 \sup \vert\varphi_1 \vert \bigg(1 + \frac{1}{\sqrt{s_{n_0} - \delta}}\bigg)
\int_{R({\cal B}_{n_0})} \frac{dR_1dR_2}{J_R(R^{-1}(R_1,R_2))}\, \big(C^+(R_1) + C^-(R_1)\big)
\nn
&\quad\leq \alpha_1 \sup \vert\varphi_1 \vert \bigg(1 + \frac{1}{\sqrt{s_{n_0} - \delta}}\bigg)
\int_{R({\cal B}_{n_0})} dR_1dR_2\, \big\{C^+(R_1) + C^-(R_1)\big\} \nn
&\quad= \alpha_1 \sup \vert\varphi_1 \vert \bigg(1 + \frac{1}{\sqrt{s_{n_0} - \delta}}\bigg)
\int_{s_{n_0}}^{2s_{n_0} + \rho_1(n_0)} \!\!\!\!\! dR_2 \Big[\big(\widetilde f(2s_{n_0} + \rmax+\delta) - \widetilde f(2s_{n_0} + \rmax -\delta)\big)\nn
&\qquad\qquad\qquad\qquad\qquad\qquad\qquad\qquad\qquad\qquad\qquad + \big(\widetilde f(s_{n_0}-\delta) - \widetilde f(s_{n_0}+\delta)\big)\Big] .
\end{align*}
The mean value theorem is now applied to bound further
\begin{multline}
\int_{{\cal B}_{n_0}} \!\!\!\!\!\! d\epsilon\, \big(C^+(R_1(t,\epsilon))^2 + C^-(R_1(t,\epsilon))^2\big)\\
\leq \alpha_1 \sup \vert\varphi_1 \vert \bigg(1 + \frac{1}{\sqrt{s_{n_0} - \delta}}\bigg)
\int_{s_{n_0}}^{2s_{n_0} + \rho_1(n_0)} \!\!\!\!\!\!\!\!\! dR_2 \: 4\delta \sup_{R_1} \vert \tilde{f}^{\prime}(R_1) \vert\\
\leq 4 \delta \alpha_1^2 \sup \vert\varphi_1 \vert^2 \bigg(1 + \frac{1}{\sqrt{s_{n_0} - \delta}}\bigg)^{\! 2}(s_{n_0} + \sup \vert\varphi_1 \vert)\,.
\end{multline}
In the later, the supremum $\sup_{R_1} \vert \tilde{f}^{\prime}(R_1) \vert$ is taken over $[s_{n_0}-\delta,2s_{n_0}+\rho_1(n_0)+\delta]$ and its upper bound $\alpha_1 \sup \vert\varphi_1 \vert \Big(1 + \frac{1}{\sqrt{s_{n_0} - \delta}}\Big)$ is uniform in $R_2$.
Integrating \eqref{Jensen_concentration} over $\epsilon \in \mathcal{B}_{n_0}$ thus yields
\begin{multline*}
\int_{{\cal B}_{n_0}} \!\!\!\!\!\! d\epsilon\, \E\big[\big(\langle \mathcal{L}\rangle_{\mathbf{n},t,\epsilon} - \mathbb{E}\langle \mathcal{L}\rangle_{\mathbf{n},t,\epsilon}\big)^2] \\
\leq 
\frac{33}{n_0} \bigg(\frac{s_{n_0}}{\delta} \bigg)^{2} \bigg(C(\varphi_1,\varphi_2,\alpha_1,\alpha_2, S) + \alpha_1 K \sup \vert \varphi_1 \vert^2 \bigg)
+ 11 C(\varphi_1,\alpha_1,S) \frac{s_{n_0}^2}{n_0}\\
+ \frac{\sup \vert \varphi_1 \vert^2}{\alpha_1 n_0}\frac{11 \ln 2}{2}
+ 4 \alpha_1^2 \sup \vert\varphi_1 \vert^2 \bigg(\sqrt{\delta} + \frac{1}{\sqrt{\nicefrac{s_{n_0}}{\delta} - 1}}\bigg)^{\! 2}(s_{n_0} + \sup \vert\varphi_1 \vert)\,.
\end{multline*}
Remembering that $s_{n_0} \leq \nicefrac{1}{2}$ and choosing $\delta = s_{n_0} n_0^{-\nicefrac{1}{4}}$ give the desired result.
\end{proof}
\end{lemma}
\putbib[sm]
\end{bibunit}

\end{document}